\documentclass[11pt]{article}
\usepackage{fullpage}
\usepackage{mathrsfs}

\usepackage{comment}
\usepackage{authblk}

\usepackage[round]{natbib}
\usepackage{times}
\usepackage{url}
\usepackage{multirow}
\usepackage{xr}

\usepackage{textcomp}

\usepackage{bbm}
\usepackage{todonotes}
\usepackage{listings}
\usepackage{textgreek} 
\usepackage{amsmath} 
\usepackage{amssymb,amsthm,mathtools}

\usepackage[colorlinks=true,citecolor=blue,urlcolor=blue,linkcolor=blue]{hyperref}
\usepackage[capitalise]{cleveref}
\usepackage{stfloats} 
\usepackage{float}
\usepackage{placeins} 
\usepackage{xcolor}

\usepackage{tablefootnote} 
\usepackage[roman]{parnotes} 
\usepackage{tabulary}
\usepackage{booktabs}
\usepackage{tabularx}

\usepackage{tikz}

\usetikzlibrary{shapes,decorations,arrows,calc,arrows.meta,fit,positioning}
\tikzset{
    -Latex,auto,node distance =1 cm and 1 cm,semithick,
    state/.style ={ellipse, draw, minimum width = 0.7 cm},
    point/.style = {circle, draw, inner sep=0.04cm,fill,node contents={}},
    bidirected/.style={Latex-Latex,dashed},
    el/.style = {inner sep=2pt, align=left, sloped}
}

\usepackage{changepage}

\usepackage{algorithm}  
\usepackage{algorithmicx}
\usepackage[noend]{algpseudocode}
\definecolor{mygray}{gray}{0.5}
\definecolor{cblue}{RGB}{8, 85, 153}
\definecolor{darkblue}{RGB}{31, 64, 96}

\definecolor{cgreen}{RGB}{8, 153, 83}
\definecolor{green}{RGB}{8, 200, 50}
\definecolor{cmaroon}{RGB}{128, 0, 0}
\algdef{SE}[DOWHILE]{Do}{doWhile}{\algorithmicdo}[1]{\algorithmicwhile\ #1}%

\renewcommand\algorithmicdo{}

\Crefname{section}{Sec}{Secs.}
\Crefname{figure}{Fig}{Figs.}
\Crefname{theorem}{Thm}{Thms.}

\usepackage{amsfonts}

\usepackage{tikzsymbols}
\usepackage[definitionLists,hashEnumerators,smartEllipses,hybrid]{markdown}

\usepackage{optidef}

\newtheorem{theorem}{Theorem}[section]

\newtheorem{lemma}[theorem]{Lemma}

\newtheorem{corollary}[theorem]{Corollary}
\newtheorem{definition}[theorem]{Definition}
\newtheorem{assumption}[theorem]{Assumption}
\newtheorem{remark}[theorem]{Remark}
\newtheorem{example}[theorem]{Example}

\newcommand{\be}{\mathbf{e}}

\newcommand{\bx}{\mathbf{x}}
\newcommand{\by}{\mathbf{y}}
\newcommand{\bz}{\mathbf{z}}

\newcommand{\bX}{\mathbf{X}}

\newcommand{\balpha}{\boldsymbol{\alpha}}

\newcommand{\bLambda}{\boldsymbol{\Lambda}}

\DeclarePairedDelimiter{\braces}{\lbrace}{\rbrace}
\DeclarePairedDelimiter{\paren}{(}{)}

\DeclarePairedDelimiter{\abs}{|}{|}

\DeclarePairedDelimiter{\norm}{\|}{\|}
\DeclarePairedDelimiter{\inprod}{\langle}{\rangle}

\newcommand{\R}{\mathbb{R}}
\renewcommand{\P}{\mathbb{P}}
\newcommand{\E}{\mathbb{E}}

\newcommand{\cell}{A}

\newcommand{\partition}{P}
\newcommand{\partitions}[1]{\mathcal{P}_{#1}}
\newcommand{\emppartitions}[1]{\mathcal{P}_{#1}^\sample}

\newcommand{\partfunc}{\mathcal{F}}

\newcommand{\data}{\mathcal{D}}

\newcommand{\funcclass}{\mathcal{F}}

\newcommand{\x}{\mathbf{X}}

\newcommand{\1}{\mathbbm{1}}

\newcommand{\sample}{\mathcal{X}}

\newcommand{\Rad}[1]{\mathcal{R}_n(#1)}

\newcommand{\sequiv}{\overset{\sample}{=}}
\newcommand{\treesum}[1]{\sum_{j=1}^{#1}}

\newcommand{\leaves}{L}
\newcommand{\nnorm}[1]{\norm{#1}_{n}}
\newcommand{\munorm}[1]{\norm{#1}_{2}}

\newcommand{\besov}{\Lambda}

\newcommand{\h}{\mathbf{h}}
\newcommand{\e}{\mathbf{e}}

\newcommand{\excess}{\mathcal{E}}

\newcommand{\PSHABn}{\mathcal{B}_{p,q}^{\mathscr{S},\mathscr{A}}(\partition_*,\boldsymbol{\Lambda})}

\newcommand{\empfunctions}[1]{\mathcal{F}_{#1}}

\newcommand{\permit}{\hat f_{\lambda}}
\newcommand{\cermit}{\hat f_{L}}

\newcommand{\heavy}{m}

\newcommand{\Excess}{E}

\newcommand{\fbound}{M}

\title{On the Statistical Optimality of Optimal Decision Trees}

\newcommand*\samethanks[1][\value{footnote}]{\footnotemark[#1]}

\author[1]{Zineng Xu\thanks{xuzineng@u.nus.edu}}
\author[2]{Subhro Ghosh\thanks{subhrowork@gmail.com}\thanks{Equal contribution; listed in alphabetical order}}
\author[1]{Yan Shuo Tan\thanks{yanshuo@nus.edu.sg}\samethanks[3]}

\affil[1]{\small Department of Statistics and Data Science, National University of Singapore}
\affil[2]{\small Department of Mathematics, National University of Singapore}

\date{}

\begin{document}

\maketitle

\begin{abstract}
While globally optimal empirical risk minimization (ERM) decision trees have become computationally feasible and empirically successful, rigorous theoretical guarantees for their statistical performance remain limited. In this work, we develop a comprehensive statistical theory for ERM trees under random design in both high-dimensional regression and classification. 
We first establish sharp oracle inequalities that bound the excess risk of the ERM estimator relative to the best possible approximation achievable by any tree with at most $L$ leaves, thereby characterizing the interpretability-accuracy trade-off. 
We derive these results using a novel uniform concentration framework based on empirically localized Rademacher complexity. 
Furthermore, we derive minimax optimal rates over a novel function class: the \emph{piecewise sparse heterogeneous anisotropic Besov} (PSHAB) space. This space explicitly captures three key structural features encountered in practice: sparsity, anisotropic smoothness, and spatial heterogeneity.
While our main results are established under sub-Gaussianity, we also provide robust guarantees that hold under heavy-tailed noise settings.
Together, these findings provide a principled foundation for the optimality of ERM trees and introduce empirical process tools broadly applicable to other highly adaptive, data-driven procedures.
\end{abstract}

\section{Introduction}
\label{section:intro}

Decision trees and their ensembles have remained among the most popular nonparametric methods for regression and classification since their inception \citep{morgan1963problems, breiman1984classification}.
Their enduring appeal stems from a unique combination of high predictive power and inherent interpretability.
Unlike ``black box" models such as neural networks, decision trees model the data through a hierarchy of logical rules that are easily visualized and understood by humans.
This transparency is particularly critical in high-stakes domains such as healthcare, criminal justice, and credit scoring, where understanding the rationale behind a prediction is as important as the prediction itself \citep{rudin2022interpretable}.

For decades, the construction of decision trees relied primarily on greedy heuristics, such as CART \citep{breiman1984classification} and C4.5 \citep{quinlan1993c45}.
Because finding the globally optimal decision tree is known to be NP-hard \citep{hyafil1976constructing}, these greedy algorithms recursively optimize local objectives without revisiting prior splits.
While computationally efficient, greedy approaches are prone to getting trapped in local optima, often producing trees that are sub-optimal in accuracy or unnecessarily complex \citep{tan2024statistical}.
However, recent advances in mixed-integer optimization (MIO) and dynamic programming, coupled with significant increases in computational power, have made it feasible to search directly over the space of decision trees (see for instance \citet{bertsimas2017optimal,verwer2017learning,carrizosa2021mathematical,lin2020generalized}).
These algorithms produce \emph{optimal decision trees}—true empirical risk minimizers (ERM)—that demonstrably outperform their greedy counterparts.
Crucially, they offer superior accuracy for a fixed budget of leaves, thereby strictly improving the interpretability-accuracy trade-off.

Despite the growing practical deployment of ERM trees, theoretical analysis of their statistical properties has lagged behind.
Existing theoretical works suffer from three primary limitations.
First, prior analyses generally focus on pure predictive accuracy without explicitly modeling the interpretability constraint—specifically, the performance achievable given a hard cap on the number of leaves.
Second, nearly all rigorous results are restricted to \emph{dyadic} decision trees, where splits are forced to occur at the geometric midpoints of cells \citep{donoho1997cart, scott2006minimax, blanchard2007optimal}.
This restriction is analytically convenient but essentially unused in practice.
Third, optimality is typically established over standard function spaces—such as H{\"o}lder, Sobolev, or Bounded Variation classes—in low-dimensional settings \citep{chatterjee2021adaptive}.
Since classical kernel methods and other non-adaptive methods are already known to be minimax optimal in these regimes, existing theory fails to articulate why tree-based methods should be preferred over non-adaptive alternatives.

To address these gaps, we develop a general theory for the statistical performance of non-dyadic ERM trees under random design. We first establish sharp \emph{oracle inequalities} that bound the excess risk of the ERM estimator relative to the best possible approximation achievable by any tree with at most $L$ leaves. By explicitly conditioning on the number of leaves $L$, these inequalities rigorously characterize the interpretability-accuracy trade-off. Crucially, we derive these results using a novel uniform concentration framework based on empirically localized Rademacher complexity. 

Second, in our view, the superior predictive performance of decision trees over kernel methods arises from their ability to perform two distinct types of automatic adaptation with minimal hyperparameter tuning:
(1) Adaptation to \emph{sparsity and anisotropy}, where the signal depends on a small subset of features or varies in smoothness across different directions; and
(2) Adaptation to \emph{spatial heterogeneity}, where the smoothness or structure of the function varies across different regions of the input space.
We therefore introduce the \emph{Piecewise Sparse Heterogeneous Anisotropic Besov} (PSHAB) space—a function class designed to capture simultaneous sparsity, anisotropic smoothness, and spatial heterogeneity. We prove that ERM trees achieve minimax optimal convergence rates over PSHAB spaces for both regression and classification. 
Notably, we establish what is, to the best of our knowledge, the first explicit convergence rates for tree-based methods under heavy-tailed noise, incorporating the intrinsic smoothing parameters of the underlying function class. While these rates do not yet achieve minimax optimality, they provide a pioneering non-asymptotic analysis that relaxes the pervasive sub-Gaussianity requirements in decision tree theory.

Finally, our results shed light on the fundamental strengths of tree-based methods in general.
Theoretical analysis of greedy algorithms like CART is notoriously difficult due to the path-dependence of the splitting procedure; existing bounds often require strong assumptions and rarely establish minimax optimality.
By analyzing the global empirical risk minimizer, we disentangle the \emph{representation} capabilities of decision trees from the \emph{optimization} challenges of specific algorithms.
Our work demonstrates the representational superiority of tree-structured models for high-dimensional, heterogeneous data, providing a theoretical foundation for their widespread empirical success.

\section{Fundamentals of tree-based algorithms}
\label{section:Fundamentals of tree-based algorithms}

\subsection{Problem formulation}
\label{section:setting}
We observe a labeled training dataset $\data = \braces*{(\bX_i,Y_i) \colon 1 \leq i \leq n}$, in which each labeled example $(\bX_i,Y_i)$ is drawn independently from a distribution $\mu$ on $[0,1]^d \times \R$.
We will study both regression and binary classification.
Let $\eta(\bx) \coloneqq \E\braces*{Y|\bX=\bx}$ denote the conditional expectation function and let $\xi \coloneqq Y - \eta(\bX)$ denote the response noise.
For binary classification, note that $Y$ is a Bernoulli random variable with $\P\braces*{Y = 1|\bX = \bx} = \eta(\bx)$.
For regression, we will allow $\xi$ to be heteroskedastic (i.e. dependent on $\bX$).
Given a loss function $l\colon \R\times \R \to \R$, the \emph{risk} of a prediction function $f\colon [0,1]^d \to \R$ is $R(f) \coloneqq \E\braces{l(Y,f(\bX))}$.
For simplicity, we will only consider squared error loss ($l_{\operatorname{reg}}(y,\hat y) = (y-\hat y)^2$) for regression and zero-one loss ($l_{\operatorname{cls}}(y,\hat y) = \1\braces{y\neq \hat y}$) for classification.
We will use the subscripts ``reg'' and ``cls'' to differentiate between the two cases where necessary.
Set $f^*(\bx) \coloneqq \eta(\bx)$ in regression and $f^*(\bx) \coloneqq \1\braces{\eta(\bx) \geq  1/2}$ for classification.
The Bayes risk is then equal to $R(f^*)$ and the \emph{excess risk} of a prediction function $f$ is defined as $\excess(f) \coloneqq R(f) - R(f^*)$.
The common goal in regression and classification is to use the training dataset $\data$ to obtain an estimate $\hat f(-;\data)$ that has small excess risk $\excess(\hat f(-;\data))$ with high probability (with respect to $\data$).
We will study estimators that are based on the empirical risk $\widehat R(f) \coloneqq n^{-1}\sum_{i=1}^n l(Y_i,f(\bX_i))$.
To simplify our notation, we will assume for the rest of this paper that $n \geq 2$.

\subsection{Notation}
\label{sec:further_notation}

We will use the following notation throughout the rest of this paper.
\paragraph{Vectors, random variables, indexing.} We use boldface to denote vectors and regular font for scalars; we use uppercase to denote random variables and lowercase to denote deterministic quantities.
For a indexed vector $\x_i$, we let $X_{ij}$ denote its $j$-th coordinate. 
For any integer $k$, we use the shorthand $[k] = \braces*{1,2,\ldots,k}$.

\paragraph{Norms and inner products.} For any measurable function $F\colon [0,1]^d\times\R \to \R$, let  $\norm{F}_{2} \coloneqq \E\braces*{F(\x,y)^2}^{1/2}$ and $\norm{F}_{n} \coloneqq \paren*{n^{-1}\sum_{i=1}^nF(\bX_i,Y_i)^2}^{1/2}$ denote its $L^2$ norms with respect to $\mu$ and with respect to the empirical measure induced by $\data$ respectively. 
Let $\norm{f}_\infty$ denote the essential supremum of the function.
We also define the inner products $\inprod{F, G} \coloneqq \E\braces*{F(\bX,Y)G(\bX,Y)}$ and $\inprod{F, G}_n \coloneqq n^{-1}\sum_{i=1}^n{F(\bX_i,Y_i)G(\bX_i,Y_i)}$.
This notation allows us to write our results and proofs more compactly.
For instance, the excess risk for regression and classification have the forms $\excess_{\operatorname{reg}}(f) = \norm{f-f^*}_2^2$ and $\excess_{\operatorname{cls}}(f) = \inprod{1-2\eta,f-f^*}$ respectively, where the latter equality holds whenever $f$ is Boolean-valued.
For a vector $\boldsymbol{u}$, we denote the $\ell_p$-norm as $\|\boldsymbol{u}\|_{p}$ for $1 \leq p < \infty$ and the infinity norm as $\|\boldsymbol{u}\|_{\infty}$.
\paragraph{Constants and asymptotic notation.}
We will use $C$ to denote a universal constant (not depending on any parameters) whose value will be allowed to vary from line to line.
Given any two functions of a vector of input parameters ($n$, $d$, etc.), $F$ and $G$, we say that $F \lesssim G$ (equivalently $G \gtrsim F$) if there is a universal constant $C > 0$ such that we have the functional inequality $F \leq CG$.
If $C$ depends on a specific parameter (e.g. $\rho$), we decorate it (or the asymptotic notation) with the parameter as the subscript (e.g. $C_\rho$ or $F \lesssim_\rho G$).
If $F \lesssim G$ and $F \gtrsim G$, we say that $F \asymp G$.

\paragraph{Cells, volumes and side lengths.} 
Let $\mathcal{I}$ be the collection of all left-closed and right-open intervals in $[0,1]$ (i.e. of the form $[a,b)$ for $0 \leq a < b < 1$), together with all closed intervals with right end-point equal to 1.
We define a \emph{cell} $\cell \coloneqq \times_{j=1}^dI_j \subseteq [0,1]^d$ to be a $d$-dimensional product of such intervals.
We denote its its volume by $|A|:=\prod_{j=1}^d I_j$.
For $j \in [d]$, we denote its side length in the $j$-th coordinate by $\ell_j(A)$.

\subsection{Partitions}
\label{section:cell-split}

A \emph{partition} $\partition$ is a collection of disjoint cells whose union is the entire space $[0,1]^d$.
We are most interested in partitions that correspond to decision trees, that is, they arise by recursive splits along coordinate axes.
More precisely, we say that $\partition'$ is a refinement of $\partition$ if $\partition\setminus\partition'=\{\cell\}$ and $\partition'\setminus\partition=\{\cell_-,\cell_+\}$, where $\cell$ is a cell, $\cell_-=\{\bx\in\cell:x_j\leq\tau\}$ and $\cell_+=\{\bx\in\cell:x_j>\tau\}$ for some coordinate $j\in[d]$ and split threshold $0<\tau<1$.
If $\partition$ can be obtained from $\braces{[0,1]^d}$ by a series of refinements, we call it a \emph{tree-based partition}.
For any positive integer $L$, denote the collection of all tree-based partitions with at most $L$ leaves via $\partitions{L}$.

Every partition $\partition$ of the covariate space induces a corresponding partition of the unlabeled training dataset $\sample =\{\bX_1,\bX_2, \ldots,\bX_n\}$.
Since multiple partitions can induce the same data partition, it is common practice to constrain split thresholds to be the observed data values (i.e. a split on feature $j$ satisfies $\tau \in \braces{ X_{ij} \colon i \in [n]}$), thereby reducing ambiguity. 
We call any partition under this constraint a \emph{valid tree-based partition}, and denote the collection of such partitions with at most $L$ leaves via $\emppartitions{L}$.
In comparison to the infinite size of $\partitions{L}$, $\emppartitions{L}$ has a finite size that can easily be bounded.

\begin{lemma} \label{lem:card-partition-set}
    The number of valid tree-based partitions with at most $L$ leaves satisfies $\abs{\emppartitions{L}} \leq (dn)^L$.
\end{lemma}

\begin{proof}
Prove this by induction.
Every element of $\emppartitions{L}$ is obtained from an element of $\emppartitions{L-1}$ by making one split.
Each split is uniquely determined by its coordinate direction and the observation whose coordinate is chosen as its threshold, which gives at most $dn$ possibilities.
\end{proof}

\subsection{Decision trees}

For any cell $\cell$, let $\1_\cell(\bx) \coloneqq \1\braces{\bx \in \cell}$ for convenience.
A \emph{decision tree function} is one that can be written as $f = \sum_{j=1}^L a_j \1_\cell$, where $\braces{\cell_1,\cell_2,\ldots,\cell_L}$ form a tree-based partition $\partition$ and $(a_1,a_2,\ldots,a_L)$ are a vector of (leaf) parameters.
For any decision tree function $f$, let $\#\operatorname{leaves}(f)$ denote the number of leaves of $f$.
A \emph{decision tree algorithm} is an estimator that, given the training data, returns a decision tree function.
For a fixed tree partition $\partition$, let $\partfunc_\partition$ denote the space of decision tree functions that are piecewise constant on $\partition$.
Let $\partfunc_L$ denote the set of decision tree functions with at most $L$ leaves.
Let $\partfunc_{L}^\sample$ denote the restriction of $\partfunc_L$ to those functions induced by valid tree-based partitions.
With this notation, we can define the central objects of our analysis.

\begin{definition}[ERM regression tree estimators]
\label{def:erm_reg}
A constrained ERM regression tree estimator (with tuning parameters $L$ and $\fbound$) is denoted as $\hat f_L$ and defined as a solution to
\begin{equation}
\label{eq:erm_reg_constrained}
    \min_{f \in \empfunctions{L}^\sample} \;\; \widehat R_{\operatorname{reg}}(f)
    \quad \text{subject to } \|f\|_\infty \leq \fbound .
\end{equation}
A penalized ERM regression tree estimator (with tuning parameters $\lambda$ and $\fbound$) is denoted as $\hat f_\lambda$ and defined as a solution to
\begin{equation}
\label{eq:erm_reg_penalized}
    \min_{f \in \empfunctions{n}^\sample} \;\; \widehat R_{\operatorname{reg}}(f) + \lambda\cdot \#\operatorname{leaves}(f)
    \quad \text{subject to } \|f\|_\infty \leq \fbound .
\end{equation}
\end{definition}

\begin{remark}[Notation]
    In our theoretical results, $M$ will be treated as a fixed constant.
    For conciseness, we thus omit the dependence on $M$ in the notation for the estimators.
\end{remark}

\begin{definition}[ERM classification tree estimators]
\label{def:erm_cls}
    A constrained ERM classification tree estimator (with tuning parameter $L$) is denoted as $\hat f_L$ and defined as a solution to
\begin{equation}
\label{def:classification_constrained}
    \min_{f \in \empfunctions{L}^\sample} \;\; \widehat R_{\operatorname{cls}}(f)
    \quad \text{subject to }\; f(\bx) \in \{0,1\} \;\; \text{for all}~\bx \in [0,1]^d.
\end{equation}
A penalized ERM classification tree estimator (with tuning parameters $\lambda$ and $\theta$) is denoted as $\hat f_{\lambda,\theta}$ and defined as a solution to
\begin{equation}
    \min_{f \in \empfunctions{n}^\sample} \;\; \widehat R_{\operatorname{cls}}(f) + \lambda\cdot \paren*{\#\operatorname{leaves}(f)}^\theta \quad \text{subject to } f(\bx) \in \{0,1\} \;\; \text{for all}~\bx \in [0,1]^d.
\end{equation}
\end{definition}

\begin{remark}
    As will be shown, the additional $\theta$ tuning parameter for the penalized ERM classification tree estimator is required to obtain optimal excess risk guarantees.
    The value to be chosen to obtain these guarantees depends on the rate of density decay at the Bayes decision boundary, formalized in the so-called Tsybakov margin assumption.
    This difference from its regression counterpart can be attributed to the geometry of the risk function and perhaps explains why in practice, optimal classification tree algorithms tend to make use of the constrained problem definition \eqref{def:classification_constrained} instead \citep{verwer2017learning, verwer2019learning, zhu2020scalable, ALES2024106515, LIU2024106629, aghaei2021strong}.
\end{remark}

\begin{remark}
For a fixed partition $\partition = \braces{\cell_1,\cell_2,\ldots,\cell_L}$, the minimizer $\hat f_\partition$ of the empirical risk over the set $\partfunc_\partition$ can be shown to have leaf parameters derived from the mean responses within each cell.
Specifically, let $N(\cell) = \sum_{i=1}^n \1_A(\x_i)$ denote the number of training data points contained in $\cell$. 
For any function $Z$ of $(\bx,y)$, let $\bar{Z}_\cell \coloneqq N(\cell)^{-1}\sum_{i=1}^n \1_\cell(\bX_i) Z(\bX_i,Y_i)$ denote the mean value of the function on data points within $\cell$.
The penalized empirical risk minimizer can be shown to be of the form $\hat f_\partition = \treesum{L}\bar Y_{A_j}\1_{A_j}$ for regression and $\hat f_\partition = \treesum{L}\1\braces*{{\bar Y_{A_j}\geq 1/2}} \1_{A_j}$ for classification.
The main optimization challenge is hence in determining the optimal tree-based partition.
\end{remark}

\section{Oracle inequalities}
\label{section:oracle-ineq}

\subsection{Oracle inequalities for regression}

We first define, for $L=1,2,\ldots$, the $L$-th \emph{tree approximation error} (for regression) as the minimum excess risk value achievable by decision tree functions with at most $L$ leaves, i.e.:
\begin{equation} \label{eq:tree_partition_coefficient}
    \Excess_{\operatorname{reg},L} \coloneqq\inf_{f \in \partfunc_{L}}\excess_{\operatorname{reg}}(f).
\end{equation}
As verified by the formula $\excess_{\operatorname{reg}}(f) = \norm{f-f^*}_2^2$, this value depends only on $f^*$ (and the covariate marginal of $\mu$) and does not at all depend on the distribution of the noise $\xi$.

\begin{theorem}[Oracle inequalities for ERM regression trees]
\label{thm:regression}
Assume the regression setting of Section~\ref{section:setting}, and let 
$\hat f_{L}$ and $\hat f_{\lambda}$ denote the constrained and penalized ERM regression tree estimators 
(Definition~\ref{def:erm_reg}). 
Suppose that $\|f^*\|_\infty \leq \fbound$ and that, for any $\bx \in [0,1]^d$, the conditional distribution of $\xi$ given $\bX = \bx$ has sub-Gaussian norm bounded by $K$.
There is a universal constant $C>0$ such that, for any $u \geq 0$, with probability at least $1 - e^{-u}$, the following holds 
simultaneously for all $L \in [n]$:
\begin{equation}
    \excess_{\mathrm{reg}}(\hat f_{L})
    \;\;\leq\;\;
    \;\inf_{0<\delta<1}\;\braces*{\frac{1+\delta}{1-\delta}
    \left(
        \Excess_{\mathrm{reg},L} 
        + \frac{C(\fbound + K)^2 \,\big(L\log(nd) + u\big)}{\delta n}
    \right)}.
    \label{eq:thm-statement:regression-1}
\end{equation}
\vspace{0.5em}
Moreover, on the same event, for any 
$\lambda \geq C (\fbound + K) \,(\log(nd)+u)/(\delta n)$ and $0<\delta<1$,
\begin{equation}
    \excess_{\mathrm{reg}}(\hat f_{\lambda})
    \;\;\leq\;\;
    \frac{1+\delta}{1-\delta}
    \cdot \min_{L \in [n]}
    \left\{
        \Excess_{\mathrm{reg},L} + 2\lambda L
    \right\}.
    \label{eq:thm-statement:regression-2}
\end{equation}
\end{theorem}

It is striking how few assumptions are required for Theorem \ref{thm:regression}---we do not make any assumptions on the covariate distribution, nor on the tree structure beyond the number of leaves.
In particular, we do not need to limit the depth of the tree or the size of its leaves, which are common assumptions in most of the literature studying decision trees.
Note that the two bounds \eqref{eq:thm-statement:regression-1} and \eqref{eq:thm-statement:regression-2} are similar.
Indeed, under an optimal choice of $\lambda$ for the penalized estimator, they become almost equivalent, 
albeit with a further minimum taken over $L \in [n]$ in \eqref{eq:thm-statement:regression-1}.
We will discuss its practical significance of these bounds before contextualizing it against related literature.

\begin{remark}[Bias-variance trade-off] If we set $\delta=1/2$, the right hand side in \eqref{eq:thm-statement:regression-1} gives a type of approximation error-estimation error decomposition of the excess risk.
As the number of allowed leaves $L$ increases, the first term decreases, while the second term increases linearly, thereby yielding a trade-off between the two quantities.
Let us compare this to the decomposition obtained had we known the optimal partition $\partition^*$, i.e. that corresponding to the minimizer of \eqref{eq:tree_partition_coefficient}.
If we let $\hat f_{\partition^*}$ denote the empirical risk minimizer over $\partfunc_{\partition^*}$, Theorem 3.1 in \citet{tan2022cautionary} states that
\begin{equation}
    \excess_{\mathrm{reg}}(\hat f_{\partition^*}) \asymp \Excess_{\mathrm{reg},L} + \frac{K^2L}{n},
\end{equation}
where we further assume the noise is homoskedastic with variance $K^2$ and that the $\mu$-measure of each leaf in $\partition^*$ is not too small.
Ignoring constant factors as well as these additional assumptions for now, we see that the statistical price paid for not knowing $\partition^*$ is essentially an additional $\log(n d)$ factor on the estimation error term.
\end{remark}

\begin{remark}[Interpretability-accuracy tradeoff] \label{rmk:Interpretability-accuracy tradeoff}
By choosing $\delta$ optimally, one can show that \eqref{eq:thm-statement:regression-1} (see Appendix \ref{appendix:proof-remark}) implies the bound
\begin{equation}
    \label{eq:reg_erm_opt_a}
    \excess_{\mathrm{reg}}(\cermit)^{1/2}\leq 
   \Excess_{\operatorname{reg},L}^{1/2}+ C\paren*{\frac{(\fbound+K)^2(L\log(nd)+u)}{n}}^{1/2},
\end{equation}
which gives a tighter characterization of the excess risk when it is dominated by the approximation error term.
This occurs, for instance, in high-stakes modeling scenarios, where practitioners often choose $L$ not to balance the bias and variance terms, but instead to balance between the overall accuracy of the model and its level of interpretability, which decays as the number of leaves increases.
Under this regime, the optimized bound \eqref{eq:reg_erm_opt_a} reveals that the ERM solution performs almost as well as the oracle benchmark, incurring an overhead (square root) excess risk that depends only on the estimation error and which decays at an $n^{-1/2}$ rate.
\end{remark}

\begin{remark}[Comparison with related work]
The bound \eqref{eq:thm-statement:regression-1} shares a similar form as Theorem 2.1 in \citet{chatterjee2021adaptive}.
Note, however, that their result is obtained in a regular grid fixed design setting, with excess risk being measured with respect to the empirical norm $\norm{-}_{n}$ rather than the population norm $\norm{-}_{2}$.
As observed in their paper (Appendix C.2), their proof technique actually does not at all rely on the regular grid assumption.
It simply recognizes that the fixed design ERM problem \eqref{eq:erm_reg_penalized} is a least squares problem with the solution vector constrained to lie within a union of $L$-dimensional Euclidean subspaces, one corresponding to each element of $\emppartitions{L}$.
Under this setting, Lemma \ref{lem:card-partition-set} can be used to show that the uniform deviation of the empirical risk has order $O(L \log (nd))$, which gives the estimation error bound.
On the other hand, this argument does not extend to a random design setting, where the elements of $\emppartitions{L}$ are themselves random subspaces depending on $\sample$ and where the oracle benchmark \eqref{eq:tree_partition_coefficient} is defined in terms of all decision tree functions rather than those realizable by valid partitions.
\end{remark}

\begin{remark}[Unknown $\norm{f^*}_\infty$]
    The assumptions of Theorem \ref{thm:regression} require us to set $\fbound \geq \norm{f^*}_\infty$.
    If $\norm{f^*}_\infty$ is unknown, one can set $M \coloneqq \max_{i \in [n]} \abs{Y_i}$ (or equivalently $M \coloneq \infty)$.
    In either case, under the sub-Gaussian assumption on the noise, we can replace $M$ in \eqref{eq:thm-statement:regression-2} with $\norm{f^*}_\infty + K(\log n)^{1/2}$.
\end{remark}

\subsection{Oracle inequalities for classification}
\label{section:oracle-ineq-class}

Similar to regression, we define, for $L=1,2,\ldots$, the $L$-th \emph{tree approximation error} (for classification) as the minimum excess risk value achievable by decision tree functions with at most $L$ leaves, i.e.:
\begin{equation} \label{eq:tree_partition_coefficient_class}
    \Excess_{\operatorname{cls},L} \coloneqq\inf_{f \in \partfunc_{L}}\excess_{\operatorname{cls}}(f).
\end{equation}
In contrast to regression, this value depends not only on the Bayes predictor $f^*$ but also on the regression function $\eta$.
In fact, $\eta$ affects not only the approximation error but also the estimation error, the latter via its interaction with the rate of density decay at the Bayes decision boundary.
This decay condition is formalized via the well-known Tsybakov margin (or noise) assumption \citep{audibert2007fast}, defined as follows.

\begin{assumption}[Tsybakov margin assumption]
\label{assum:tsybakov}
Under the classification setting of Section \ref{section:setting}, we say that the distribution $\mu$ satisfies the Tsybakov margin assumption with parameters $ M>0, 0\leq \rho < \infty$ if the following holds for all $0 < t \leq 1/2$: 
    \begin{equation}
        \P\braces*{\abs*{\eta(\x)-1/2}\leq t}\leq Mt^{\rho}.
    \end{equation}
\end{assumption}

\begin{remark}[Understanding the margin assumption]
This assumption controls the amount of probability mass concentrated near the decision boundary, that is, in regions where $\eta(\mathbf{x}) \approx 1/2$. Specifically, it requires that, with high probability, $\eta(\mathbf{x})$ is either equal to $1/2$ or is bounded away from this value. When the underlying distribution satisfies the margin assumption, sharper classification guarantees can be obtained. Notably, while the margin assumption does not alter the complexity of the regression function class itself, it has a pronounced effect on the convergence rate of the excess risk through its structural implications on the data-generating distribution \citep{audibert2007fast}.
\end{remark}

\begin{theorem}[Oracle inequalities for ERM classification trees]
\label{thm:classification}
Assume the classification setting of Section~\ref{section:setting}, and let 
$\hat f_{L}$ and $\hat f_{\lambda}$ denote the constrained and penalized ERM classification tree estimators 
(Definition~\ref{def:erm_cls}). 
Suppose that Assumption \ref{assum:tsybakov} holds for some choice of parameters $(M,\rho)$.
There is a universal constant $C>0$ such that, for any $u \geq 0$, with probability at least $1 - e^{-u}$, the following holds 
simultaneously for all $L \in [n]$ and all $0 < \delta < 1$:
\begin{equation}
    \excess_{\operatorname{cls}}(\hat f_{L})
    \;\;\leq\;\;
    \frac{1+\delta}{1-\delta}
    \left(
        \Excess_{\operatorname{cls},L} 
        + C_{M,\rho}\delta^{-\rho/(2+\rho)}\paren*{\frac{\big(L\log(nd) + u\big)}{n}}^{(1+\rho)/(2+\rho)}
    \right).
    \label{eq:thm-statement:classification-1}
\end{equation}
\vspace{0.5em}
Moreover, on the same event, for any 
$\lambda\geq C_{M,\rho}\delta^{-\rho/(2+\rho)}((\log(nd)+u)/n)^{(1+\rho)/(2+\rho)}$ and $\theta\geq (1+\rho)/(2+\rho)$,
\begin{equation}
    \excess_{\operatorname{cls}}(\hat f_{\lambda,\theta})
    \;\;\leq\;\;
    \frac{1+\delta}{1-\delta}
    \cdot \min_{L \in [n]}
    \left\{
        \Excess_{\operatorname{cls},L} + 2\lambda L^{\theta}
    \right\}.
    \label{eq:thm-statement:classification-2}
\end{equation}
\end{theorem}

\begin{remark}[Role of $\rho$]
Any distribution trivially satisfies Assumption \ref{assum:tsybakov} with $M=1$ and $\rho=0$.
Under this choice of parameters, the estimation error term in \eqref{eq:thm-statement:classification-1} decays at the rate $n^{-1/2}$, matching the rate obtained for regression in \eqref{eq:thm-statement:regression-1} under the squared $L^2$ loss.
In contrast, when Assumption \ref{assum:tsybakov} holds with a large value of $\rho$, the estimation error term decays at an almost linear rate.
More generally, a larger $\rho$—corresponding to a faster decay of the marginal density near the Bayes decision boundary—leads to a faster rate of decay of the estimation error.
\end{remark}

\begin{remark}[Interpretability-accuracy tradeoff]
\label{rmk:Interpretability-accuracy tradeoff-cls}
    By choosing $\delta$ optimally, one can show that \eqref{eq:thm-statement:classification-1} (see Appendix \ref{appendix:proof-remark}) implies the bound
\begin{equation}
    \label{eq:cls_erm_opt_a}
    \excess_{\operatorname{cls}}(\cermit)^{(2+\rho)/(2+2\rho)} \leq 
   \Excess_{\operatorname{cls},L}^{(2+\rho)/(2+2\rho)}+ C_{M,\rho}\paren*{\frac{L\log(nd)+u}{n}}^{1/2}.
\end{equation}    
\end{remark}

\begin{remark}[Choice of $\theta$]
    From the assumption on $\theta$, we see that $\theta$ should be chosen between $1/2$ and $1$, with larger values chosen when there is faster density decay.
    Indeed, in a close to noiseless setting (i.e. $\rho \gg 1$), we should set $\theta$ close to $1$, while under no assumptions at all, we should set $\theta = 1/2$.
\end{remark}

\begin{remark}[Comparison with related work]
    Oracle inequalities for dyadic ERM classification trees were derived by \citet{scott2006minimax,blanchard2007optimal}.
    Both works study penalized estimators, with \citet{scott2006minimax} using a ``spatially adaptive'' penalty (see Theorem 3 therein), while \citet{blanchard2007optimal} uses $\theta = 1$.
    Both results are fairly opaque---\citet{scott2006minimax}'s bound is stated in terms of their complicated penalty, while \citet{blanchard2007optimal} makes a very strong assumption on the data (see equation (13) therein).
    To the best of our knowledge, Theorem \ref{thm:classification} provides the first oracle inequalities for non-dyadic ERM classification trees. 
\end{remark}

\section{Piecewise sparse heterogeneous anisotropic Besov spaces}
\label{section:function-space-PSHAB}

Towards establishing ideal spatial adaptation for the ERM tree estimators, we construct a family of function classes, each of which we call a \emph{piecewise sparse heterogeneous anisotropic Besov (PSHAB) space}.
Such a function class elaborates upon the classical definition of anisotropic Besov spaces \citep{leisner2003nonlinear}, which we first define.

Given a domain $\Omega\subseteq[0,1]^d$, the \emph{$r$-th order finite difference} of a function $f$ at $\bx$ with step $\h \in \mathbb{R}^d$ is defined recursively as $\Delta_{\h}^0f(\bx)\coloneq f(\bx)$ and
\begin{equation*}
    \Delta_{\h}^rf(\bx)\coloneq \Delta_{\h}^{r-1}f(\bx+\h)- \Delta_{\h}^{r-1}f(\bx), \quad \text{for } r\geq 1,
\end{equation*}
where the difference is defined on the set $\Omega(r,\h)\coloneq\{\bx \in \Omega : \bx+k\h\in\Omega \text{ for all } 0 \leq k \leq r\}$.
Let $\e_j$ denote the $j$-th standard basis vector in $\mathbb{R}^d$. The \textit{$j$-th partial modulus of smoothness of order $r$} is defined as
\begin{equation*}
    \omega_{j,p}^{[r]}(f,t,\Omega) \coloneq \sup_{0<h\leq t} \norm{\Delta_{h\e_j}^r f}_{L^p(\Omega(r,h\e_j))},
\end{equation*}
where $\norm{\cdot}_{L^p}$ denotes the standard $L^p$ norm.

\begin{definition}[Anisotropic Besov space]
Given $\Omega\subseteq[0,1]^d$ and parameters $\boldsymbol{\alpha} = (\alpha_1, \ldots, \alpha_d) \in (0, 1]^d$, $0 < p, q \leq \infty$, the \emph{Besov seminorm along the $j$-th direction} is defined as
\begin{equation*}
    |f|_{B_{j,p,q}^{\alpha_j}(\Omega)} \coloneq 
    \begin{cases}
        \paren{ \int_0^{\infty} \paren{ t^{-\alpha_j} \omega_{j,p}^{[r_j]}(f,t,\Omega) }^{q} \frac{dt}{t} }^{1/q} & (q < \infty), \\
        \sup_{t > 0} t^{-\alpha_j}\omega_{j,p}^{[r_j]}(f,t,\Omega) & (q = \infty),
    \end{cases}
\end{equation*}
where $\boldsymbol{r}=(r_1, \dots, r_d)$ such that $r_j = \lfloor \alpha_j\rfloor+1$.
Define the norm
\begin{equation} \label{eq:anisotropic_besov_norm}
    \norm{f}_{B_{p,q}^{\boldsymbol{\alpha}}(\Omega)}
    \coloneq \norm{f}_{L^p(\Omega)}+
    \sum_{j=1}^d |f|_{B_{j,p,q}^{\alpha_j}(\Omega)}.
\end{equation}
Define the \emph{anisotropic Besov space} $B_{p,q}^{\boldsymbol{\alpha}}(\Omega)$ to be the class of functions whose norm \eqref{eq:anisotropic_besov_norm} is finite.
Finally, for any $\besov > 0$, we use $B_{p,q}^{\boldsymbol{\alpha}}(\Omega,\besov) \coloneq \{ f \in B_{p,q}^{\boldsymbol{\alpha}}(\Omega) \colon \norm{f}_{B_{p,q}^{\boldsymbol{\alpha}}(\Omega)} \leq \besov \}$ to denote the ball in $B_{p,q}^{\boldsymbol{\alpha}}(\Omega)$ of radius $\besov$.
\end{definition}

\begin{remark}[Understanding Besov spaces]
    Besov spaces are often used to model spatially inhomogeneous functions because they can be characterized in terms of decay rates of wavelet coefficients~\citep{hardle2012wavelets}. Indeed, given a sufficiently smooth scaling function $\phi$ and orthonormal wavelet basis $\braces{\psi_{j,k}}$ for $L^2([0,1])$, let $\beta_0 \coloneqq \inprod{f,\phi}$ and $\beta_{j,k} \coloneqq \inprod{f,\psi_{j,k}}$ be the coefficients of a function $f$. Then, we have
    \begin{equation*}
        \norm{f}_{B^\alpha_{p,q}} \asymp \abs{\beta_0} + \paren*{\sum_{j\geq 0} \paren*{2^{j(\alpha+1/2-1/p)}\norm{\boldsymbol{\beta}_{j,\cdot}}_p}^q}^{1/q}.
    \end{equation*}
    This decomposition highlights the roles of the parameters: $p$ controls the spatial concentration of fluctuations within a single spatial scale (with smaller $p$ allowing for more spatially sparse heterogeneity), while $\alpha$ and $q$ control the rate of decay of fluctuations across scales (with larger $\alpha$ and smaller $q$ enforcing faster decay and hence greater global regularity).
    Unsurprisingly, we have the embeddings $B^{\alpha}_{p,q}([0,1]) \subset B^{\alpha'}_{p',q'}([0,1])$ if $\alpha \geq \alpha', p \geq p', q \leq q' $.
\end{remark}

\begin{remark}[Maximum smoothness]
    The usual definition of Besov spaces allows the smoothness parameters to be larger than 1. Since piecewise constant estimators such as decision trees are not adaptive to higher levels of smoothness, we restrict our attention to $\alpha_i \leq 1$ for $i \in [d]$.
\end{remark}

\begin{remark}[Relationship between Besov spaces and other function spaces]
    The flexibility of the Besov space definition as we vary $\alpha,p,q$ also allows them to act as a unifying framework for other commonly used function spaces.
    In particular, the spatially homogeneous H{\"o}lder and Sobolev spaces are represented as $C^\alpha([0,1]) = B^\alpha_{\infty,\infty}([0,1])$ and $W^{\alpha,p}([0,1]) = B^\alpha_{p,p}([0,1])$ respectively for $0 < \alpha < 1$, $1 < p < \infty$.
    We also have the following sandwich relationship with bounded variation functions: $B^1_{1,1}([0,1]) \subset BV([0,1]) \subset B^1_{1,\infty}([0,1])$.
\end{remark}

Next, we introduce notation to describe sparsity constraints.
For any vector $\bx \in \R^d$ and subset of indices $S \subset [d]$, we let $\bx_S$ denote the restriction of $\bx$ to the indices in $S$.
For any function class $\mathcal{F}$, let $\mathcal{F}_S$ denote the subclass of functions $f$ in $\mathcal{F}$ such that $f(\bx) = g(\bx_S)$ for some $g :\R^{|S|}\to \mathbb{R}$.

We now use anisotropic Besov balls together with sparsity constraints as building blocks to define the PSHAB space.
Specifically, this class partitions the covariate space $[0,1]^d$ into $B$ disjoint cells and imposes separate anisotropic Besov norm and sparsity constraints on each cell.
To formalize the collection of sparse index sets and smoothness parameters across cells, we define
$\widetilde{\mathscr{S}}\coloneq \{(S_1,\ldots,S_B): S_b \subset [d]\}$ and
$\widetilde{\mathscr{A}}\coloneq \{(\balpha_1,\ldots,\balpha_B): \balpha_b \in (0,1]^d\}$.

\begin{definition}[Piecewise sparse heterogeneous anisotropic Besov space]
Given a partition $\partition_* = \{G_b\}_{b=1}^B$ of $[0,1]^d$, parameters $0 < p, q \leq \infty$, and $\boldsymbol{\Lambda}=(\Lambda_1,\ldots,\Lambda_B) \in \R^B_+$, consider $\boldsymbol{S}=(S_1,\ldots,S_B)\in\widetilde{\mathscr{S}}$ and $\boldsymbol{A}=(\balpha_1,\ldots,\balpha_B)\in \widetilde{\mathscr{A}}$. We define
\begin{equation*}
\mathcal{B}_{p,q}^{\boldsymbol{S},\boldsymbol{A}}(\partition_*,\bLambda)
\coloneq
\braces*{f \in L^p([0,1]^d): f|_{G_b} \in \paren*{B_{p,q}^{\boldsymbol{\alpha}_b}(G_b,\Lambda_b)}_{S_b}}.
\end{equation*}
For $\mathscr{S}\subseteq\widetilde{\mathscr{S}}$ and $\mathscr{A}\subseteq\widetilde{\mathscr{A}}$, we then define the piecewise sparse heterogeneous anisotropic Besov space as
\begin{equation*} 
\mathcal{B}_{p,q}^{\mathscr{S},\mathscr{A}}(\partition_*,\boldsymbol{\Lambda}) \coloneq
\bigcup_{\substack{\boldsymbol{S}\in\mathscr{S},\\ \boldsymbol{A}\in\mathscr{A}}}
\mathcal{B}_{p,q}^{\boldsymbol{S},\boldsymbol{A}}(\partition_*,\bLambda).
\end{equation*}
\end{definition}

\begin{remark}[Motivation for PSHAB spaces]
    Although anisotropic Besov spaces already comprise anisotropic and spatially inhomogeneous functions, they do not yet capture the full range of flexibility afforded by regression trees.
    Indeed, anisotropic Besov spaces still enforce the same directionality of anisotropy and potentially the same sparsity pattern across the entire covariate space.
    Decision trees, however, follow a divide and conquer strategy and can adapt to the sparsity, anisotropy, and other structure on each cell of a partition independently of all other cells.
    Such behavior is more accurately captured by demonstrating minimax adaptation to PSHAB spaces.
\end{remark}

\begin{remark}[Comparisons with related definitions]
    Our definition is similar to, but generalizes, two definitions occurring in recent work analyzing posterior contraction rates for Bayesian trees.
    In comparison to \citet{liu2024spatial}'s construction of what they call ``region-wise'' anisotropic Besov spaces, PSHAB adds additional sparsity constraints on each piece.
    In comparison to \citet{jeong2023art}'s construction of sparse piecewise heterogeneous anisotropic H{\"o}lder spaces, PSHAB relaxes the H{\"o}lder condition and allows the sparsity pattern to vary across pieces.
    Furthermore, in comparison to both definitions, PSHAB allows heterogeneity in the Besov norm constraint on each piece.
\end{remark}

\section{Approximation bounds over PSHAB spaces}
\label{section:approximation}

In Section~\ref{section:oracle-ineq}, our oracle inequalities established that the generalization error of ERM trees is fundamentally constrained by the tree approximation error, $\Excess_{\operatorname{reg},L}$ and $\Excess_{\operatorname{cls},L}$. 
Having introduced the PSHAB space in Section~\ref{section:function-space-PSHAB} as a natural model for spatially heterogeneous and anisotropic data, our next step is to quantify this approximation error for target functions belonging to this class.
To make the statements of the results in the remainder of this paper more concise, we first enumerate some assumptions and conditions for use later.

\begin{assumption}[Bounded density]
\label{assum:bound-density}
The covariate distribution $\mu_\bX$ is absolutely continuous with respect to Lebesgue measure with density $p_\bX$.
Furthermore, one of the following two conditions hold:
\begin{enumerate}
    \item[(i)] There exist a constant $ c_{\max} > 0$ such that $p_\bX(\bx) \leq c_{\max}$ for all $\bx \in [0,1]^d$.
    \item[(ii)] There exist constants $c_{\min}, c_{\max} > 0$ such that $c_{\min} \leq p_\bX(\bx) \leq c_{\max}$ for all $\bx \in [0,1]^d$.
\end{enumerate}
\end{assumption}

\begin{assumption}[PSHAB parameter regularity]
\label{assum:PSHAB_parameters}
The PSHAB space $\PSHABn$ is specified by parameters
$\mathscr{S}\subseteq\widetilde{\mathscr{S}}$,
$\mathscr{A}\subseteq\widetilde{\mathscr{A}}$,
$B\in\mathbb{N}$,
$\boldsymbol{\Lambda}=(\Lambda_1,\ldots,\Lambda_B)\in\mathbb{R}_+^B$,
and a tree-based partition $\partition_*=\{G_b\}_{b=1}^B$.
We further define the following quantities:
\begin{equation}
\begin{split}
    s
    &\coloneq
    s(\mathscr{S})
    \;=\;
    \sup\bigl\{|S_b| : (S_1,\ldots,S_B)\in \mathscr{S},\; b\in[B]\bigr\}, \\[0.4em]
    \alpha_{\min}
    &\coloneq
    \alpha_{\min}(\mathscr{A})
    \;=\;
    \inf\bigl\{\underline{\alpha_b} :
    (\balpha_1,\ldots,\balpha_B)\in \mathscr{A},\; b\in[B]\bigr\}, \\[0.4em]
    \bar{\alpha}
    &\coloneq
    \bar{\alpha}(\mathscr{S},\mathscr{A})
    \;=\;
    \inf\bigl\{H(S_b,\balpha_b) :
    (S_1,\ldots,S_B)\in \mathscr{S},\;
    (\balpha_1,\ldots,\balpha_B)\in \mathscr{A},\;
    b\in[B]\bigr\}.
    \end{split}
\end{equation}
Here, for any index set $S\subseteq[d]$ and any smoothness vector
$\balpha=(\alpha_1,\ldots,\alpha_d)$, we define
$\underline{\alpha}\coloneq \min_{k\in[d]} \alpha_k$,
and the harmonic mean of $\balpha$ over $S$ by
$H(S,\balpha)
\coloneq
((1/|S|)\sum_{k\in S}(1/\alpha_k))^{-1}$.
In addition, we assume $0 < p, q \leq \infty$, with the pair $(p,q)$ further satisfying one of the following conditions:
\begin{enumerate}
\item[(i)] $p > (\bar\alpha/s+1/2)^{-1}$;
\item[(i')] $p > (\bar\alpha/s+1/2)^{-1}$, with the additional restriction that $q \leq p$ if $1 <p < 2$;
\item[(ii)]  $p > (\bar\alpha/s+1)^{-1}$.
\end{enumerate}
\end{assumption}

\begin{definition}[Auxiliary quantities]\label{def:quantities}
Under the parameters specified in Assumption \ref{assum:PSHAB_parameters}, we define the following quantities.
\begin{equation}
\begin{split}
\boldsymbol{v}_1
&\coloneq
\boldsymbol{v}_1(p,\bLambda,\partition_*)
\;=\;
\bigl(\Lambda_1^2 |G_1|^{1-2/p},\ldots,\Lambda_B^2 |G_B|^{1-2/p}\bigr), \\[0.3em]
\boldsymbol{v}_2
&\coloneq
\boldsymbol{v}_2(p,\bLambda,\partition_*)
\;=\;
\bigl(\Lambda_1 |G_1|^{1-1/p},\ldots,\Lambda_B |G_B|^{1-1/p}\bigr), \\[0.3em]
\boldsymbol{v}_3
&\coloneq
\boldsymbol{v}_3(p,\bLambda,\partition_*)
\;=\;
\bigl(\Lambda_1 |G_1|^{-1/p},\ldots,\Lambda_B |G_B|^{-1/p}\bigr). \\[0.4em]
\end{split}
\end{equation}
\end{definition}

\begin{remark}
In Assumption \ref{assum:PSHAB_parameters}, two different ranges of the parameter $p$ are considered. Specifically, Assumption \ref{assum:PSHAB_parameters}(i) and (i') corresponds to the regression setting, whereas Assumption \ref{assum:PSHAB_parameters}(ii) pertains to the classification setting. Accordingly, the quantity $\boldsymbol{v}_1$ defined in Definition \ref{def:quantities} is used in the analysis of regression, while $\boldsymbol{v}_2$ and $\boldsymbol{v}_3$ are used in the analysis of classification.
\end{remark}

We now state the approximation results.
The first theorem establishes the rate for regression trees, while the second theorem establishes the approximation rate for classification trees, accounting for the Tsybakov margin parameter $\rho$.

\begin{theorem}[Regression approximation]\label{thm:approx-PSHAB-reg}
    In the setting of Theorem \ref{thm:regression}, suppose that $f^*\in\PSHABn$, and grant Assumption \ref{assum:bound-density}(i) and Assumption \ref{assum:PSHAB_parameters}(i'). 
    Then if $\leaves \geq 2B$, the approximation error satisfies
    \begin{equation}\label{eq:approx-reg-conclude-1}
    \Excess_{\operatorname{reg},\leaves}
    \lesssim_{s,\alpha_{\min},\bar\alpha,c_{\max}}\,
    \norm{\boldsymbol{v}_1}_{\frac{s}{s+2\bar\alpha}}\,\leaves^{-2\bar\alpha/s}.
    \end{equation}
\end{theorem}

\begin{theorem}[Classification approximation]\label{thm:approx-PSHAB-cls}
    In the setting of Theorem~\ref{thm:classification}, suppose $\eta\in\PSHABn$, and grant Assumption~\ref{assum:bound-density}(i) and Assumption~\ref{assum:PSHAB_parameters}(ii). 
    Then if $\leaves \geq 2B$, the following statements hold:
    \begin{enumerate}
        \item[(i)] If $\rho=0$, the approximation error satisfies
        \begin{equation}\label{eq:approx-cls-conclude-1}
        \Excess_{\operatorname{cls},\leaves}
        \lesssim_{s,\alpha_{\min},\bar\alpha,c_{\max}}
        \norm{\boldsymbol{v}_2}_{\frac{s}{s+\bar\alpha}}
        \,\leaves^{-\bar\alpha/s}.
        \end{equation}
        \item[(ii)] If $\rho> 0$ and we further assume $s/\bar\alpha<p\leq\infty$ and $0<q\leq p$, the approximation error satisfies
        \begin{equation}\label{eq:approx-cls-conclude-2}
        \Excess_{\operatorname{cls},\leaves}
        \lesssim_{s,\alpha_{\min},\bar\alpha,M,\rho,c_{\max}}
        \norm{\boldsymbol{v}_3}_{\frac{s}{\bar\alpha}}^{\rho+1}
        \,\leaves^{-(\rho+1)\bar\alpha/s}.
        \end{equation}
    \end{enumerate}
\end{theorem}
\begin{remark}[Comparing \eqref{eq:approx-cls-conclude-1} and \eqref{eq:approx-cls-conclude-2} when $\rho=0$]
\label{rmk:rho=0}
Since $\sum_{b=1}^B |G_b| = 1$, Hölder's inequality yields $\|\boldsymbol{v}_2\|_{\frac{s}{s+\bar\alpha}} 
\le 
\|\boldsymbol{v}_3\|_{\frac{s}{\bar\alpha}}.$
Consequently, the right-hand side of \eqref{eq:approx-cls-conclude-1} is no larger than that of \eqref{eq:approx-cls-conclude-2}. 
At first glance, this may appear counterintuitive, as \eqref{eq:approx-cls-conclude-1} applies to a broader function class than \eqref{eq:approx-cls-conclude-2}. The explanation is that the upper bound for the PSHAB class derived via Tsybakov's noise condition \ref{assum:tsybakov} is not sharp in the degenerate case $\rho=0$.
\end{remark}

The formal proofs of these approximation guarantees are deferred to Appendix~\ref{appendix:proof-apprrox-PSHAB}. However, we briefly outline the main technical innovations required to establish them. 
First, while it is known that dyadic piecewise constant functions enjoy the optimal approximation rates \citep{akakpo2012adaptation} under the strictly fractional smoothness regime ($\alpha_j < 1$), we need to extend these results to the boundary case ($\alpha_j = 1$) via Besov space embedding theory.
Second, to accommodate the piecewise nature of the PSHAB space, we analyze the local approximation error on each structural piece independently.
Because $\Lambda_b$ and $|G_b|$ potentially vary across the $B$ pieces, the global approximation error cannot be bounded by a simple uniform grid. 
Instead, we solve for the optimal allocation of tree leaves to the $B$ pieces via constrained optimization. 


\section{Ideal spatial adaptation}
\label{section:ideal_spatial_adaption}

We are now ready to establish the main statistical guarantees of this paper. By combining the data-driven estimation bounds provided by our oracle inequalities (Section~\ref{section:oracle-ineq}) with the structural approximation bounds over PSHAB spaces (Section~\ref{section:approximation}), we derive explicit generalization upper bounds for our ERM tree estimators. 
Crucially, we will show that these estimators automatically adapt to the underlying sparsity, anisotropy, and spatial heterogeneity of the target function, achieving minimax optimal rates (up to logarithmic factors) without requiring prior knowledge of the PSHAB parameters.
The proofs of the results in this section are deferred to Appendix \ref{section:optimal_rates_proof}.

\subsection{Spatial adaptation for ERM regression trees}
\label{section:main results-PSHAB-reg}

\begin{theorem}[Upper bound on PSHAB for ERM regression trees]\label{thm:regression-exact-rate-PSHAB}
In the setting of Theorem~\ref{thm:regression}, suppose $f^*\in\PSHABn$, and grant Assumptions~\ref{assum:bound-density}(i) and~\ref{assum:PSHAB_parameters}(i'). 
 Let $n$ be sufficiently large, in that $n\geq N_1$ as defined in Remark~\ref{rem:min_sample_size_reg}.
Let $u\geq 0$ and $\lambda>0$ be such that 
\[C_1(M+K)^2(\log(nd)+u)/n\leq \lambda \leq C_2(M+K)^2(\log(nd)+u)/n\] for big enough positive constants $C_2>C_1 > 0$.  
Then, with probability at least $1-e^{-u}$,  we have
\begin{equation}\label{eq:reg-conclude-1}
\excess_{\operatorname{reg}}(\permit)
\lesssim_{s,\alpha_{\min},\bar\alpha,p,c_{\max}}\;
\norm{\boldsymbol{v}_1}_{\frac{s}{s+2\bar\alpha}}^{\frac{s}{s+2\bar\alpha}}\;
\paren*{\frac{(M+K)^2(\log(nd)+u)}{n}}^{\frac{2\bar\alpha}{s+2\bar\alpha}}.
\end{equation}

\end{theorem}

\begin{remark}[Minimum sample size constraints of Theorem \ref{thm:regression-exact-rate-PSHAB}]
    \label{rem:min_sample_size_reg}
    In the setting of Theorem \ref{thm:regression-exact-rate-PSHAB}, define $N_1$ to be the smallest integer $N$ such that for all $n \geq N$,
\begin{equation}\label{eq:N_1}
        n\geq  
        C\max\braces*{\paren*{\frac{\norm{\boldsymbol{v}_1}_{\frac{s}{s+2\bar\alpha}}}{(M+K)^2(\log(nd)+u)}}^{\frac{s}{2\bar\alpha}}
       ,\, \frac{(M+K)^2(\log(nd)+u)}{\norm{\boldsymbol{v}_1}_{\frac{s}{s+2\bar\alpha}}}\,B^{\frac{s+2\bar\alpha}{s}}}.
\end{equation}
\end{remark}

Since \eqref{eq:reg-conclude-1} holds for arbitrary $\partition_*$ and $\bLambda$, sharper or more explainable upper bounds can be obtained under suitable regularity conditions on the partition $\partition_*$ and the Besov norm vector $\bLambda$. We provide two examples illustrating the application of \eqref{eq:reg-conclude-1}.

In Example \ref{exam:1}, we apply H\"older's inequality to show that the generalization upper bound can be explicitly controlled by the norm of $\bLambda$ and the number of cells $B$.
\begin{example}[Refer Eq.\eqref{eq:reg-conclude-1} ]\label{exam:1}
Assume $p\geq 2$. H{\"o}lder's inequality yields
$\norm{\boldsymbol{v}_1}_{\frac{s}{s+2\bar\alpha}}
\leq
\norm{\bLambda}_{\frac{ps}{s+p\bar\alpha}}^2$ for any $\{G_b\}_{b=1}^B$.
Consequently,
\begin{equation}\label{eq:reg-conclude-2}
\excess_{\operatorname{reg}}(\permit)
\lesssim_{s,\alpha_{\min},\bar\alpha,p,c_{\max}}\,
\norm{\bLambda}_{\frac{ps}{s+p\bar\alpha}}^{\frac{2s}{s+2\bar\alpha}}\,
\paren*{\frac{(M+K)^2(\log(nd)+u)}{n}}^{\frac{2\bar\alpha}{s+2\bar\alpha}}.
\end{equation}
Moreover, by Jensen's inequality
$\norm{\bLambda}_{\frac{ps}{s+p\bar\alpha}}^{\frac{2s}{s+2\bar\alpha}}
\leq
\norm{\boldsymbol{\besov}}_{\infty}^{\frac{2s}{s+2\bar\alpha}}
B^{1+\paren*{\frac{2}{p}-1}\frac{s}{s+2\bar\alpha}}$,
we obtain the explicit bound
\begin{equation}\label{eq:reg-conclude-3}
\excess_{\operatorname{reg}}(\permit)
\lesssim_{s,\alpha_{\min},\bar\alpha,p,c_{\max}}\,
\norm{\boldsymbol{\besov}}_{\infty}^{\frac{2s}{s+2\bar\alpha}}
B^{1+\paren*{\frac{2}{p}-1}\frac{s}{s+2\bar\alpha}}\,
\paren*{\frac{(M+K)^2(\log(nd)+u)}{n}}^{\frac{2\bar\alpha}{s+2\bar\alpha}}.
\end{equation}
\end{example}

\begin{remark}[Dependence on $p$ and $q$]\label{rmk:dependence-p}
The Besov space regularity parameters $p$ and $q$ do not affect the upper bound's rate in $n$.
On the other hand, $p$ affects the upper bound's dependence on the size and heterogeneity of the partition via the definition of $\mathbf{v}_1$, the norm of $\bLambda$ in \eqref{eq:reg-conclude-2} or the exponent of $B$ in \eqref{eq:reg-conclude-3}.
\end{remark}

In Example \ref{exam:2}, we assume that $\Lambda_b\asymp|G_b|^{1/p}$ for all $1\leq b\leq B$. This regularity condition requires the local Besov norms on the pieces $\{G_b\}_{b\in[B]}$ to scale at the same order as $|G_b|^{1/p}$. A trivial example is the constant function $f\equiv c$, for which $\Lambda_b=\norm*{f|_{G_b}}_{B_{p,q}^{\boldsymbol{\alpha}_b}(G_b)}=c|G_b|^{1/p}.$
See Remark \ref{rmk:Plambda-1} for further discussion.

\begin{example}[Refer Eq.\eqref{eq:reg-conclude-1}]\label{exam:2}
    Let $C>1$. Suppose that
    $C^{-1}|G_b|^{1/p}\leq \Lambda_b \leq  C|G_b|^{1/p}$, $
    \forall\,1\le b\le B$. Then by H\"older's inequality, 
    $\norm{\boldsymbol{v}_1}_{\frac{s}{s+2\bar\alpha}}^{\frac{s}{s+2\bar\alpha}}
    \lesssim
    \norm{\boldsymbol{\besov}}_{p}^{\frac{2s}{s+2\bar\alpha}}B^{\frac{2\bar\alpha}{s+2\bar\alpha}}$,
    and hence
    \begin{equation}\label{eq:reg-conclude-4}
    \excess_{\operatorname{reg}}(\permit)
    \lesssim_{s,\bar\alpha,\alpha_{\min},p,c_{\max}}
    \norm{\boldsymbol{\besov}}_{p}^{\frac{2s}{s+2\bar\alpha}}
    \paren*{\frac{B(M+K)^2 (\log(nd)+u)}{n}}^{\frac{2\bar\alpha}{s+2\bar\alpha}}.
    \end{equation}
\end{example}

\begin{remark}[Understanding assumptions on $(\partition_*,\bLambda)$]\label{rmk:Plambda-1}
    If $\norm*{f|_{G_b}}_{B_{p,q}^{\boldsymbol{\alpha}_b}(G_b)} \leq \Lambda_b$ and we let $\tilde f$ be the affine extension of $f|_{G_b}$ to $[0,1]^d$, then we have $\norm{\tilde f}_{B_{p,q}^{\boldsymbol{\alpha}_b}([0,1]^d)} \leq \abs{G_b}^{-1/p}\Lambda_b$.
    The additional assumption on $\partition_*$ and $\bLambda$ in Example \ref{exam:2} can hence be interpreted as saying that the affine extensions of each component of $f$ have similar norms, therefore enforcing a type of homogeneity for the function $f$.
\end{remark}

To assess the optimality of our upper bounds \eqref{eq:reg-conclude-1}, we next establish minimax lower bounds for regression with PSHAB spaces.
\begin{definition}[Minimax risk]
    \label{def:minimax_risk}
    Consider the setting of Section \ref{section:setting} and in addition assume Gaussian noise, i.e. that $\xi \sim \mathcal{N}(0,K^2)$.
    Recall that for any function space $\mathcal{F}$, the $L^2(\mu)$-minimax risk for regression over $\mathcal{F}$ is defined as 
    \begin{equation*}
        \mathcal{M}_{\operatorname{reg,n}}(\funcclass) \coloneqq \inf_{\hat f}\sup_{f^*\in\funcclass} \E\braces*{\excess_{\operatorname{reg}}(\hat f;\data)},
    \end{equation*}
    where the expectation is taken over $\data$ and the infimum is taken over all estimators, that is measurable functions $\hat f(-;-)$ whose first input is a point $\bx \in [0,1]^d$ and whose second input is a labeled dataset $\data$ of size $n$.
\end{definition}

\begin{theorem}[Minimax lower bound under regression]
\label{thm:minimax}
In the setting of Definition \ref{def:minimax_risk}, suppose that Assumption \ref{assum:bound-density}(ii) and Assumption \ref{assum:PSHAB_parameters}(ii) hold. 
Assume there exists a constant $C>0$ such that $\ell_j(G_b)^{s/\bar\alpha}\geq C$ for all $j\in[d]$ and $b\in[B]$, and $C^{-1}|G_b|^{1/p-1/2}\leq \Lambda_b\leq  C|G_b|^{1/p-1/2}$
for all $b\in [B]$. If there exist sequences $(S_1, \ldots, S_B)$ in $\mathscr{S}$ and $(\balpha_1, \ldots, \balpha_B)$ in $\mathscr{A}$ such that $|S_b| = s$ and $H(S_b, \balpha_b) = \bar\alpha$ for all $b\in[B]$, then 
\begin{equation}\label{eq:minimax-lower-1}
    \mathcal{M}_{\operatorname{reg,n}}\paren*{\PSHABn}
    \gtrsim_{s,\bar\alpha,c_{\min},c_{\max},K}
    \norm{\boldsymbol{v}_1}_{\frac{s}{s+2\bar\alpha}}^{\frac{s}{s+2\bar\alpha}}
    \; n^{-\frac{2\bar\alpha}{s+2\bar\alpha}}.
\end{equation}
\end{theorem}

Comparing Theorem \ref{thm:minimax} with Theorem \ref{thm:regression-exact-rate-PSHAB}, we see that for fixed choices of $\mathscr{A},\mathscr{S}, p, q,\bLambda,\partition_*$, ERM regression trees achieve the minimax rate in terms of $n$ and $\boldsymbol{v}_1$ up to logarithmic factors.

\begin{remark}[Related minimax theory]
    It is known that the minimax rate for anisotropic Besov spaces (up to log factors) is $n^{-2\bar{\alpha}/(d+2\bar{\alpha})}$ and that it can be achieved by locally adaptive kernel estimators \citep{kerkyacharian2001nonlinear}, wavelet thresholding estimators \citep{neumann2000multivariate}, and deep learning methods \citep{suzuki2021deep}.
    \citet{jeong2023art} derived the minimax rate for sparse piecewise heterogeneous anisotropic H{\"o}lder spaces, i.e. for $p=q=\infty$, and showed that it can be achieved by Bayesian CART and forests under the assumption that $B = O(1)$.
\end{remark}

\begin{remark}[Regularity of $(\partition_*,\bLambda)$]
\label{remark:minimax-reg-P_*}
The condition $\ell_j(G_b)^{s/\bar\alpha}\geq C$ in Theorem \ref{thm:minimax} excludes partitions $\partition_*$ that contain excessively small cells. The additional requirement $\Lambda_i\asymp |G_i|^{1/p-1/2}$ ensures that the components of $\boldsymbol{v_1}$ are comparable. 
In particular, suppose that $|G_b|\asymp B^{-1}$, $\ell_j(G_b)\asymp B^{-1/d}$ for all $b\in[B]$ and $j\in[d]$, and $\Lambda_i\asymp \Lambda_j$ for $1\leq i,j\leq B$. Under these conditions, the regularity requirements on $\partition_*$ and $\bLambda$ hold whenever $\log B\lesssim d\bar\alpha/s$.

The assumption that $\partition_*$ is tree-based may be relaxable; see \citet{jeong2023art}.
\end{remark}

\begin{remark}[Combinatorial term and sample size]
    In the context of minimax estimation over sparse function classes, the risk bound typically includes an additional term of order $\frac{s\log(d/s)}{n}$ \citep{raskutti2012minimax}. This term arises by considering the combinatorial entropy of the support set, specifically $\log \binom{d}{s}$. For the PSHAB class, where supports are selected independently across $B$ blocks, the corresponding term is expected to scale with $\frac{1}{n}\log \binom{d}{s}^B \asymp \frac{Bs\log(d/s)}{n}$. While our current construction focuses on the smoothness term and does not explicitly capture this combinatorial factor, we conjecture that the full minimax rate should indeed include this additive term.

    Moreover, consider the homogeneous setting where $|G_b| \asymp 1/B$ for all $b \in [B]$. Focusing on the primary scaling with respect to $n$ and $B$, we omit logarithmic factors and the dependence on $\bLambda$. Under this simplification, the sample size requirement \eqref{eq:N_1} reduces to $n \gtrsim B^{1-2/p}$. The lower bound derived in \eqref{eq:minimax-lower-1} is of the order $ B^{1+\paren*{\frac{2}{p}-1}\frac{s}{s+2\bar\alpha}}n^{-\frac{2\bar\alpha}{s+2\bar\alpha}}.$
    A straightforward calculation reveals that, under the aforementioned sample size condition, this rate satisfies $B^{1+\paren*{\frac{2}{p}-1}\frac{s}{s+2\bar\alpha}}n^{-\frac{2\bar\alpha}{s+2\bar\alpha}} \gtrsim \frac{B}{n}.$
    This implies that in this regime, the non-parametric rate dominates the parametric term $B/n$, thereby confirming the optimality of our lower bound under the constraints discussed in Remark~\ref{rem:min_sample_size_reg}.
\end{remark}
\begin{remark}[Dependence on $s$ and $d$]
    The ambient dimension $d$ occurs in the upper bounds \eqref{eq:reg-conclude-1} and \eqref{eq:reg-conclude-2} only as a logarithmic factor.
    On the other hand, the dependence on the intrinsic dimension $s$ is in fact exponential, given our current proof techniques and without further assumptions.
    Nonetheless, when all smoothness parameters $\alpha_{bk}$, for $b \in [B], k \in [d]$, are strictly smaller than $1$, it is easy to show that the dependence on $s$ is linear.
    In this case, the optimal rate in $n$ is preserved even when $s$ is allowed to grow polylogarithmically.
\end{remark}

\begin{remark}[Choice of $\lambda$]
    Although Theorem \ref{thm:regression-exact-rate-PSHAB} seems to require oracle knowledge of an appropriate value for the regularization parameter $\lambda$, an appropriate value can be chosen using a held-out validation set.
    Using our uniform concentration results, one can show that this sample splitting procedure still provides the optimal rate \eqref{eq:reg-conclude-1}.
\end{remark}

\subsection{Spatial adaptation for ERM classification trees}
\label{section:main-result-PSHAB-class}

\begin{theorem}[Upper bound on PSHAB for ERM classification trees]\label{thm:classification-specific-rate-PSHAB}

In the setting of Theorem \ref{thm:classification}, suppose $\eta \in \PSHABn $, and grant Assumptions~\ref{assum:bound-density}(i) and~\ref{assum:PSHAB_parameters}(ii).
Let $u\geq 0$, $\lambda,\theta > 0$ be such that $ C_1((\log(nd)+u)/n)^{\theta}\leq \lambda \leq C_2((\log(nd)+u)/n)^{\theta}$ and $\theta=(1+\rho)/(2+\rho)$ for big enough constants $C_2 > C_1 > 0$.
Then with probability at least $1-e^{-u}$, the following hold:
\begin{enumerate}
        \item[(i)] If $\rho=0$ and $n$ is sufficiently large, i.e., $n\geq N_2$ as defined in Remark~\ref{rmk:min_sample_size_cls}, then
        \begin{equation}\label{eq:cls-conclude-1}
            \excess_{\operatorname{cls}}(\hat{f}_{\lambda,\theta})\lesssim_{s,\alpha_{\min},\bar\alpha,p,M,c_{\max}}
            \,\norm{\boldsymbol{v}_2}_{\frac{s}{s+\bar\alpha}}^{\frac{s}{s+2\bar\alpha}}
            \,\paren*{\frac{\log(nd)+u}{n}}^{\frac{\bar\alpha}{s+2\bar\alpha}}.
        \end{equation}
        \item[(ii)] If $\rho> 0$ and we further assume $s/\bar\alpha<p\leq\infty$ and $0<q\leq p$. If $n$ is sufficiently large, i.e. $n \geq N_3$ as defined in Remark \ref{rmk:min_sample_size_cls}, then,
    \begin{equation}\label{eq:cls-conclude-2}
            \excess_{\operatorname{cls}}(\hat{f}_{\lambda,\theta})\lesssim_{s,\alpha_{\min},\bar\alpha,p,\rho,M,c_{\max}}\,
            \norm{\boldsymbol{v}_3}_{\frac{s}{\bar\alpha}}^{\frac{(1+\rho)s}{s+(2+\rho)\bar\alpha}}\,\paren*{\frac{\log(nd)+u}{n}}^{\frac{(1+\rho)\bar\alpha}{s+(2+\rho)\bar\alpha}}.
    \end{equation}
    \end{enumerate}
\end{theorem}

\begin{remark}[Minimum sample size constraints of Theorem \ref{thm:classification-specific-rate-PSHAB}]\label{rmk:min_sample_size_cls}
    In the setting of Theorem \ref{thm:regression-exact-rate-PSHAB}, define $N_2$ to be the smallest integer $N$ such that for all $n \geq N$,
    \begin{equation*}
        n\geq  
        C\max\braces*{\paren*{\frac{\norm{\boldsymbol{v}_2}_{\frac{s}{s+\bar\alpha}}}{\log(nd)+u}}^{\frac{s}{2\bar\alpha}}
       ,\, \frac{\log(nd)+u}{\norm{\boldsymbol{v}_2}_{\frac{s}{s+\bar\alpha}}}\,B^{\frac{s+\bar\alpha}{s}}}.
    \end{equation*}
        Define $N_3$ to be the smallest integer $N$ such that for all $n \geq N$,
        \begin{equation*}
        n\geq  
        C\max\braces*{\paren*{\frac{\norm{\boldsymbol{v}_3}_{\frac{s}{\bar\alpha}}^{2+\rho}}{\log(nd)+u}}^{\frac{s}{(2+\rho)\bar\alpha}} 
       ,\, \frac{\log(nd)+u}{\norm{\boldsymbol{v}_3}_{\frac{s}{\bar\alpha}}^{2+\rho}}\,B^{\frac{s+(2+\rho)\bar\alpha}{s}}}.
    \end{equation*}
\end{remark}

A similar comparison between \eqref{eq:cls-conclude-1} and \eqref{eq:cls-conclude-2} in the case $\rho=0$ follows the same reasoning as in Remark \ref{rmk:rho=0}.
Analogous to Theorem \ref{thm:regression-exact-rate-PSHAB}, we present two examples illustrating the applications of \eqref{eq:cls-conclude-1} and \eqref{eq:cls-conclude-2}, respectively. Example \ref{exam:3} is established under the same conditions as in the regression setting considered in Examples \ref{exam:1} and \ref{exam:2}, corresponding to the trivial case of Tsybakov's condition \ref{assum:tsybakov}.

\begin{example}[Refer Eq.\eqref{eq:cls-conclude-1}]\label{exam:3}
When $\rho=0$, Tsybakov's noise condition in Assumption~\ref{assum:tsybakov} becomes vacuous, and the optimal choice of $\theta$ is $1/2$. Moreover, by H\"older's inequality, when $p\geq 1$ we have $\norm{\boldsymbol{v}_2}_{\frac{s}{s+\bar\alpha}}
\le
\norm{\bLambda}_{\frac{ps}{s+p\bar\alpha}}$
for any partition $(G_1,\ldots,G_B)$. It follows that
\begin{equation*}
\excess_{\operatorname{cls}}(\hat{f}_{\lambda,\theta})
\lesssim_{s,\alpha_{\min},\bar\alpha,p,M,c_{\max}}
\norm{\bLambda}_{\frac{ps}{s+p\bar\alpha}}^{\frac{s}{s+2\bar\alpha}}
\paren*{\frac{\log(nd)+u}{n}}^{\frac{\bar\alpha}{s+2\bar\alpha}}.
\end{equation*}
 Moreover, by Jensen's inequality
$\norm{\bLambda}_{\frac{ps}{s+p\bar\alpha}}^{\frac{s}{s+2\bar\alpha}}
\leq
\norm{\boldsymbol{\besov}}_{\infty}^{\frac{2s}{s+2\bar\alpha}}
B^{\frac{1}{2}+\paren*{\frac{1}{p}-\frac{1}{2}}\frac{s}{s+2\bar\alpha}}$,
we obtain the explicit bound
\begin{equation*}
\excess_{\operatorname{cls}}(\hat{f}_{\lambda,\theta})
\lesssim_{s,\alpha_{\min},\bar\alpha,p,M,c_{\max}}
\norm{\boldsymbol{\besov}}_\infty^{\frac{s}{s+2\bar\alpha}}
B^{\frac{1}{2}+\paren*{\frac{1}{p}-\frac{1}{2}}\frac{s}{s+2\bar\alpha}}
\paren*{\frac{\log(nd)+u}{n}}^{\frac{\bar\alpha}{s+2\bar\alpha}}.
\end{equation*}

Furthermore, suppose there exists a constant $C$ such that
$C^{-1}|G_b|^{1/p}\leq \Lambda_b \leq  C|G_b|^{1/p}$, $
    \forall\,1\leq b\leq B.$
Then by H\"older's inequality $\norm{\boldsymbol{v}_2}_{\frac{s}{s+\bar\alpha}}\lesssim \norm{\boldsymbol{\besov}}_{p}\,B^{\frac{\bar\alpha}{s}}$, and hence
\begin{equation*}
\excess_{\operatorname{cls}}(\hat{f}_{\lambda,\theta})
\lesssim_{s,\alpha_{\min},\bar\alpha,p,M,c_{\max}}
\norm{\boldsymbol{\besov}}_p^{\frac{s}{s+2\bar\alpha}}
\paren*{\frac{B(\log(nd)+u)}{n}}^{\frac{\bar\alpha}{s+2\bar\alpha}}.
\end{equation*}
\end{example}

In \eqref{eq:cls-conclude-2}, when $\rho>0$, if the measure of any cell $|G_b|$ tends to zero, then $\norm{\boldsymbol{v}_3}_{\frac{s}{\bar\alpha}}$ diverges. It is therefore natural to investigate the optimal regime of \eqref{eq:cls-conclude-2} under additional regularity conditions on $(\partition_*,\bLambda)$, as illustrated in Example \ref{exam:4}.

\begin{example}[Refer Eq.\eqref{eq:cls-conclude-2}]\label{exam:4}
When $\rho>0$, if one of the cell measures $|G_b|$ tends to zero, then $\norm{\boldsymbol{v}_3}_{\frac{s}{\bar\alpha}}$ diverges. Moreover, by H\"older's inequality, $\norm{\boldsymbol{v}_3}_{\frac{s}{\bar\alpha}}
\geq
\norm{\bLambda}_{\frac{ps}{s+p\bar\alpha}}$,
with equality when $|G_b| \propto \Lambda_b^{\frac{ps}{s+p\bar\alpha}}$. for $b=1,\ldots,B$. Therefore, if the partition $\partition_*$ satisfies $|G_b| \asymp\Lambda_b^{\frac{ps}{s+p\bar\alpha}}$ for $b=1,\ldots,B$, then
\begin{equation}\label{eq:cls-conclude-3}
\excess_{\operatorname{cls}}(\hat{f}_{\lambda,\theta})\lesssim_{s,\alpha_{\min},\bar\alpha,p,\rho,M,c_{\max}}
\,\norm{\bLambda}_{\frac{ps}{s+p\bar\alpha}}^{\frac{(1+\rho)s}{s+(2+\rho)\bar\alpha}}
\,\paren*{\frac{\log(nd)+u}{n}}^{\frac{(1+\rho)\bar\alpha}{s+(2+\rho)\bar\alpha}}.
\end{equation}
Since $\norm{\bLambda}_{\frac{ps}{s+p\bar\alpha}}\leq \norm{\boldsymbol{\besov}}_{\infty}\, B^{\frac{s+p\bar\alpha}{ps}}$, it follows that
\begin{equation}\label{eq:cls-conclude-4}
\excess_{\operatorname{cls}}(\hat{f}_{\lambda,\theta})\lesssim_{s,\alpha_{\min},\bar\alpha,p,\rho,M,c_{\max}}
\norm{\boldsymbol{\besov}}_{\infty}^{\frac{(1+\rho)s}{s+(2+\rho)\bar\alpha}}
B^{\frac{1}{p}\frac{(1+\rho)s}{s+(2+\rho)\bar\alpha}+\frac{(1+\rho)\bar\alpha}{s+(2+\rho)\bar\alpha}}
\paren*{\frac{\log(nd)+u}{n}}^{\frac{(1+\rho)\bar\alpha}{s+(2+\rho)\bar\alpha}}.
\end{equation}
\end{example}

If we impose the same regularity condition as in Example \ref{exam:2}, namely that $\Lambda_b\asymp|G_b|^{1/p}$ for all $b=1,\ldots,B$, then, together with $\sum_{b=1}^B|G_b|=1$, it follows that $\Lambda_b|G_b|^{-1/p}\asymp \norm{\bLambda}_p$. Consequently, each component of $\boldsymbol{v}_3$ is of order $\norm{\bLambda}_p$. See Example \ref{exam:5} for further details.

\begin{example}[Refer Eq.\eqref{eq:cls-conclude-2}]
\label{exam:5}
Suppose there exists a constant $C$ such that $C^{-1}|G_b|^{1/p}\leq \Lambda_b \leq  C|G_b|^{1/p}$, $
\forall\,1\le b\le B.$
Then $\norm{\boldsymbol{v}_3}_{\frac{s}{\bar\alpha}}\asymp \norm{\boldsymbol{\besov}}_{p}\,B^{\frac{\bar\alpha}{s}}$, and hence
\begin{equation}\label{eq:cls-conclude-5}
\excess_{\operatorname{cls}}(\hat{f}_{\lambda,\theta})\lesssim_{s,\alpha_{\min},\bar\alpha,p,\rho,M,c_{\max}}
\norm{\boldsymbol{\besov}}_{p}^{\frac{(1+\rho)s}{s+(2+\rho)\bar\alpha}}
\paren*{\frac{B(\log(nd)+u)}{n}}^{\frac{(1+\rho)\bar\alpha}{s+(2+\rho)\bar\alpha}}.
\end{equation}
\end{example}

Similar to the regression case, we can establish minimax lower bounds for classification with PSHAB spaces. 

\begin{definition}[Minimax risk]
    \label{def:minimax_risk-cls}
    Consider the classification setting of Section \ref{section:setting}.
    Then for any function space $\mathcal{F}$, the minimax risk for clssification over $\mathcal{F}$ is defined as 
    \begin{equation*}
        \mathcal{M}_{\operatorname{cls},n}(\funcclass) \coloneqq \inf_{\hat f}\sup_{\eta\in\funcclass} \E\braces*{\excess_{\operatorname{cls}}(\hat f;\data)},
    \end{equation*}
    where the expectation is taken over $\data$ and the infimum is taken over all classifiers, that is measurable functions $\hat f(-;-)$ whose first input is a point $\bx \in [0,1]^d$ and whose second input is a labeled dataset $\data$ of size $n$.
\end{definition}

\begin{theorem}[Minimax lower bound under classification]
\label{thm:minimax-cls} 
   In the setting of Definition~\ref{def:minimax_risk-cls}, grant Assumption \ref{assum:bound-density}(i) and Assumption \ref{assum:PSHAB_parameters}(ii). Assume there is a universal constant $C$ such that $\ell_j(G_b)\geq C^{-1}B^{-1/d}$ for any $j\in[d]$ and $b\in[B]$, $C^{-1}\leq \Lambda_i/\Lambda_j \leq C$ for all $1 \leq i,j \leq B$  and $\log B\leq C d/s$. If there exist sequences $(S_1, \ldots, S_B)$ in $\mathscr{S}$ and $(\balpha_1, \ldots, \balpha_B)$ in $\mathscr{A}$ such that $|S_b| = s$ and $(\balpha_b)_{S_b} = (\bar\alpha, \ldots, \bar\alpha)$ for all $b \in [B]$, then there is a constant $C_{s,\bar\alpha,\rho}$ such that for any $n\geq C_{s,\bar\alpha,\rho} B\norm{\boldsymbol{\bLambda}}_{\infty}^{s/\bar\alpha}$,
    \begin{equation}\label{eq:minimax-conclude-cls}
        \mathcal{M}_{\operatorname{cls},n}(\mathcal{B}_{\infty,\infty}^{\mathscr{S},\mathscr{A}}(\partition_*,\boldsymbol{\Lambda}))
        \gtrsim_{s,\bar\alpha,\rho,M,c_{\max}}
            \norm{\boldsymbol{\besov}}_{\infty}^{\frac{(1+\rho)s}{s+(2+\rho)\bar\alpha}} \paren*{B/n}^{\frac{(1+\rho)\bar\alpha}{s+(2+\rho)\bar\alpha}}.
    \end{equation}
\end{theorem}

\begin{remark}[Minimax lower bound for more general Besov space]
    When $p=q=\infty$, the Besov norm implies H{\"o}lder continuity. Moreover, for any $p\geq 1$, we have $\norm{f|_{G_b}}_{B_{p,\infty}^{\boldsymbol{\alpha}}(G_b)}\lesssim \norm{f|_{G_b}}_{B_{\infty,\infty}^{\boldsymbol{\alpha}}(G_b)} $ for any $b=1,\ldots,B$, and thus $\mathcal{B}_{\infty,\infty}^{\mathscr{S},\mathscr{A}}(\partition_*,\boldsymbol{\Lambda})\subseteq \mathcal{B}_{p,\infty}^{\mathscr{S},\mathscr{A}}(\partition_*,\boldsymbol{\Lambda})$. It follows that \eqref{eq:minimax-conclude-cls} also holds over $\mathcal{B}_{p,\infty}^{\mathscr{S},\mathscr{A}}(\partition_*,\boldsymbol{\Lambda})$.
\end{remark}

   It is straightforward to verify that the regularity condition $\ell_j(G_b)\gtrsim B^{-1/d}$ in Theorem \ref{thm:minimax-cls} implies $|G_b|\asymp B^{-1}$. Combined with the assumption $\Lambda_i\asymp \Lambda_j$, this matches the setting of Example \ref{exam:5}. 
    Comparing \eqref{eq:minimax-conclude-cls} and \eqref{eq:cls-conclude-5} shows that, when $p=q=\infty$, and for fixed $\bar\alpha$, $\alpha_{\min}$, and $s$, ERM classification trees achieve the minimax rate in terms of $n$, $B$, and $\bLambda$, up to logarithmic factors, provided that $\Lambda_1\asymp\cdots\asymp\Lambda_B$ and $\ell_j(G_b)\geq C^{-1}B^{-1/d}$ for all $j\in[d]$ and $b\in[B]$.
    We are currently unable to establish matching minimax lower bounds for other values of $p$ and $q$. Nevertheless, we conjecture that the bounds in \eqref{eq:cls-conclude-1} and \eqref{eq:cls-conclude-2} remain rate-optimal, analogous to the regression setting.


\begin{remark}[Related minimax theory]
    It is known that the minimax rate for isotropic Besov spaces (up to log factors) is $n^{-2(1+\rho)\bar\alpha/(d+(2+\rho)\bar\alpha)}$ and that it can be achieved by dyadic ERM trees \citep{AOS2014erm}.
    \citet{scott2006minimax} establish minimax rates for dyadic ERM trees under what they call ``box-counting'' complexity assumptions on the Bayes decision boundary, but it is unclear how their assumptions related to classical smoothness asumptions.
    We are unaware of any results that address either piecewise or anistropic versions of H{\"o}lder, Sobolev or Besov function spaces.
\end{remark}

\begin{remark}[Removing the bounded density assumption]
When $p>s/\bar\alpha$, the space $B_{p,q}^{\boldsymbol{\alpha}}([0,1]^d,\besov)$ is continuously embedded into $C([0,1]^d)$, the space of continuous functions \citep{suzuki2021deep}. In this regime, Assumption~\ref{assum:bound-density}(i) in Theorems~\ref{thm:regression-exact-rate-PSHAB} and \ref{thm:classification-specific-rate-PSHAB} is no longer needed.
\end{remark}

\section{Uniform concentration and derivation of oracle inequalities} 
\label{sec:oracle_inequalities_proofs}

Establishing uniform concentration is a central technical challenge in the analysis of adaptive tree-based estimators.
In order to obtain our sharp oracle inequalities, we develop a uniform concentration theory based on empirically localized Rademacher complexity \citep{bartlett2005local}. 
To set up the analysis, let $\partfunc_\partition^*$ denote the linear span of $\partfunc_\partition$ and $f^*$. We define the global function space of interest as $\partfunc_{L}^* = \cup_{\partition \in \partitions{L}} \partfunc_{\partition}^*$. 
Our proof strategy proceeds in five main steps:

\begin{enumerate}[(i)]
    \item \textbf{Empirical localization:} We first bound the empirical Rademacher complexity of the empirically localized tree function class, that is, the empirical Rademacher complexity of $\partfunc_{L}^*$ constrained to functions satisfying $\norm{f}_n \leq r$ for some radius $r > 0$. 
    Conditioned on the unlabeled dataset $\sample$, this function class is isometric to a union of $L$-dimensional Euclidean balls. By applying a union bound over the valid tree-based partitions (Lemma \ref{lem:card-partition-set}), we can bound this empirical complexity.   
    
    \item \textbf{Unconditional expected suprema:} Using symmetrization and contraction arguments, we replace empirical localization into localization under the population norm and obtain bounds on the local Rademacher complexity as well as expected suprema over the localized deviations of the empirical norms and the multiplier processes. 
    
    \item \textbf{High-probability bounds:} We then apply logarithmic Sobolev inequalities (specifically, Bousquet's inequality) to obtain sharp, high-probability deviation bounds for these process suprema.
    
    \item \textbf{Self-normalization via peeling:} The deviation bounds of these processes depend on the scale $r$ of the localization. We employ a peeling argument to obtain self-normalized bounds that hold for all $f \in \partfunc_{L}^*$ and which scale with the function's true $L^2$ norm and supremum norm.   
    
    \item \textbf{Risk decomposition:} Finally, we decompose the empirical excess risk deviation into terms comprising these empirical norm and multiplier processes, applying the self-normalized bounds to establish the final oracle inequalities for both regression and classification (Theorems~\ref{thm:regression} and \ref{thm:classification}).
\end{enumerate}

In the remainder of this section, we provide additional technical details for the proof.
Steps (i) and (ii) are deferred to Lemmas \ref{lem:subG_multiplier_sup_cond_exp} and \ref{lem:subG_multiplier_sup_exp} respectively in Appendix~\ref{section:proof-Rademacher}. 
The result of Step (iii) is stated as Lemma~\ref{lem:emp_process_suprema}, with its proof also deferred to Appendix~\ref{section:proof-Rademacher}.
We execute Steps (iv) and (v) in the main text, with the peeling argument detailed in Lemma~\ref{lem:final_conc}.
\begin{lemma}[Localized deviation bounds]
\label{lem:emp_process_suprema}
    Suppose that for any value $\bx \in [0,1]^d$, the conditional distribution of $\xi_i$ given $\bX_i=x$ has mean zero and sub-Gaussian norm bounded by $K$ for some $K > 0$.
    For any $L \in [n]$, $0 < r \leq 1$, $g\colon [0,1]^d\to\R$, and $u \geq 0$, the following deviation bounds hold with probability at least $1-e^{-u}$:
    \begin{equation} \label{eq:sq_norm_sup_deviation}
        \sup_{\substack{f \in \partfunc_\leaves^* \\
        \norm{f}_2 \leq r,\norm{f}_\infty \leq 1}} \abs*{\norm{f}_n^2- \norm{f}_2^2} \lesssim r\paren*{\frac{L\log(nd) + u}{n}}^{1/2} + \frac{L\log(nd) + u}{n},
    \end{equation}
    \begin{equation} \label{eq:subG_multiplier_sup_deviation}
        \sup_{\substack{f \in \partfunc_\leaves^* \\
        \norm{f}_2 \leq r,\norm{f}_\infty \leq 1}} \abs*{\inprod{f,\xi}_n} \lesssim K\paren*{r\paren*{\frac{L\log(nd) + u}{n}}^{1/2} + \frac{L\log(nd) + u}{n}},
    \end{equation}
    \begin{equation} \label{eq:function_multiplier_sup_deviation}
        \sup_{\substack{f \in \partfunc_\leaves^* \\
        \norm{f}_2 \leq r,\norm{f}_\infty \leq 1}} \abs*{\inprod{f,g}_n - \inprod{f,g}_\mu} \lesssim \norm{g}_\infty\paren*{r\paren*{\frac{L\log(nd) + u}{n}}^{1/2} + \frac{L\log(nd) + u}{n}}.
    \end{equation}
\end{lemma}


\begin{lemma}[Self-normalized deviation bounds]
\label{lem:final_conc}
    Under the same conditions as Lemma \ref{lem:emp_process_suprema}, for any $g\colon [0,1]^d\to\R$ and $u \geq 0$, with probability at least $1-e^{-u}$, the following hold for any $L \in [n]$ and $f \in \partfunc^*_L$:
    \begin{equation}
    \label{eq:sq_norm_final_conc}
        \abs*{\norm{f}_n^2- \norm{f}_2^2} \lesssim \norm{f}_\infty\paren*{\norm{f}_2\paren*{\frac{L\log(nd) + u}{n}}^{1/2} + \norm{f}_\infty\paren*{\frac{ L\log(nd) + u}{n}}},
    \end{equation}
    \begin{equation}
    \label{eq:subG_multiploer_final_conc}
        \abs*{\inprod{f,\xi}_n} \lesssim K\paren*{\norm{f}_2\paren*{\frac{L\log(nd) + u}{n}}^{1/2} + \norm{f}_\infty\paren*{\frac{ L\log(nd) + u}{n}}},
    \end{equation}
    \begin{equation}
    \label{eq:function_multiplier_final_conc}
        \abs*{\inprod{f,g}_n - \inprod{f,g}_\mu} \lesssim \norm{g}_\infty\paren*{\norm{f}_2\paren*{\frac{L\log(nd) + u}{n}}^{1/2} + \norm{f}_\infty\paren*{\frac{ L\log(nd) + u}{n}}}.
    \end{equation}
\end{lemma}

\begin{proof}
    To derive the conclusions of this lemma from Lemma \ref{lem:emp_process_suprema}, we use a ``peeling'' argument.
    We illustrate how to use this to prove \eqref{eq:sq_norm_final_conc}, with \eqref{eq:subG_multiploer_final_conc} and \eqref{eq:function_multiplier_final_conc} following similarly.
    For each $k, L \in [n]$, choose $r = e^{-k+1}$ and $u' \coloneqq u + 2\log(2n)$.
    Using Lemma \ref{lem:emp_process_suprema}, there is an event $\mathcal{A}_{k,L}$ with probability at least $1-2e^{-u'}$ such that \eqref{eq:sq_norm_sup_deviation} holds for these choices of $L$, $r$, and $u'$.
    Since $L\log(nd) + u' \leq 5L\log(nd) + u$, on this event, \eqref{eq:sq_norm_sup_deviation} holds (with a different $C$) even if we replace $u'$ with $u$.
    Now condition on the intersection $\mathcal{A} \coloneqq \cap_{k,L=1}^n \mathcal{A}_{k,L}$.
    By the union bound, the total error probability is at most
    \begin{equation}
        \P\braces{\mathcal{A}^c} \leq n^2e^{-u'} = n^2(2n)^{-2}e^{-u} \leq e^{-u}/4.
    \end{equation}
    Meanwhile, for any $L$, consider any $f \in \partfunc_L^*$.
    Set $\tilde{f} \coloneqq f/\norm{f}_\infty$.
    If $\norm{\tilde f}_2 \leq e^{-n+1}$, then by $\mathcal{A}_{n,L}$, we have
    \begin{equation}
    \label{eq:final_conc_helper1}
    \begin{split}
        \abs*{\norm{\tilde{f}}_n^2- \norm{\tilde{f}}_2^2} & \leq Ce^{-n+1}\paren*{\frac{L\log(nd) + u}{n}}^{1/2} + \frac{CL\log(nd) + u}{n} \\
        & \leq \frac{CL\log(nd) + u}{n},
    \end{split}
    \end{equation}
    as the second term on the right hand side is larger than the first term (after multiplying by a constant factor if necessary).
    Otherwise, set $k = \lfloor\log(1/\norm{\tilde f}_2)\rfloor + 1$.
    We have $\norm{\tilde{f}}_2 \leq e^{-k+1} \leq e\norm{\tilde{f}}_2$, which together with $\mathcal{A}_{k,L}$, implies that
    \begin{equation}
    \label{eq:final_conc_helper2}
        \abs*{\norm{\tilde{f}}_n^2- \norm{\tilde{f}}_2^2} \leq Ce\norm{\tilde{f}}_2\paren*{\frac{L\log(nd) + u}{n}}^{1/2} + \frac{CL\log(nd) + u}{n}.
    \end{equation}
    Finally, whichever of \eqref{eq:final_conc_helper1} or \eqref{eq:final_conc_helper2} holds, multiplying through by $\norm{f}_\infty^2$ gives \eqref{eq:sq_norm_final_conc}.
\end{proof}

\begin{proof}[Proof of Theorem \ref{thm:regression}]
    First, condition on the event for which the conclusions of Lemma \ref{lem:final_conc} hold.
    We define the empirical excess estimator of any estimator $f$ as $\widehat\excess_{\operatorname{reg}}(f)=\nnorm{f-Y}^2-\nnorm{\xi}^2$. It is evident that $\widehat\excess_{\operatorname{reg}}(f)= \nnorm{f-f^*}^2 -2\inprod{f - f^*,\xi}_n$.
    For any $f \in \partfunc_L$ with $\norm{f}_\infty \leq \fbound$, we therefore have
    \begin{equation}
    \begin{split}
        \abs*{\widehat\excess_{\operatorname{reg}}(f) - \excess_{\operatorname{reg}}(f)} & = \abs*{\nnorm{f-f^*}^2-\munorm{f-f^*}^2 - 2\inprod{f - f^*,\xi}_n} \\
        & \leq C(\norm{f-f^*}_\infty + K)\paren*{\norm{f-f^*}_2\paren*{\frac{L\log(nd) + u}{n}}^{1/2} + \norm{f-f^*}_\infty\paren*{\frac{ L\log(nd) + u}{n}}} \\
        & \leq \excess_{\operatorname{reg}}(f)^{1/2} \cdot C(\fbound + K)\paren*{\frac{L\log(nd) + u}{n}}^{1/2} + C(\fbound + K)^2\paren*{\frac{ L\log(nd) + u}{n}}.
    \end{split}
    \end{equation}
    Applying Young's inequality to the first term on the right hand side, we obtain the family of bounds
    \begin{equation}
    \label{eq:reg_theorem_helepr1}
        \abs*{\widehat\excess_{\operatorname{reg}}(f) - \excess_{\operatorname{reg}}(f)} \leq \delta \excess_{\operatorname{reg}}(f) + \frac{C(M+K)^2\paren*{L\log(nd) + u}}{\delta n}
    \end{equation}
    for $0 < \delta < 1$.

    Next, let $\tilde{f}_L$ denote the function achieving the infimum in $\eqref{eq:tree_partition_coefficient}$ (since $\partfunc$ is a closed set, this infimum is attained).
    It is easy to see that on each leaf $A$ of its partition, $\tilde{f}_L$ attains the value $\E\braces{f^*(\bX)~|~A}$, which implies that $\norm{\tilde{f}_L}_\infty \leq \norm{f^*}_\infty$.
    By the definition of $\hat{f}_L$, we therefore have $\widehat\excess_{\operatorname{reg}}(\hat f_L) \leq \widehat\excess_{\operatorname{reg}}(\tilde{f}_L)$.
    Combining this with \eqref{eq:reg_theorem_helepr1} gives
    \begin{equation}
        \excess_{\operatorname{reg}}(\hat f_\leaves) \leq \Excess_{\operatorname{reg},\leaves} + \delta\paren*{\excess_{\operatorname{reg}}(\hat f_\leaves) + \Excess_{\operatorname{reg},\leaves}} +  \frac{C(M+K)^2\paren*{\leaves\log(nd) + u}}{\delta n}.
    \end{equation}
    Rearranging this completes the proof of \eqref{eq:thm-statement:regression-2}.

    To prove \eqref{eq:thm-statement:regression-2}, continue to condition on the same event.
    Let $\hat \leaves$ denote the number of leaves of $f_\lambda$.
    For any $\leaves \in [n]$, we have
    \begin{equation}
        \widehat\excess_{\operatorname{reg}}(\hat f_\lambda) + \lambda\hat \leaves \leq \widehat\excess_{\operatorname{reg}}(\tilde f_\leaves) + \lambda\leaves.
    \end{equation}
    Combining this with \eqref{eq:reg_theorem_helepr1} as before gives
    \begin{equation}
        \excess_{\operatorname{reg}}(\hat f_\lambda) \leq \Excess_{\operatorname{reg},\leaves} + \delta\paren*{\excess_{\operatorname{reg}}(\hat f_\leaves) + \Excess_{\operatorname{reg},\leaves}} +  \frac{C(M+K)^2\paren*{(L + \hat L)\log(nd) + u}}{\delta n} + \lambda (L - \hat L).
    \end{equation}
    Using the assumption on $\lambda$ and rearranging completes the proof.
\end{proof}

\begin{proof}[Proof of Theorem \ref{thm:classification}]
    We define the empirical excess estimator of any estimator $f$ as $\widehat\excess_{\operatorname{cls}}(f)=\nnorm{f-Y}^2-\nnorm{f^*-Y}^2.$ It is evident that $\widehat\excess_{\operatorname{cls}}(f)=\inprod{1-2Y,f-f^*}_n$. We then can write
    \begin{equation} \label{eq:classification_theorem_helper1}
        \begin{split}
        \widehat\excess_{\operatorname{cls}}(f) - \excess_{\operatorname{cls}}(f) & = \inprod{1-2Y,f-f^*}_n - \inprod{1-2Y,f-f^*}_\mu \\
        & = -2\inprod{Y - \eta,f-f^*}_n + \inprod{1-2\eta,f-f^*}_n - \inprod{1-2\eta,f-f^*}_\mu.
        \end{split}
    \end{equation}
    As such, conditioning on the event on which the conclusions of Lemma \ref{lem:final_conc} hold, we get
    \begin{equation}
        \abs*{\widehat\excess_{\operatorname{cls}}(f) - \excess_{\operatorname{cls}}(f)} \leq \norm{f-f^*}_2^{1/2} \cdot C\paren*{\frac{L\log(nd) + u}{n}}^{1/2} + C\paren*{\frac{ L\log(nd) + u}{n}}.
    \end{equation}
    Next, by Proposition 1 in \citet{tsybakov2004optimal}, we have
    \begin{equation}
        \norm{f-f^*}_2 \leq C_\rho \excess_{\operatorname{cls}}(f)^{\rho/(1+\rho)}.
    \end{equation}
    Applying Young's inequality with exponents $p=2(1+\rho)/\rho$ and $q = 2(1+\rho)/(2+\rho)$ to the first term in \eqref{eq:classification_theorem_helper1}, we get the family of bounds
    \begin{equation}
    \label{eq:classification_theorem_helper2}
        \abs*{\widehat\excess_{\operatorname{cls}}(f) - \excess_{\operatorname{cls}}(f)} \leq \delta\excess_{\operatorname{cls}}(f) + C_\rho\delta^{-\rho/(2+\rho)}\paren*{\frac{\leaves\log(nd) + u}{n}}^{(1+\rho)/(2+\rho)} + C\paren*{\frac{ \leaves\log(nd) + u}{n}}
    \end{equation}
    for $0 < \delta < 1$.
    Notice that the last term above is smaller than the second term, except when $\frac{L\log(nd) + u}{n} \geq 1$, in which case the claim is vacuous.
    Hence, it can be removed from the inequality.
    As before, for any $L$, we have $\widehat\excess_{\operatorname{cls}}(\hat f_L) \leq \widehat\excess_{\operatorname{cls}}(\tilde{f}_L)$.
    Combining this with \eqref{eq:classification_theorem_helper2} and rearranging completes the proof of \eqref{eq:thm-statement:classification-2}.
\end{proof}

\begin{remark}[Other uniform concentration strategies]
    \citet{syrgkanis2020estimation} seems to be the only existing work making use of local Rademacher complexity to derive uniform concentration for tree-based estimators.
    In particular, they study CART estimators in a binary feature setting.
    However, they neither make use of empirical localization, nor do they obtain self-normalized deviation bounds.
    \citet{chatterjee2021adaptive} obtain self-normalized concentration bounds, but only in a fixed design setting---since empirical averages do not have to be controlled, local Rademacher complexity can be avoided.
    Earlier work make use of more classical techniques such as VC dimension \citep{AOS2014erm} or covering numbers \citep{wager2015adaptive,chi2022asymptotic}.
    Such approaches are not only too coarse to obtain the self-normalized bounds required for our sharp oracle inequalities, but furthermore require imposing structural assumptions on the trees to control complexity, such as dyadic splits, bounded depth, balance conditions, or sparsity of splitting variables (e.g., \citet{blanchard2007optimal,chi2022asymptotic,Mazumder2024,klusowski2024large}). 
\end{remark}

\section{Heavier-tailed noise}
\label{section:heavy-tailed}

In this section, we extend the regression setting to accommodate heavier-tailed noise, contrasting with the sub-Gaussian assumptions in Theorem~\ref{thm:regression}. We provide refined versions of these results under the assumption that the noise lies in an Orlicz space $L^{\Phi}$ defined below.

\begin{definition}[Orlicz spaces]
A function $\Phi \colon [0, \infty) \to [0, \infty)$ is a \textit{Young function} if it is convex, strictly increasing, and satisfies $\Phi(0) = 0$ with $\lim_{t \to \infty} \Phi(t) = \infty$. 
Let $(\Omega, \mathcal{F}, \mathbb{P})$ be a probability space. For any real-valued random variable $X$, the \textit{Luxemburg norm} (relative to $\Phi$) is defined as
\[
\|X\|_\Phi = \inf \left\{ \lambda > 0 : \mathbb{E} \left[ \Phi \left( \frac{|X|}{\lambda} \right) \right] \le 1 \right\},
\]
where we define $\inf \emptyset = \infty$. The \textit{Orlicz space} $L^\Phi$ is the Banach space of random variables defined by
\[
L^\Phi = \left\{ X : \|X\|_\Phi < \infty \right\}.
\]
\end{definition}

\begin{definition}[$L^m$ and $L^{\psi_{\beta}}$ spaces]
Two fundamental special cases of Orlicz spaces are ubiquitous in statistical learning. Setting $\Phi(t) = t^m$ ($m \geq 1$) recovers the classical $L^m$ space, where the Luxemburg norm reduces to the standard $L^m$ norm. Alternatively, setting $\Phi(t) =\psi_{\beta}= \exp(t^\beta) - 1$ ($\beta \geq 1$) yields the exponential Orlicz space $L^{\psi_\beta}$, with the norm defined as $\|X\|_{\psi_\beta} = \inf \left\{ \lambda > 0 \mathrel{\Big|} \mathbb{E}\left\{ \exp\left(\frac{|X|^\beta}{\lambda^\beta}\right) - 1 \right\} \le 1 \right\}. $
The special cases $\beta=1$ and $\beta=2$ correspond to the standard spaces of sub-exponential and sub-Gaussian random variables, respectively.
\end{definition}

\begin{theorem}[Oracle inequality under heavier noise]
\label{thm:regression-heavy}
Assume the regression setting of Section~\ref{section:setting}, and let $\hat f_{\lambda}$ denote the penalized ERM regression tree estimator (Definition~\ref{def:erm_reg}). Suppose that $\|f^*\|_\infty \le \fbound$ and that, for every $\bx\in[0,1]^d$, the conditional distribution of $\xi$ given $\bX=\bx$ belongs to $L^{\Phi}$ for some Young function $\Phi:[0,\infty)\to[0,\infty)$. Let $p_0>0$ and define
\[
K= \sup_{\bx \in [0,1]^d}\norm{\xi|\bX=\bx}_{\Phi}\,\Phi^{-1}\!\left(n/p_0\right).
\]
Then there exists a universal constant $C > 0$ such that, for any $u>0$, with probability at least $1-e^{-u}-p_0$, the following bound holds for all $\lambda\geq C(\fbound+K)\,(\log(nd)+u)/(\delta n)$:
\begin{equation}
\excess_{\mathrm{reg}}(\hat f_{\lambda})
\le
\frac{1+\delta}{1-\delta}
\cdot
\min_{L\in[n]}
\left\{
\Excess_{\mathrm{reg},L}
+
2\lambda L
\right\}.
\label{eq:thm-statement:regression-heavy}
\end{equation}
\end{theorem}

\begin{theorem}
\label{thm:regression-exact-rate-heavy}
Under the same setting as Theorem~\ref{thm:regression-heavy}, suppose $f^* \in \PSHABn$, and grant Assumption~\ref{assum:bound-density}(i) and Assumption~\ref{assum:PSHAB_parameters}(i). There exists a constant $C_1$ such that, for any $u \geq 1$, with probability at least $1-e^{-u}-p_0$, the following holds: for any big enough positive constants $C_2 > C_1$ and any $\lambda > 0$ satisfying $C_1(M+K)^2(\log(nd)+u)/n \leq \lambda \leq C_2(M+K)^2(\log(nd)+u)/n$, the bounds from Theorem~\ref{thm:regression-exact-rate-PSHAB}(i) and (ii) hold simultaneously.
\end{theorem}

\begin{remark}[Generalization bounds under $L^{\psi_\beta}$ noise]
Let $\Phi(t)=\psi_{\beta}(t)=\exp(t^\beta)-1$ with $\beta\geq 1$, so that $\xi\mid\bX=\bx$ belongs to $L^{\psi_\beta}$. Taking $p_0=e^{-u}$, the bounds below hold.
\begin{itemize}
\item[(i)]
Under the conditions of Example~\ref{exam:1}, with probability at least $1-2e^{-u}$
\begin{equation*}
\excess_{\operatorname{reg}}(\permit)
\lesssim_{s,\alpha_{\min},\bar\alpha,p,\rho,M,c_{\max}}
\norm{\boldsymbol{\besov}}_{\infty}^{\frac{2s}{s+2\bar\alpha}}
B^{1+\left(\frac{2}{p}-1\right)\frac{s}{s+2\bar\alpha}}
\left(
\frac{\paren*{M+\|\xi\|_{\psi_\beta}(\log(nd)+u)^{1/\beta}}^2(\log(nd)+u))}{n}
\right)^{\frac{2\bar\alpha}{s+2\bar\alpha}}.
\end{equation*}

\item[(ii)]
Under the conditions of Example~\ref{exam:2}, with probability at least $1-2e^{-u}$
\begin{equation*}
\excess_{\operatorname{reg}}(\permit)
\lesssim_{s,\alpha_{\min},\bar\alpha,p,\rho,M,c_{\max}}
\norm{\boldsymbol{\besov}}_{p}^{\frac{2s}{s+2\bar\alpha}}
\left(
B\frac{\paren*{M+\|\xi\|_{\psi_\beta}(\log(nd)+u)^{1/\beta}}^2(\log(nd)+u))}{n}
\right)^{\frac{2\bar\alpha}{s+2\bar\alpha}}.
\end{equation*}
\end{itemize}
\end{remark}

\begin{remark}[Generalization bounds under $L^{\heavy}$ noise]
 Let $\Phi(t)=t^\heavy$ with $\heavy>2$, so that $\xi\mid\bX=\bx$ belongs to $L^m$. 
 For any $t > 0$, taking $p_0=t^{-1}\log^{-1}n$, the bounds below hold.
\begin{itemize}
    \item [(i)] Under the conditions of Example~\ref{exam:1}, with probability at least $1-e^{-u}-p_0$
    \begin{equation*}
    \excess_{\operatorname{reg}}(\permit)
    \lesssim_{s,\alpha_{\min},\bar\alpha,p,\rho,M,c_{\max}}
    \norm{\boldsymbol{\besov}}_{\infty}^{\frac{2s}{s+2\bar\alpha}}
    B^{1+\left(\frac{2}{p}-1\right)\frac{s}{s+2\bar\alpha}}
    \left(
    \frac{(M+\|\xi\|_\heavy)^2 t^{2/m}(\log (n))^{2/m}(\log(nd)+u)}{n^{1-2/\heavy}}
    \right)^{\frac{2\bar\alpha}{s+2\bar\alpha}}.
    \end{equation*}
    \item [(ii)] Under the conditions of Example~\ref{exam:2},  with probability at least $1-e^{-u}-p_0$
    \begin{equation*}
    \excess_{\operatorname{reg}}(\permit)
    \lesssim_{s,\alpha_{\min},\bar\alpha,p,\rho,M,c_{\max}}
    \norm{\boldsymbol{\besov}}_{p}^{\frac{2s}{s+2\bar\alpha}}
    \left(
    B\frac{(M+\|\xi\|_\heavy)^2 t^{2/m}(\log (n))^{2/m}(\log(nd)+u)}{n^{1-2/\heavy}}
    \right)^{\frac{2\bar\alpha}{s+2\bar\alpha}}.
    \end{equation*}
\end{itemize}
\end{remark}

From the explicit bounds above, under light-tailed noise in $L^{\psi_\beta}$, we obtain the rate:
\[
\widetilde{O}\!\left(
B^{1+\left(\frac{2}{p}-1\right)\frac{s}{s+2\bar\alpha}}
n^{-\frac{2\bar\alpha}{s+2\bar\alpha}}
\right)
\quad\text{or}\quad
\widetilde{O}\!\left(\paren*{B/n}^{\frac{2\bar\alpha}{s+2\bar\alpha}}\right),
\]
which is minimax optimal up to polylogarithmic factors. Under heavy-tailed noise in $L^\heavy$, the rate becomes:
\[
\widetilde{O}\!\left(
B^{1+\left(\frac{2}{p}-1\right)\frac{s}{s+2\bar\alpha}}
n^{-\frac{2(1-2/\heavy)\bar\alpha}{s+2\bar\alpha}}
\right)
\quad\text{or}\quad
\widetilde{O}\!\left((B/n)^{\frac{2(1-2/\heavy)\bar\alpha}{s+2\bar\alpha}}\right).
\]
These rates are consistent, recovering the light-tailed behavior since $1-2/\heavy\to1$ as $\heavy\to\infty$.

Although ERM trees do not achieve the optimal minimax rate under heavy-tailed noise \citep{han2018convergenceratessquaresregression}, they still attain a nontrivial convergence rate. To the best of our knowledge, this is the first result that explicitly characterizes how the tail index $\heavy$ affects the convergence behavior of tree-based estimators.

A closer inspection of the proof shows that the suboptimality under heavy-tailed noise arises from the difficulty of controlling the sample responses $\by=\{y_i\}_{i=1}^n$, rather than from the tree structure itself. Because standard ERM trees estimate values via simple leaf-averaging, they are inherently sensitive to extreme outliers. The loss in rate is therefore driven purely by variance inflation, not by approximation bias, and the resulting upper bounds are not governed by the usual nonparametric bias phenomena associated with smoothing or boundary effects \citep{cattaneo2022pointwise}. This highlights a clear methodological gap: recovering optimal minimax rates under heavy-tailed noise will likely require tree-building procedures that incorporate robust leaf evaluators, such as median-of-means or explicit response clipping, while preserving the spatial adaptivity of the partition.

\section{Conclusion}
\label{section:conclusion}

This work establishes a comprehensive theoretical framework for empirical risk minimization (ERM) decision trees within a random design setting. The findings sharply capture the accuracy-interpretability trade-off for trees and offer a rigorous explanation of the inherent ability of ERM trees to automatically adapt to sparsity, anisotropy, and spatial inhomogeneity, as captured by piecewise sparse heterogeneous anisotropic Besov (PSHAB) spaces.
The last section in our paper investigated the robustness of ERM trees to heavy-tailed noise, revealing potential degradation in performance.
This may be slightly concerning given the use of decision trees to model economic data, which is known to exhibit heavy-tailed behavior.
A nature direction for future work is thus modifying ERM trees such structure.
Finally, our uniform concentration framework can potentially be used to derive tighter generalization results for other tree-based algorithms such as CART and Random Forests, for which minimax results are currently unknown.

\subsection*{Acknowledgements}

SG was supported in part by the Singapore MOE Grants R-146-000-312-114,
A-8002014-00-00, A-8003802-00-00, E-146-00-0037-01 and A-8000051-00-00.
YT was supported by NUS Startup Grant A-8000448-00-00 and MOE AcRF Tier 1 Grant A-8002498-00-00.

\bibliography{00_refs}

@book{vershynin2018high,
  title     = {High-Dimensional Probability: An Introduction with Applications in Data Science},
  author    = {Vershynin, Roman},
  volume    = {47},
  year      = {2018},
  publisher = {Cambridge University Press}
}

@misc{han2018convergenceratessquaresregression,
  title         = {Convergence Rates of Least Squares Regression Estimators with Heavy-Tailed Errors},
  author        = {Han, Qiyang and Wellner, Jon A.},
  year          = {2018},
  eprint        = {1706.02410},
  archivePrefix = {arXiv},
  primaryClass  = {math.ST}
}

@inproceedings{Mazumder2024,
 author = {Mazumder, Rahul and Wang, Haoyue},
 booktitle = {Advances in Neural Information Processing Systems},
 pages = {57754--57782},
 publisher = {Curran Associates, Inc.},
 title = {On the Convergence of CART under Sufficient Impurity Decrease Condition},
 volume = {36},
 year = {2023}
}

@article{chi2022asymptotic,
author = {Chien-Ming Chi and Patrick Vossler and Yingying Fan and Jinchi Lv},
title = {{Asymptotic properties of high-dimensional random forests}},
volume = {50},
journal = {The Annals of Statistics},
number = {6},
publisher = {Institute of Mathematical Statistics},
pages = {3415 -- 3438},
year = {2022},
doi = {10.1214/22-AOS2234},
}

@book{pattern,
  title     = {A Probabilistic Theory of Pattern Recognition},
  author    = {Devroye, Luc and Györfi, László and Lugosi, Gábor},
  volume    = {31},
  year      = {2013},
  publisher = {Springer Science \& Business Media}
}

@book{breiman1984classification,
  title     = {Classification and Regression Trees},
  author    = {Breiman, Leo and Friedman, Jerome and Stone, Charles J. and Olshen, Richard A.},
  year      = {1984},
  publisher = {CRC Press},
  address   = {Belmont, CA}
}

@book{quinlan1993c45,
  title     = {C4.5: Programs for Machine Learning},
  author    = {Quinlan, J. Ross},
  year      = {1993},
  publisher = {Morgan Kaufmann Publishers},
  address   = {San Mateo, CA},
  isbn      = {1-55860-238-0}
}

@article{hyafil1976constructing,
  title     = {Constructing Optimal Binary Decision Trees is NP-Complete},
  author    = {Hyafil, Laurent and Rivest, Ronald L.},
  journal   = {Information Processing Letters},
  volume    = {5},
  number    = {1},
  pages     = {15--17},
  year      = {1976},
  publisher = {Elsevier},
  doi       = {10.1016/0020-0190(76)90095-8}
}

@article{jeong2023art,
  title   = {The Art of BART: Minimax Optimality over Nonhomogeneous Smoothness in High Dimension},
  author  = {Jeong, Seonghyun and Ročková, Veronika},
  journal = {Journal of Machine Learning Research},
  volume  = {24},
  number  = {337},
  pages   = {1--65},
  year    = {2023}
}

@inproceedings{liu2024spatial,
  title     = {Spatial Properties of Bayesian Unsupervised Trees},
  author    = {Liu, Linxi and Ma, Li},
  booktitle = {Proceedings of the Thirty-Seventh Conference on Learning Theory},
  pages     = {3556--3581},
  year      = {2024},
  publisher = {PMLR},
  series    = {Proceedings of Machine Learning Research},
  volume    = {247}
}

@article{suzuki2021deep,
  title   = {Deep Learning is Adaptive to Intrinsic Dimensionality of Model Smoothness in Anisotropic Besov Space},
  author  = {Suzuki, Taiji and Nitanda, Atsushi},
  journal = {Advances in Neural Information Processing Systems},
  volume  = {34},
  pages   = {3609--3621},
  year    = {2021}
}

@article{audibert2007fast,
  title     = {Fast Learning Rates for Plug-in Classifiers},
  author    = {Audibert, Jean-Yves and Tsybakov, Alexandre B.},
  journal   = {The Annals of Statistics},
  volume    = {35},
  number    = {2},
  pages     = {608--633},
  year      = {2007},
  publisher = {Institute of Mathematical Statistics},
  doi       = {10.1214/009053607000000688}
}

@article{bertsimas2017optimal,
  title     = {Optimal Classification Trees},
  author    = {Bertsimas, Dimitris and Dunn, Jack},
  journal   = {Machine Learning},
  volume    = {106},
  number    = {7},
  pages     = {1039--1082},
  year      = {2017},
  publisher = {Springer}
}

@article{verwer2019learning,
  title   = {Learning Optimal Classification Trees Using a Binary Linear Program Formulation},
  author  = {Verwer, Sicco and Zhang, Yingqian},
  journal = {Proceedings of the AAAI Conference on Artificial Intelligence},
  volume  = {33},
  number  = {01},
  pages   = {1625--1632},
  year    = {2019},
  doi     = {10.1609/aaai.v33i01.33011624}
}

@article{zhu2020scalable,
  title   = {A Scalable MIP-Based Method for Learning Optimal Multivariate Decision Trees},
  author  = {Zhu, Haoran and Murali, Pavankumar and Phan, Dzung and Nguyen, Lam and Kalagnanam, Jayant},
  journal = {Advances in Neural Information Processing Systems},
  volume  = {33},
  pages   = {1771--1781},
  year    = {2020}
}

@article{aghaei2021strong,
  title     = {Strong Optimal Classification Trees},
  author    = {Aghaei, Sina and G\'{o}mez, Andr\'{e}s and Vayanos, Phebe},
  journal   = {Operations Research},
  volume    = {73},
  number    = {4},
  pages     = {2223--2241},
  year      = {2025},
  publisher = {INFORMS},
  doi       = {10.1287/opre.2021.0034}
}

@inproceedings{bos2024piecewise,
  title     = {Piecewise Constant and Linear Regression Trees: An Optimal Dynamic Programming Approach},
  author    = {van den Bos, Mim and van der Linden, Jacobus G. M. and Demirović, Emir},
  booktitle = {International Conference on Machine Learning},
  year      = {2024}
}

@article{zhang2023optimal,
  title   = {Optimal Sparse Regression Trees},
  author  = {Zhang, Rui and Xin, Rui and Seltzer, Margo and Rudin, Cynthia},
  journal = {Proceedings of the AAAI Conference on Artificial Intelligence},
  volume  = {37},
  number  = {9},
  pages   = {11270--11279},
  year    = {2023},
  doi     = {10.1609/aaai.v37i9.26334}
}

@article{demirovic2022murtree,
  title   = {MurTree: Optimal Decision Trees via Dynamic Programming and Search},
  author  = {Demirović, Emir and Lukina, Anna and Hebrard, Emmanuel and Chan, Jeffrey and Bailey, James and Leckie, Christopher and Ramamohanarao, Kotagiri and Stuckey, Peter J.},
  journal = {Journal of Machine Learning Research},
  volume  = {23},
  number  = {26},
  pages   = {1--47},
  year    = {2022}
}

@inproceedings{lin2020generalized,
  title        = {Generalized and Scalable Optimal Sparse Decision Trees},
  author       = {Lin, Jimmy and Zhong, Chudi and Hu, Diane and Rudin, Cynthia and Seltzer, Margo},
  booktitle    = {International Conference on Machine Learning},
  pages        = {6150--6160},
  year         = {2020},
  organization = {PMLR}
}

@article{hu2019optimal,
  title   = {Optimal Sparse Decision Trees},
  author  = {Hu, Xiyang and Rudin, Cynthia and Seltzer, Margo},
  journal = {Advances in Neural Information Processing Systems},
  volume  = {32},
  year    = {2019}
}

@article{ nobel1996histogram,
author = {Andrew Nobel},
title = {{Histogram regression estimation using data-dependent partitions}},
volume = {24},
journal = {The Annals of Statistics},
number = {3},
publisher = {Institute of Mathematical Statistics},
pages = {1084 -- 1105},
year = {1996},
doi = {10.1214/aos/1032526958},
}

@article{carrizosa2021mathematical,
  title     = {Mathematical Optimization in Classification and Regression Trees},
  author    = {Carrizosa, Emilio and Molero-Río, Cristina and Romero Morales, Dolores},
  journal   = {Top},
  volume    = {29},
  number    = {1},
  pages     = {5--33},
  year      = {2021},
  publisher = {Springer}
}

@inproceedings{verwer2017learning,
  title        = {Learning Decision Trees with Flexible Constraints and Objectives Using Integer Optimization},
  author       = {Verwer, Sicco and Zhang, Yingqian},
  booktitle    = {Integration of AI and OR Techniques in Constraint Programming},
  pages        = {94--103},
  year         = {2017},
  organization = {Springer}
}

@inproceedings{narodytska2018learning,
  title     = {Learning Optimal Decision Trees with SAT},
  author    = {Nina Narodytska and Alexey Ignatiev and Filipe Pereira and Joao Marques-Silva},
  booktitle = {Proceedings of the Twenty-Seventh International Joint Conference on
               Artificial Intelligence, {IJCAI-18}},
  publisher = {International Joint Conferences on Artificial Intelligence Organization},
  pages     = {1362--1368},
  year      = {2018},
  month     = {7},
  doi       = {10.24963/ijcai.2018/189},
}

@article{Schidler_Szeider_2021,
  title   = {SAT-Based Decision Tree Learning for Large Data Sets},
  author  = {Schidler, Andre and Szeider, Stefan},
  journal = {Proceedings of the AAAI Conference on Artificial Intelligence},
  volume  = {35},
  number  = {5},
  pages   = {3904--3912},
  year    = {2021},
  doi     = {10.1609/aaai.v35i5.16509}
}

@article{neumann2000multivariate,
  title     = {Multivariate Wavelet Thresholding in Anisotropic Function Spaces},
  author    = {Neumann, Michael H.},
  journal   = {Statistica Sinica},
  pages     = {399--431},
  year      = {2000},
  publisher = {JSTOR}
}

@article{kerkyacharian2001nonlinear,
  title     = {Nonlinear Estimation in Anisotropic Multi-Index Denoising},
  author    = {Kerkyacharian, Gérard and Lepski, Oleg and Picard, Dominique},
  journal   = {Probability Theory and Related Fields},
  volume    = {121},
  number    = {2},
  pages     = {137--170},
  year      = {2001},
  publisher = {Springer}
}

@article{raskutti2012minimax,
  title     = {Minimax-Optimal Rates for Sparse Additive Models Over Kernel Classes via Convex Programming},
  author    = {Raskutti, Garvesh and Wainwright, Martin J. and Yu, Bin},
  journal   = {The Journal of Machine Learning Research},
  volume    = {13},
  number    = {1},
  pages     = {389--427},
  year      = {2012},
  publisher = {JMLR}
}

@article{mourtada2017universal,
  title   = {Universal Consistency and Minimax Rates for Online Mondrian Forests},
  author  = {Mourtada, Jaouad and Gaïffas, Stéphane and Scornet, Erwan},
  journal = {Advances in Neural Information Processing Systems},
  volume  = {30},
  year    = {2017}
}

@article{scornet2016random,
  title     = {Random Forests and Kernel Methods},
  author    = {Scornet, Erwan},
  journal   = {IEEE Transactions on Information Theory},
  volume    = {62},
  number    = {3},
  pages     = {1485--1500},
  year      = {2016},
  publisher = {IEEE}
}

@article{tan2024statistical,
  title   = {Statistical-Computational Trade-Offs for Greedy Recursive Partitioning Estimators},
  author  = {Tan, Yan Shuo and Klusowski, Jason M. and Balasubramanian, Krishnakumar},
  journal = {arXiv preprint arXiv:2411.04394},
  year    = {2024}
}

@article{anisot-besov-3,
  title   = {Entropy Numbers in Function Spaces with Mixed Integrability},
  author  = {Triebel, Hans},
  journal = {Revista Matemática Complutense},
  volume  = {24},
  pages   = {169--188},
  year    = {2011}
}

@article{binev2005universal-class-1,
  title   = {Universal Algorithms for Learning Theory Part I: Piecewise Constant Functions},
  author  = {Binev, Peter and Cohen, Albert and Dahmen, Wolfgang and DeVore, Ronald and Temlyakov, Vladimir and Bartlett, Peter},
  journal = {Journal of Machine Learning Research},
  volume  = {6},
  number  = {9},
  year    = {2005}
}

@article{binev2007universal,
  title     = {Universal Algorithms for Learning Theory Part II: Piecewise Polynomial Functions},
  author    = {Binev, Peter and Cohen, Albert and Dahmen, Wolfgang and DeVore, Ronald},
  journal   = {Constructive Approximation},
  volume    = {26},
  pages     = {127--152},
  year      = {2007},
  publisher = {Springer}
}

@article{blanchard2007optimal,
  title     = {Optimal Dyadic Decision Trees},
  author    = {Blanchard, Gilles and Schäfer, Christin and Rozenholc, Yves and Müller, K-R},
  journal   = {Machine Learning},
  volume    = {66},
  pages     = {209--241},
  year      = {2007},
  publisher = {Springer}
}

@article{AOS2014erm,
  title     = {Classification Algorithms Using Adaptive Partitioning},
  author    = {Binev, Peter and Cohen, Albert and Dahmen, Wolfgang and DeVore, Ronald},
  journal   = {The Annals of Statistics},
  volume    = {42},
  number    = {6},
  pages     = {2141--2163},
  year      = {2014},
  publisher = {Institute of Mathematical Statistics},
  doi       = {10.1214/14-AOS1234}
}

@inproceedings{tan2022cautionary,
  title        = {A Cautionary Tale on Fitting Decision Trees to Data from Additive Models: Generalization Lower Bounds},
  author       = {Tan, Yan Shuo and Agarwal, Abhineet and Yu, Bin},
  booktitle    = {International Conference on Artificial Intelligence and Statistics},
  pages        = {9663--9685},
  year         = {2022},
  organization = {PMLR}
}

@article{akakpo2012adaptation,
  title     = {Adaptation to Anisotropy and Inhomogeneity via Dyadic Piecewise Polynomial Selection},
  author    = {Akakpo, Nathalie},
  journal   = {Mathematical Methods of Statistics},
  volume    = {21},
  pages     = {1--28},
  year      = {2012},
  publisher = {Springer}
}

@inproceedings{syrgkanis2020estimation,
  title        = {Estimation and Inference with Trees and Forests in High Dimensions},
  author       = {Syrgkanis, Vasilis and Zampetakis, Manolis},
  booktitle    = {Conference on Learning Theory},
  pages        = {3453--3454},
  year         = {2020},
  organization = {PMLR}
}

@article{yang1999information,
  title     = {Information-Theoretic Determination of Minimax Rates of Convergence},
  author    = {Yang, Yuhong and Barron, Andrew},
  journal   = {The Annals of Statistics},
  pages     = {1564--1599},
  year      = {1999},
  publisher = {JSTOR}
}

@article{klusowski2020sparse,
  title   = {Sparse Learning with CART},
  author  = {Klusowski, Jason},
  journal = {Advances in Neural Information Processing Systems},
  volume  = {33},
  pages   = {11612--11622},
  year    = {2020}
}

@article{wager2015adaptive,
  title   = {Adaptive Concentration of Regression Trees, with Application to Random Forests},
  author  = {Wager, Stefan and Walther, Guenther},
  journal = {arXiv preprint arXiv:1503.06388},
  year    = {2015}
}

@article{klusowski2024large,
  title     = {Large Scale Prediction with Decision Trees},
  author    = {Klusowski, Jason M. and Tian, Peter M.},
  journal   = {Journal of the American Statistical Association},
  volume    = {119},
  number    = {545},
  pages     = {525--537},
  year      = {2024},
  publisher = {Taylor \& Francis}
}

@article{chatterjee2021adaptive,
  title     = {Adaptive Estimation of Multivariate Piecewise Polynomials and Bounded Variation Functions by Optimal Decision Trees},
  author    = {Chatterjee, Sabyasachi and Goswami, Subhajit},
  journal   = {The Annals of Statistics},
  volume    = {49},
  number    = {5},
  pages     = {2531--2551},
  year      = {2021},
  publisher = {Institute of Mathematical Statistics},
  doi       = {10.1214/20-AOS2045}
}

@article{scott2006minimax,
  title     = {Minimax-Optimal Classification with Dyadic Decision Trees},
  author    = {Scott, Clayton and Nowak, Robert D.},
  journal   = {IEEE Transactions on Information Theory},
  volume    = {52},
  number    = {4},
  pages     = {1335--1353},
  year      = {2006},
  publisher = {IEEE}
}

@article{donoho1997cart,
author = {David L. Donoho},
title = {{CART and best-ortho-basis: a connection}},
volume = {25},
journal = {The Annals of Statistics},
number = {5},
publisher = {Institute of Mathematical Statistics},
pages = {1870 -- 1911},
year = {1997},
doi = {10.1214/aos/1069362377},
}

@article{morgan1963problems,
  title     = {Problems in the Analysis of Survey Data, and a Proposal},
  author    = {Morgan, James N. and Sonquist, John A.},
  journal   = {Journal of the American Statistical Association},
  volume    = {58},
  number    = {302},
  pages     = {415--434},
  year      = {1963},
  publisher = {Taylor \& Francis}
}

@article{rudin2022interpretable,
  title     = {Interpretable Machine Learning: Fundamental Principles and 10 Grand Challenges},
  author    = {Rudin, Cynthia and Chen, Chaofan and Chen, Zhi and Huang, Haiyang and Semenova, Lesia and Zhong, Chudi},
  journal   = {Statistics Surveys},
  volume    = {16},
  pages     = {1--85},
  year      = {2022},
  publisher = {American Statistical Association},
  doi       = {10.1214/21-SS133}
}

@book{hardle2012wavelets,
  title     = {Wavelets, Approximation, and Statistical Applications},
  author    = {Härdle, Wolfgang and Kerkyacharian, Gerard and Picard, Dominique and Tsybakov, Alexander},
  volume    = {129},
  year      = {2012},
  publisher = {Springer Science \& Business Media}
}

@article{LIU2024106629,
  title     = {Optimal Classification Trees with Leaf-Branch and Binary Constraints},
  author    = {Liu, Enhao and Hu, Tengmu and Allen, Theodore and Hermes, Christoph},
  journal   = {Computers \& Operations Research},
  volume    = {166},
  pages     = {106629},
  year      = {2024},
  doi       = {10.1016/j.cor.2024.106629}
}

@article{ALES2024106515,
  title     = {New Optimization Models for Optimal Classification Trees},
  author    = {Ales, Zacharie and Hur\'{e}, Valentine and Lambert, Am\'{e}lie},
  journal   = {Computers \& Operations Research},
  volume    = {164},
  pages     = {106515},
  year      = {2024},
  publisher = {Elsevier},
  doi       = {10.1016/j.cor.2023.106515}
}

@book{boucheron2013concentration,
  title     = {Concentration Inequalities: A Nonasymptotic Theory of Independence},
  author    = {Boucheron, Stéphane and Lugosi, Gábor and Massart, Pascal},
  publisher = {Oxford University Press},
  year      = {2013},
  month     = {02},
  isbn      = {9780199535255},
  doi       = {10.1093/acprof:oso/9780199535255.001.0001}
}

@article{tsybakov2004optimal,
  title     = {Optimal Aggregation of Classifiers in Statistical Learning},
  author    = {Tsybakov, Alexander B.},
  journal   = {The Annals of Statistics},
  volume    = {32},
  number    = {1},
  pages     = {135--166},
  year      = {2004},
  publisher = {Institute of Mathematical Statistics},
  doi       = {10.1214/aos/1079120131}
}

@article{embedding2008embedding,
  title   = {Embeddings for Anisotropic Besov Spaces},
  author  = {Pérez Lázaro, F. J.},
  journal = {Acta Mathematica Hungarica},
  volume  = {119},
  number  = {1-2},
  pages   = {25--40},
  year    = {2008},
  isbn    = {0236-5294}
}

@article{leisner2003nonlinear,
  title     = {Nonlinear Wavelet Approximation in Anisotropic Besov Spaces},
  author    = {Leisner, Christopher},
  journal   = {Indiana University Mathematics Journal},
  pages     = {437--455},
  year      = {2003},
  publisher = {JSTOR}
}

@article{he2025foundational,
  title   = {Foundational Theory for Optimal Decision Tree Problems. I. Algorithmic and Geometric Foundations},
  author  = {He, Xi},
  journal = {arXiv preprint arXiv:2509.11226},
  year    = {2025}
}

@article{cattaneo2022pointwise,
  title   = {On the Pointwise Behavior of Recursive Partitioning and Its Implications for Heterogeneous Causal Effect Estimation},
  author  = {Cattaneo, Matias D. and Klusowski, Jason M. and Tian, Peter M.},
  journal = {arXiv preprint arXiv:2211.10805},
  year    = {2022}
}

@article{sun2016stabilized,
  title     = {Stabilized Nearest Neighbor Classifier and Its Statistical Properties},
  author    = {Sun, Will Wei and Qiao, Xingye and Cheng, Guang},
  journal   = {Journal of the American Statistical Association},
  volume    = {111},
  number    = {515},
  pages     = {1254--1265},
  year      = {2016},
  publisher = {Taylor \& Francis}
}

@techreport{audibert2004classification,
  title       = {Classification Under Polynomial Entropy and Margin Assumptions and Randomized Estimators},
  author      = {Audibert, Jean-Yves},
  institution = {Laboratoire de Probabilités et Modèles Aléatoires, Univ. Paris VI and VII},
  number      = {905},
  year        = {2004}
}

@article{bartlett2005local,
author = {Peter L. Bartlett and Olivier Bousquet and Shahar Mendelson},
title = {{Local Rademacher complexities}},
volume = {33},
journal = {The Annals of Statistics},
number = {4},
publisher = {Institute of Mathematical Statistics},
pages = {1497 -- 1537},
year = {2005},
doi = {10.1214/009053605000000282},
}

\appendix
\counterwithin{equation}{section}
\counterwithin{theorem}{section} 
\section{Proofs for Section \ref{sec:oracle_inequalities_proofs}}
\label{section:proof-Rademacher}

In this section, we provide omitted proofs of results in Section \ref{sec:oracle_inequalities_proofs}, on uniform concentration and derivation of our oracle inequalities.
The order of the results follows the recipe provided at the start of Section \ref{sec:oracle_inequalities_proofs}.
Consider a fixed function $f^*\colon [0,1]^d \to \R$.
Noticing that $\partfunc_{L} = \cup_{\partition \in \partitions{L}} \partfunc_{\partition}$,
let $\partfunc_\partition^*$ denote the linear span of $\partfunc_\partition$ and $f^*$ and set $\partfunc_{L}^* = \cup_{\partition \in \partitions{L}} \partfunc_{\partition}^*$.

\begin{lemma}
\label{lem:subG_multiplier_sup_cond_exp}
    For any fixed $\sample$, suppose $\mathcal{Z} \coloneqq \braces{Z_1,Z_2,\ldots,Z_n}$ are independent centered sub-Gaussian random variables with bounded sub-Gaussian norm i.e. $\max_{1 \leq i \leq n}\norm{Z_i}_{\psi_2} \leq K$ for some $K > 0$.
    For any $0 < r \leq 1$, $u \geq 1$, conditioned on $\sample$, with probability at least $1-e^{-u}$, we have the bound
    \begin{equation}
    \label{eq:subG_multiplier_sup_cond_deviation}
        \sup_{\substack{f \in \partfunc_\leaves^* \\
        \norm{f}_n \leq r}} \inprod{f,Z} \lesssim rK\paren*{\frac{L\log(nd) + u}{n}}^{1/2}.
    \end{equation}
    In particular,
    \begin{equation}
        \E_{\mathcal{Z}}\braces*{\sup_{\substack{f \in \partfunc_\leaves^* \\
        \norm{f}_n \leq r}} \inprod{f,Z}} \lesssim rK\paren*{\frac{L\log(nd)}{n}}^{1/2}
    \end{equation}
    where $C > 0$ is a universal constant.
\end{lemma}

\begin{proof}
    Fix some partition $\partition \in \partitions{L}$ and let $\mathcal{F}_\partition^*$ be the linear span of $\partfunc_\partition$ and $f^*$.
    It is clear that this function space equipped with the rescaled empirical norm $n^{1/2}\norm{-}_{2,n}$ is a Euclidean subspace of $\R^n$ of dimension at most $L+1$.
    To simplify, denote $\partfunc_{\partition,r}^*\coloneq \{f\in \partfunc_{\partition}^*: \norm{f}_n \leq r\}$ and $\partfunc_{L,r}^*\coloneq \{f\in \partfunc_{L}^*: \norm{f}_n \leq r\}$.
    The collection $\paren*{\sum_{i=1}^n Z_if(\bX_i)}_{f \in \mathcal{F}_{P,r}^*}$  can be viewed as a stochastic process with sub-Gaussian increments.
    Indeed, by Hoeffding's inequality, we have
    \begin{equation}
        \norm*{\sum_{i=1}^n Z_i f_1(\bX_i) - \sum_{i=1}^n Z_i f_2(\bX_i)}_{\psi_2} \leq Kn^{1/2}\norm{f_1-f_2}_{2,n}.
    \end{equation}
    For any $u \geq 1$, Talagrand's comparison inequality \citep{vershynin2018high} together with the standard upper bound on the Gaussian width of a Euclidean ball then implies that
    \begin{equation}
        \sup_{f \in \mathcal{F}_{P,r}^*} \sum_{i=1}^n Z_i f(\bX_i) \leq CKrn^{1/2}\paren*{(L+1)^{1/2} + u}
    \end{equation}
    with probability at least $1-2e^{-u^2}$, for some $C > 0$.

    Using this tail bound, we compute
\begin{equation}
\label{eq:tree_process_helper2}
\begin{split}
    \E_{\mathcal{Z}}\!\left\{
        \sup_{\partition \in \emppartitions{L}}
        \sup_{f \in \mathcal{F}_{\partition,r}^*} 
        \frac{1}{n}\sum_{i=1}^n Z_i f(\bX_i)
    \right\}
    & \leq \frac{C K r (L+1)^{1/2}}{n^{1/2}}
        + \int_0^\infty 
        \P_{\mathcal{Z}}\!\Biggl\{
            \sup_{\partition \in \emppartitions{L}}
            \sup_{f \in \mathcal{F}_{P,r}^*} 
            \frac{1}{n}\sum_{i=1}^n Z_i f(\bX_i) \\[-0.5ex]
    & \hspace{7em} \geq 
            \frac{C K r (L+1)^{1/2}}{n^{1/2}} + u
        \Biggr\} \, du \\[1ex]
    & \leq \frac{C K r (L+1)^{1/2}}{n^{1/2}}
        + \int_0^\infty 
            \min\!\Biggl\{
                2 (dn)^{\leaves} 
                \exp\!\left(-\frac{n u^2}{C^2 K^2 r^2}\right),
                \, 1 
            \Biggr\} du \\[1ex]
    & \leq C r K \left(\frac{L \log(nd)}{n}\right)^{1/2}.
\end{split}
\end{equation}
    Note that to obtain the second inequality, we used Lemma \ref{lem:card-partition-set} as well as a union bound over $\partition \in \emppartitions{L}$, while the last inequality follows after adjusting the constant $C$ appropriately.
    
    Finally, by Lemma \ref{lemma:partition-equiv-sample-population}, for every $\partition \in \partitions{L}$, there exists $\partition' \in \emppartitions{L}$ so that $\mathcal{F}_\partition^*$ and $\mathcal{F}_{\partition'}^*$ give the same Euclidean subspace under this norm.
    We therefore have
    \begin{equation}
        \sup_{f \in \partfunc_{L,r}^*} \sum_{i=1}^n Z_i f(\bX_i) = \sup_{\partition \in \emppartitions{L}}
        \sup_{f \in \mathcal{F}_{\partition,r}^*} 
        \frac{1}{n}\sum_{i=1}^n Z_i f(\bX_i).
    \end{equation}
    Combining this with \eqref{eq:tree_process_helper2} completes the proof of the lemma.
\end{proof}

\begin{lemma}
\label{lem:subG_multiplier_sup_exp}
    Let $(\bX_1,Z_1), (\bX_2,Z_2),\ldots, (\bX_n,Z_n) \in [0,1]^d\times\R$ be IID random variables such that for any value $\bx \in [0,1]^d$, the conditional distribution of $\xi_i$ given $\bX_i=x$ has mean zero and sub-Gaussian norm bounded by $K$ for some $K > 0$.
    For any $L \in [n]$, $0 < r \leq 1$, $g\colon [0,1]^d \to \R$,
    we have the bounds
    \begin{equation} \label{eq:sq_norm_sup_exp}
        \E\braces*{\sup_{\substack{f \in \partfunc_\leaves^* \\
        \norm{f}_2 \leq r,\norm{f}_\infty \leq 1}} \paren*{\norm{f}_n^2- \norm{f}_2^2}} \lesssim \paren*{\frac{L\log(nd)}{n}}^{1/2} + \frac{L\log(nd)}{n},
    \end{equation}
    \begin{equation} \label{eq:subG_multiplier_sup_exp}
        \E\braces*{\sup_{\substack{f \in \partfunc_\leaves^* \\
        \norm{f}_2 \leq r,\norm{f}_\infty \leq 1}} \inprod{f,Z}_n} \lesssim rK\paren*{\frac{L\log(nd)}{n}}^{1/2} + \frac{KL\log(nd)}{n},
    \end{equation}
    \begin{equation} \label{eq:function_multiplier_sup_exp}
        \E\braces*{\sup_{\substack{f \in \partfunc_\leaves^* \\
        \norm{f}_2 \leq r,\norm{f}_\infty \leq 1}} \paren*{\inprod{f,g}_n - \inprod{f,g}_\mu}} \lesssim r\norm{g}_\infty\paren*{\frac{L\log(nd)}{n}}^{1/2} + \frac{\norm{g}_\infty L\log(nd)}{n},
    \end{equation}
    where $C > 0$ is a universal constant.
\end{lemma}

\begin{proof}
\textit{Step 1: Upper bound for Rademacher complexity.}
    We first prove \eqref{eq:subG_multiplier_sup_exp} when assuming that $Z_i$ is a Rademacher random variable independent of $\bX_i$ for each $i$.
    For convenience, let us use $\mathcal{G}$ to denote the set over which the supremum is taken on the left hand side of \eqref{eq:subG_multiplier_sup_exp} and denote the whole quantity as $\Rad{\mathcal{G}}$.
    Notice that for each fixed $\sample$, 
    we have the inclusion
    \begin{equation}
        \mathcal{G} \subset \braces*{f \in \partfunc_\leaves^* \colon \norm{f}_n \leq \sup_{f \in \mathcal{G}} \norm{f}_n}.
    \end{equation}
    Using Lemma \ref{lem:subG_multiplier_sup_cond_exp}, the conditional expectation satisfies
    \begin{equation}
    \label{eq:subG_multiplier_sup_exp_helper}
        \E\braces*{\sup_{f \in \mathcal{G}} \inprod{f,Z}_n~\vline~ \sample} \leq C\paren*{\frac{L\log(nd)}{n}}^{1/2} \sup_{f \in \mathcal{G}} \norm{f}_n
    \end{equation}
    for some universal constant $C > 0$.
    Next, it is easy to compute
    \begin{equation}
        \E\braces*{\sup_{f \in \mathcal{G}} \norm{f}_n} \leq \E\braces*{\sup_{f \in \mathcal{G}} \paren*{\norm{f}_n^2 - \norm{f}_2^2}}^{1/2} + \sup_{f \in \mathcal{G}} \norm{f}_2.
    \end{equation}
    The second term on the right hand side is equal to $r$ by the definition of $\mathcal{G}$, while the first term can be bounded as
    \begin{equation}
    \label{eq:subG_multiplier_sup_exp_helper2}
        \E\braces*{\sup_{f \in \mathcal{G}} \paren*{\norm{f}_n^2 - \norm{f}_2^2}} \leq 2\Rad{\braces*{f^2 \colon f \in \mathcal{G}}} \leq 2\Rad{\mathcal{G}}.
    \end{equation}
    Here, the first inequality is by symmeterization, while the second inequality uses the Ledoux-Talagrand contraction inequality and the assumption that all functions in $\mathcal{G}$ have supremum norm bounded by $1$.
    Taking a further expectation over $\sample$ in \eqref{eq:subG_multiplier_sup_exp_helper} and plugging these bounds back into the resulting inequality, we get
    \begin{equation}
        \Rad{\mathcal{G}} \leq C\paren*{\frac{L\log(nd)}{n}}^{1/2}\paren*{(2\Rad{\mathcal{G}})^{1/2} + r}.
    \end{equation}
    This is a quadratic inequality in $\Rad{\mathcal{G}}^{1/2}$, which can be solved (and squared) to get
    \begin{equation}
    \label{eq:subG_multiplier_sup_exp_helper3}
        \Rad{\mathcal{G}} \leq 2Cr\paren*{\frac{L\log(nd)}{n}}^{1/2} + 4C^2\paren*{\frac{L\log(nd)}{n}}.
    \end{equation}
    Note that because of \eqref{eq:subG_multiplier_sup_exp_helper2}, we have also finished proving \eqref{eq:sq_norm_sup_exp}.
    
    \textit{Step 2: General upper bound.}
    Following the same steps as in Step 1, we obtain
    \begin{equation}
        \E\braces*{\sup_{f\in\mathcal{G}}\inprod{f,Z}_n} \leq CK\paren*{\frac{L\log(nd)}{n}}^{1/2}\paren*{(2\Rad{\mathcal{G}})^{1/2} + r}.
    \end{equation}
    Plugging in \eqref{eq:subG_multiplier_sup_exp_helper3} and by some simple algebra, we obtain \eqref{eq:subG_multiplier_sup_exp}.

    \textit{Step 3: Bounding \eqref{eq:function_multiplier_sup_exp}.}
    Using symmeterization and contraction, we have
    \begin{equation}
        \E\braces*{\sup_{f \in \mathcal{G}} \paren*{\inprod{f,g}_n - \inprod{f,g}_\mu}} \leq 2\Rad{\braces*{fg \colon f \in \mathcal{G}}} \leq 2\norm{g}_\infty\Rad{\mathcal{G}}.
    \end{equation}
    The bound on Rademacher complexity from Step 1 finishes the proof.
\end{proof}

\begin{proof}[Proof of Lemma \ref{lem:emp_process_suprema}]
    To prove \eqref{eq:sq_norm_sup_deviation}, we will use the logarithmic Sobolev inequalities technique for bounding suprema of empirical processes \citep{boucheron2013concentration}.
    Notice that for any $f\in \mathcal{G}$,
    \begin{equation}
        \E\braces*{\paren*{f(\bX)^2 - \norm{f}_2^2}^2} \leq \E\braces*{f(\bX)^4} \leq \norm{f}_2^2 \leq r^2.
    \end{equation}
    Applying Bousquet's inequality (Theorem 12.5 in \citet{boucheron2013concentration}) gives us a probability at least $1-e^{-u}/5$ event on which
    \begin{equation}
        \sup_{f \in \mathcal{G}}\paren*{\norm{f}_n^2 - \norm{f}_2^2}  \leq \E\braces*{\sup_{f \in \mathcal{G}}\paren*{\norm{f}_n^2 - \norm{f}_2^2}} + C\paren*{r^2 + \E\braces*{\sup_{f \in \mathcal{G}}\paren*{\norm{f}_n^2 - \norm{f}_2^2}}}^{1/2}\paren*{\frac{u}{n}}^{1/2} + \frac{Cu}{n}.
    \end{equation}
    Applying \eqref{eq:sq_norm_sup_exp} to the right hand side and simplifying gives the bound
    \begin{equation}
        \label{eq:sq_norm_sup_deviation_upper}
        \sup_{\substack{f \in \mathcal{G}}} \paren*{\norm{f}_n^2- \norm{f}_2^2} \leq Cr\paren*{\frac{L\log(nd) + u}{n}}^{1/2} + \frac{CL\log(nd) + u}{n}.
    \end{equation}
    Using a similar argument but with the process $\norm{f}_2^2 - \norm{f}_n^2$ gives a probability at least $1-e^{-u}/5$ event on which
    \begin{equation}
        \sup_{\substack{f \in \mathcal{G}}} \paren*{\norm{f}_2^2- \norm{f}_n^2} \leq Cr\paren*{\frac{L\log(nd) + u}{n}}^{1/2} + \frac{CL\log(nd) + u}{n}.
    \end{equation}
    On the intersection of the two events, \eqref{eq:sq_norm_sup_deviation} holds.
    The same argument, combined with \eqref{eq:function_multiplier_sup_exp}, can be used to show \eqref{eq:function_multiplier_sup_deviation}. 

    It remains to prove \eqref{eq:subG_multiplier_sup_deviation}.
    First notice that on the event for which \eqref{eq:sq_norm_sup_deviation_upper} holds, we have
    \begin{equation}
        \begin{split}
            \sup_{f\in\mathcal{G}}\norm{f}_n^2 & \leq \sup_{f\in\mathcal{G}}\norm{f}_2^2 + \sup_{\substack{f \in \mathcal{G}}} \paren*{\norm{f}_n^2- \norm{f}_2^2} \\
            & \leq r^2 + Cr\paren*{\frac{L\log(nd) + u}{n}}^{1/2} + \frac{CL\log(nd) + u}{n} \\
            & \leq C\paren*{r + \paren*{\frac{L\log(nd) + u}{n}}^{1/2}}^2.
        \end{split}
    \end{equation}
    Since $\mathcal{G}$ is symmetric, we have $\sup_{f\in\mathcal{G}} \abs{\inprod{f,Z}_n} = \sup_{f\in\mathcal{G}} \inprod{f,Z}_n$.
    Next further condition on the probability at least $1-e^{-u}/5$ event on which \eqref{eq:subG_multiplier_sup_cond_deviation} holds.
    We then have
    \begin{equation}
    \begin{split}
        \sup_{f\in\mathcal{G}} \inprod{f,\xi}_n & \leq CK\sup_{f\in\mathcal{G}}\norm{f}_n\paren*{\frac{L\log(nd) + u}{n}}^{1/2} \\
        & \leq CKr\paren*{\frac{L\log(nd) + u}{n}}^{1/2} + \frac{CKL\log(nd) + u}{n}.
    \end{split}
    \end{equation}
    As such, on the intersection of all these events, \eqref{eq:sq_norm_sup_deviation}, \eqref{eq:subG_multiplier_sup_deviation}, and \eqref{eq:function_multiplier_sup_deviation} hold, with error probability at most $e^{-u}$.
   
\end{proof}

\section{Proofs for Section \ref{section:approximation}}
\label{appendix:proof-apprrox-PSHAB}

In this section we provide omitted proofs of the PSHAB space approximation bounds stated in Section \ref{section:approximation}.
The order of the results follows the outline described at the end of Section \ref{section:approximation}.
To recap, we first provide an approximation bound for anisotropic Besov spaces with domain $[0,1]^d$ (Lemma \ref{cor:approx-aniso-besov-piecewise}), which extends a result of \citet{akakpo2012adaptation} to the boundary smoothness case $(\alpha_j = 1)$ via Besov space embedding theory.
The next step is to extend the approximation bound to anisotropic Besov spaces with other domains (Lemma \ref{lemma:approx-aniso-besov-piecewise-1}).
Finally, we combine the bounds over each piece in the partition $P_*$ and optimize the leaf allocation to obtain the bounds in Theorem \ref{thm:approx-PSHAB-reg} and Theorem \ref{thm:approx-PSHAB-cls}.


\begin{lemma}[Approximation bound for anisotropic Besov space]
\label{cor:approx-aniso-besov-piecewise}
    Let $\boldsymbol{\alpha} \in (0,1]^d$, $\bar\alpha=H([d],\balpha)$, $0 < p \leq \infty$, and $1 \leq m \leq \infty$ such that
    \begin{equation*}
        \bar\alpha / d > (1/p - 1/m)_+.\footnote{Here, $(x)_+ \coloneq \max\{x,0\}$.}
    \end{equation*}
    Assume $f \in B_{p,q}^{\boldsymbol{\alpha}}([0,1]^d, \besov)$, where $(p,q)$ satisfy the one of the two following conditions:
    \begin{enumerate}
        \item[(i)] $0\leq q\leq \infty$, $0<p\leq 1$ or  $m\leq p\leq\infty$;
        \item[(ii)] $0\leq q\leq p$, $1<p<m$.
    \end{enumerate}
    Then for any $L \in \mathbb{N}$,
    \begin{equation}\label{eq:conclud-approx-anis-besov}
        \inf_{\tilde{f} \in \partfunc_\leaves}
        \norm{\tilde{f} - f}_{L^m([0,1]^d)} \lesssim_{d,\alpha_{\min},\bar\alpha,p,m} \besov \cdot \leaves^{-\bar\alpha / d}.
    \end{equation}
\end{lemma}


\begin{lemma}[Piecewise approximation bound for PSHAB space]
\label{lemma:approx-aniso-besov-piecewise-1}
    Let $\boldsymbol{\alpha} \in (0,1]^d$, $A \subseteq [0,1]^d$ be an axis-aligned rectangle, $S\subset [d]$ with $|S|=s$ and $\bar\alpha=H(S,\balpha)$, 
    and let $f \in \paren*{B_{p,q}^{\boldsymbol{\alpha}}(A, \besov)}_S$.
    Suppose that $p$ and $q$ are as in Corollary~\ref{cor:approx-aniso-besov-piecewise}.
    Then for any $L \in \mathbb{N}$, 
    \begin{equation*}
        \inf_{\tilde{f} \in \partfunc_\leaves(A)}
        \norm{\tilde{f} - f}_{L^m(A)} 
        \lesssim_{s,\alpha_{\min},\bar\alpha,p,m} |A|^{1/m-1/p} \besov   \leaves^{-\bar\alpha / s},
    \end{equation*}
    where $\partfunc_\leaves(A)=\big\{\sum_{j=1}^\leaves a_j\1_{A_j}:\{A_j\big\}_{j=1}^\leaves \text{ is a tree based partition of $A$}, a_j\in\mathbb{R},j\in[\leaves]\}$.
\end{lemma}

\begin{proof}[Proof of Theorem \ref{thm:approx-PSHAB-reg}]
    Let $\boldsymbol{v}_1 = (v_1, \ldots, v_B)$ be defined as in Definition~\ref{def:quantities}.
    Suppose that we allocate $\leaves_b$ samples to each $G_b$. Applying Lemma~\ref{lemma:approx-aniso-besov-piecewise-1} with $m=2$, we obtain that for every $b\in[B]$ there exists a piecewise constant function $f_b$, associated with a tree-based partition of $G_b$, such that
    \begin{equation*}
    \norm{f_b - f^*|_{G_b}}_{L^2(G_b)}
    \le
    C_1\,\besov_b\,|G_b|^{1/2-1/p}\,\leaves_b^{-H(S_b,\balpha_b)/|S_b|}
    \leq
    C_1\,v_b^{1/2}\,\leaves_b^{-\bar\alpha/s},
    \end{equation*}
    where $C_1$ depends only on $s$, $\alpha_{\min}$, $\bar\alpha$, and $p$.  
    Define $f$ by combining the local approximations, that is, $f|_{G_b}=f_b$ for each $b\in[B]$. Since $\partition_*$ is tree-based, the induced partition underlying $f$ is also tree-based. Hence $f\in\partfunc_\leaves$, and
    \begin{equation}\label{eq:approx-reg-end-1}
    \norm{f-f^*}_{L^2(\Omega)}^2
    \le
    \sum_{b=1}^B
    \norm{f_b - f^*|_{G_b}}_{L^2(G_b)}^2
    \le
    C_1
    \sum_{b=1}^B
    v_b\,\leaves_b^{-2\bar\alpha/s}.
    \end{equation}

    To minimize the right-hand side of \eqref{eq:approx-reg-end-1} with respect to the leaf allocation $(L_b)_{b=1}^B$, we choose the weights $w_b$ proportional to the optimal scaling. Specifically, by Lemma \ref{lemma:opimization}, let
    \begin{equation*}
        w_b = \frac{v_b^{s/(s+2\bar\alpha)}}{\sum_{j=1}^B v_j^{s/(s+2\bar\alpha)}}, \quad b=1,\ldots,B.
    \end{equation*}
    By Lemma~\ref{lemma:allocation}, there exists an integer allocation $L_1, \dots, L_B$ satisfying $\sum_{b=1}^B L_b = L$ and $L_b \ge (L-B)w_b$. Under the assumption $L \ge 2B$, we have $L - B \ge L/2$, which implies
    \begin{equation*}
        L_b \geq \frac{L w_b}{2}.
    \end{equation*}
    Substituting this lower bound into \eqref{eq:approx-reg-end-1} yields
    \begin{align} \label{eq:approx-reg-end-2}
        \norm{f-f^*}_{L^2(\Omega)}^2
        &\leq C_2 \sum_{b=1}^B v_b^{s/(s+2\bar\alpha)} (L w_b)^{-2\bar\alpha/s} \nonumber \\
        &\leq C_3 L^{-2\bar\alpha/s} \paren*{\sum_{b=1}^B v_b^{s/(s+2\bar\alpha)}}^{1 + 2\bar\alpha/s} \nonumber \\
        &= C_3 \norm{\boldsymbol{v}_1}_{\frac{s}{s+2\bar\alpha}} L^{-2\bar\alpha/s},
    \end{align}
    where the constants $C_2, C_3$ depend only on $s$, $\alpha_{\min}$, $\bar\alpha$, and $p$. 

    Furthermore, by Assumption \ref{assum:bound-density}(i):
    \begin{equation} \label{eq:approx-reg-end-3}
        \Excess_{\operatorname{reg},L}=\inf_{f \in \partfunc_{L}}\norm{f-f^*}_{2}^2
        \leq c_{\max}\inf_{f \in \partfunc_{L}}\norm{f-f^*}_{L^2(\Omega)}^2.
    \end{equation}
    The bound \eqref{eq:approx-reg-conclude-1} then follows from combining \eqref{eq:approx-reg-end-2} and
    \eqref{eq:approx-reg-end-3}.
\end{proof}

\begin{proof}[Proof of Theorem \ref{thm:approx-PSHAB-cls}]
    The proof proceeds as follows.
When $\rho=0$, that is, when Tsybakov's noise assumption is trivial, we apply Lemma \ref{lemma:approx-aniso-besov-piecewise-1} to $\paren*{B_{p, q}^{\boldsymbol{\alpha}_b}(G_b, \Lambda_b)}_{S_b}$ with $m=1$ for each $b\in[B]$ to obtain the optimal approximation error on each piece. When $\rho>0$, we instead use Lemma \ref{lemma:approx-aniso-besov-piecewise-1} with $m=\infty$. We then aggregate the resulting piecewise approximation errors and conclude the stated bound via standard binary classification arguments.

    \textit{Case 1: $\rho=0$}.
    For any $\tilde{\eta}\in\partfunc_\leaves$, define $f_{\tilde{\eta}}\coloneq\1\{\tilde{\eta}\geq 1/2\}$, so that $f_{\tilde{\eta}}\in\partfunc_\leaves$. By Theorem~2.2 of \citet{pattern}, letting $f^*$ denote the Bayes classifier,
    \begin{equation}\label{eq:bayes-cls}
    \excess_{\operatorname{cls}}(f_{\tilde{\eta}})
    =
    2\E\braces*{\abs*{\eta(\x)-\frac{1}{2}}\1\braces*{f_{\tilde{\eta}}(\x)\neq f^*(\x)}}.
    \end{equation}
    Moreover, the event $\{f_{\tilde{\eta}}(\x)\neq f^*(\x)\}$ implies $\abs*{\eta(\x)-\frac{1}{2}}\leq \abs*{\tilde{\eta}(\x)-\eta(\x)}$. Combining this with \eqref{eq:bayes-cls} yields
    \begin{equation}\label{eq:approx-cls-small-p-1}
    \begin{split}
    \excess_{\operatorname{cls}}(f_{\tilde{\eta}})
    &\leq
    2\E\braces*{\abs*{\eta(\x)-\frac{1}{2}}\1\braces*{\abs*{\eta(\x)-\frac{1}{2}}\leq \abs*{\tilde{\eta}(\x)-\eta(\x)}}}\\
    &\leq
    2\E\braces*{\abs*{\tilde{\eta}(\x)-\eta(\x)}}\\
    &\leq
    2c_{\max}\norm{\tilde{\eta}-\eta}_{L^1(\Omega)},
    \end{split}
    \end{equation}
    where the last inequality follows from Assumption~\ref{assum:bound-density}. It therefore suffices to control $\norm{\tilde{\eta}-\eta}_{L^1(\Omega)}$ for a suitable choice of $\tilde{\eta}$.

    Let $\boldsymbol{v}_2 = (v_1, \ldots, v_B)$ be defined as in Definition~\ref{def:quantities}, where
    \begin{equation*}
    v_b\coloneq |G_b|^{1-1/p}\,\Lambda_b,
    \qquad
    b=1,\ldots,B.
    \end{equation*}
    We then apply Lemma~\ref{lemma:approx-aniso-besov-piecewise-1} with $m=1$. For every $b\in[B]$, there exists a decision tree function $\eta_b$, associated with a tree-based partition of $G_b$, such that
    \begin{equation*}
    \norm{\eta_b-\eta|_{G_b}}_{L^1(G_b)}
    \le
    C_1\,|G_b|^{1-1/p}\,\besov_b\,\leaves_b^{-H(S_b,\balpha_b)/|S_b|}
    \leq 
    C_1\,v_b\,\leaves_b^{-\bar\alpha/s},
    \end{equation*}
    where $C_1$ depends only on $s$, $\alpha_{\min}$, $\bar\alpha$, and $p$. Define $\tilde{\eta}$ by $\tilde{\eta}|_{G_b}=\eta_b$ for each $b\in[B]$. Since $\partition_*$ is tree-based, the induced partition underlying $\tilde{\eta}$ is also tree-based, and hence $\tilde{\eta}\in\partfunc_\leaves$. Moreover,
    \begin{equation}\label{eq:eq-approx-cls-1}
    \norm{\tilde{\eta}-\eta}_{L^1(\Omega)}
    =
    \sum_{b=1}^B
    \norm{\eta_b-\eta|_{G_b}}_{L^1(G_b)}
    \le
    C_1\sum_{b=1}^B v_b\,\leaves_b^{-\bar\alpha/s}.
    \end{equation}
    
    Analogous to the proof of Theorem~\ref{thm:approx-PSHAB-reg}, we employ Lemma~\ref{lemma:opimization},  and Lemma~\ref{lemma:allocation} to determine the optimal allocation. We define the weights
    \begin{equation*}
        w_b = \frac{v_b^{s/(s+\bar\alpha)}}{\sum_{j=1}^B v_j^{s/(s+\bar\alpha)}}, \quad b=1,\ldots,B.
    \end{equation*}
    By Lemma~\ref{lemma:allocation}, there exists an integer allocation such that $L_b \ge (L-B)w_b$. Under the assumption $L \ge 2B$, this implies the lower bound $L_b \ge L w_b / 2$. Substituting these estimates into \eqref{eq:eq-approx-cls-1} yields
    \begin{align} \label{eq:eq-approx-cls-2}
        \norm{\tilde{\eta}-\eta}_{L^1(\Omega)}
        &\leq C_2 \sum_{b=1}^B v_b^{s/(s+\bar\alpha)} \paren*{ L w_b }^{-\bar\alpha/s} \nonumber \\
        &\leq C_3 L^{-\bar\alpha/s} \paren*{\sum_{b=1}^B v_b^{s/(s+\bar\alpha)}}^{1 + \bar\alpha/s} \nonumber \\
        &= C_3 \norm{\boldsymbol{v}_2}_{\frac{s}{s+\bar\alpha}} L^{-\bar\alpha/s},
    \end{align}
    where $C_3$ depends only on $s$, $\alpha_{\min}$, $\bar\alpha$, and $p$. Combining this bound with \eqref{eq:eq-approx-cls-2}, and noting that $\Excess_{\operatorname{cls},L}\le \excess_{\operatorname{cls}}(f_{\tilde{\eta}})$, we obtain \eqref{eq:approx-cls-conclude-1}.

    \textit{Case 2: $\rho>0$.}
    Let $\boldsymbol{v}_3 = (v_1', \ldots, v_B')$ be defined as in Definition~\ref{def:quantities}, where
    \begin{equation*}
    v_b'\coloneq |G_b|^{-1/p}\,\Lambda_b,
    \qquad
    b=1,\ldots,B.
    \end{equation*}
    Applying Lemma~\ref{lemma:approx-aniso-besov-piecewise-1} with $m=\infty$, we obtain that, for each $b\in[B]$, there exists a decision function $\zeta_b$, associated with a tree-based partition of $G_b$, such that
    \begin{equation*}
    \norm{\zeta_b - \eta|_{G_b}}_{\infty}
    \le
    C_3\,|G_b|^{-1/p}\,\besov_b\,\leaves_b^{-H(S_b,\balpha_b)/|S_b|}
    \leq 
    C_3\,v_b'\,\leaves_b^{-\bar\alpha/s},
    \end{equation*}
    where $C_3$ depends only on $s$, $\alpha_{\min}$, $\bar\alpha$, and $p$. Define $\tilde{\zeta}=\sum_{b=1}^B \1_{G_b}\zeta_b$. Since $\partition_*$ is tree-based, the induced partition underlying $\tilde{\zeta}$ is also tree-based, and hence $\tilde{\zeta}\in\partfunc_\leaves$. Moreover,
    \begin{equation}\label{eq:approx-cls-large-p-1}
    \begin{split}
    \norm{\tilde{\zeta}-\eta}_{\infty(\Omega)}
    &\le
    \max_{1\le b\le B}\norm{\zeta_b - \eta|_{G_b}}_{\infty(G_b)}\\
    &\le
    C_3\max_{1\le b\le B} v_b'\,\leaves_b^{-\bar\alpha/s}
    \eqcolon
    \epsilon.
    \end{split}
    \end{equation}
    Let $f_{\tilde{\zeta}}=\1\braces*{\tilde{\zeta}\geq 1/2}$, so that $f_{\tilde{\zeta}}\in\partfunc_\leaves$. Let $f^*$ denote the Bayes classifier and define $M(\x)=\abs*{\eta(\x)-\tfrac{1}{2}}$. By the same argument leading to \eqref{eq:approx-cls-small-p-1},
    \begin{equation*}
    \excess_{\operatorname{cls}}(f_{\tilde{\zeta}})
    \le
    2\E\braces*{M(\x)\1\braces*{M(\x)\le \abs*{\eta(\x)-\tilde{\zeta}(\x)}}}.
    \end{equation*}
    Combining this bound with \eqref{eq:approx-cls-large-p-1} yields
    \begin{equation}\label{eq:Tsy:rho+1}
    \begin{split}
    \excess_{\operatorname{cls}}(f_{\tilde{\zeta}})
    &\le
    2\E\braces*{M(\x)\1\braces*{M(\x)\le \epsilon}}\\
    &\le
    2\epsilon\,\P\braces*{M(\x)\le \epsilon}\\
    &\le
    C_{\rho,M}\,\epsilon^{\rho+1},
    \end{split}
    \end{equation}
    where the last inequality follows from Assumption~\ref{assum:tsybakov}.
    
    To minimize the right-hand side of \eqref{eq:Tsy:rho+1}, it suffices to minimize the term $\max_{1\le b\le B} v_b'\,\leaves_b^{-\bar\alpha/s}$ over all allocations $(\leaves_1,\ldots,\leaves_B)$. Analogous to the proof for the case $\rho=0$, by invoking Lemma~\ref{lemma:optimization-2} and Lemma~\ref{lemma:allocation} under the condition that $L\geq 2B$, there exists an allocation satisfying $\leaves_b \ge L w_b / 2$, where
    \begin{equation*}
        w_b
        =
        \frac{(v_b')^{s/\bar\alpha}}
        {\sum_{j=1}^B (v_j')^{s/\bar\alpha}},
        \qquad
        b=1,\ldots,B.
    \end{equation*}
    Substituting this allocation into \eqref{eq:approx-cls-large-p-1} yields
    \begin{equation}\label{eq:epsilon}
        \epsilon
        \le
        C_4
        \paren*{\sum_{b=1}^B (v_b')^{s/\bar\alpha}}^{\bar\alpha/s}
        \leaves^{-\bar\alpha/s},
    \end{equation}
    where $C_4$ depends only on $s$, $\alpha_{\min}$, $\bar\alpha$, and $p$.

    Combining \eqref{eq:Tsy:rho+1} and \eqref{eq:epsilon}, and noting that $\Excess_{\operatorname{cls},\leaves}\le \excess_{\operatorname{cls}}(f_{\tilde{\zeta}})$, we obtain \eqref{eq:approx-cls-conclude-2}.
\end{proof}

\section{Proofs for Section \ref{section:ideal_spatial_adaption}}
\label{section:optimal_rates_proof}

In this section, we provide omitted proofs for the generalization bounds illustrating ideal spatial adaptation stated in Section \ref{section:ideal_spatial_adaption}. The proofs proceed by balancing the approximation error $\Excess_L$ against the estimation error penalties identified in our oracle inequalities.

\begin{proof}[Proof of Theorem \ref{thm:regression-exact-rate-PSHAB}]
    By the oracle inequality for regression (Theorem \ref{thm:regression}, equation \eqref{eq:thm-statement:regression-2} evaluated at $\delta=1/2$), the following holds for any $\leaves\in[n]$:
    \begin{equation}\label{eq:generalization-1}
         \excess_{\mathrm{reg}}(\hat f_{\lambda})
    \leq
        3\Excess_{\mathrm{reg},L} + 4\lambda L
        \asymp \Excess_{\mathrm{reg},L}+(M+K)^2\frac{ \log(nd)+u}{n}\leaves.
    \end{equation}
    
    Applying the approximation bound \eqref{eq:approx-reg-conclude-1} from Theorem \ref{thm:approx-PSHAB-reg} and plugging in the chosen value of $\lambda$, we obtain that for every $\leaves$ satisfying $2B\leq \leaves \leq n$,
    \begin{equation}
         \excess_{\mathrm{reg}}(\hat f_{\lambda})
    \leq C_{s,\alpha_{\min},\bar\alpha,c_{\max}}\,
    \norm{\boldsymbol{v}_1}_{\frac{s}{s+2\bar\alpha}}\,\leaves^{-2\bar\alpha/s} + C_1(M+K)^2\frac{ \log(nd)+u}{n}\leaves.
    \end{equation}
    To optimize this bias-variance trade-off, we balance the two terms by setting $\leaves=\lfloor C\leaves_1\rfloor$ for some universal constant $C\geq1$, where
    \begin{equation*}
        \leaves_1 = \norm{\boldsymbol{v}_1}_{\frac{s}{s+2\bar\alpha}}^{\frac{s}{s+2\bar\alpha}} \, \paren*{(M+K)^2\frac{ \log(nd)+u}{n}}^{-\frac{s}{s+2\bar\alpha}},
    \end{equation*}
    which directly yields the desired bound \eqref{eq:reg-conclude-1}. Finally, it is straightforward to verify that the condition $n\geq \leaves \geq 2B$ holds whenever $n \geq N_1$ as defined in Remark \ref{rem:min_sample_size_reg}.
\end{proof}

\begin{proof}[Proof of Theorem \ref{thm:classification-specific-rate-PSHAB}]
    Similarly, by the oracle inequality for classification (Theorem \ref{thm:classification}, equation \eqref{eq:thm-statement:classification-2} evaluated at $\delta=1/2$), the following holds for any $\leaves\in[n]$:
    \begin{equation}\label{eq:generalization-cls-1}
         \excess_{\operatorname{cls}}(\hat f_{\lambda,\theta})
         \;\;\leq\;\;
            3\Excess_{\operatorname{cls},L} + 4\lambda L^{\theta}
            \asymp \Excess_{\operatorname{cls},L}
            +\;\paren*{\leaves\frac{\log(nd)+u}{n}}^{\frac{\rho+1}{\rho+2}}.
    \end{equation}
    
    \textit{Case (i): $\rho=0$.} We apply the approximation bound \eqref{eq:approx-cls-conclude-1}. For every $\leaves$ satisfying $2B\leq \leaves \leq n$, we have:
    \begin{equation}\label{eq:generalization-cls-2}
        \excess_{\operatorname{cls}}(\hat f_{\lambda,\theta})
        \leq
        C_{s,\alpha_{\min},\bar\alpha,p,c_{\max}}
        \norm{\boldsymbol{v}_2}_{\frac{s}{s+\bar\alpha}}
        \leaves^{-\bar\alpha/s}+C_{M}\paren*{\leaves\frac{\log(nd)+u}{n}}^{\frac{1}{2}}.
    \end{equation}
    Setting $\leaves=\lfloor C\leaves_1\rfloor$ to balance the terms for some constant $C\geq 1$, where
    \begin{equation*}
        \leaves_1=
        \norm{\boldsymbol{v}_2}_{\frac{s}{s+\bar\alpha}}^{\frac{2s}{s+2\bar\alpha}}
        \paren*{\frac{\log(nd)+u}{n}}^{-\frac{s}{s+2\bar\alpha}},
    \end{equation*}
    yields \eqref{eq:cls-conclude-1}. This choice of $L$ is valid under the minimum sample size constraint $n \geq N_2$.

    \textit{Case (ii): $\rho > 0$.} We use the approximation bound \eqref{eq:approx-cls-conclude-2}. Plugging this into \eqref{eq:generalization-cls-1} yields:
     \begin{equation}\label{eq:generalization-cls-5}
        \excess_{\operatorname{cls}}(\hat f_{\lambda,\theta})
        \leq
        C_{s,\alpha_{\min},\bar\alpha,\rho,M,c_{\max}}\,
        \norm{\boldsymbol{v}_3}_{\frac{s}{\bar\alpha}}^{\rho+1}
        \,\leaves^{-(\rho+1)\bar\alpha/s}+C_{\rho,M}\paren*{\leaves\frac{\log(nd)+u}{n}}^{\frac{\rho+1}{\rho+2}}.
    \end{equation}
    Balancing these terms by setting $\leaves=\lfloor C\leaves_2\rfloor$ for some universal constant $C\geq 1$, where
    \begin{equation*}
        \leaves_2=  \norm{\boldsymbol{v}_3}_{\frac{s}{\bar\alpha}}^{\frac{(2+\rho)s}{s+(2+\rho)\bar\alpha}}\,\paren*{\frac{\log(nd)+u}{n}}^{-\frac{s}{s+(2+\rho)\bar\alpha}},
    \end{equation*}
    yields the final bound \eqref{eq:cls-conclude-2}, valid for sample sizes $n \geq N_3$.
\end{proof}

\section{Proofs of minimax lower bounds}
\subsection{Proof of Theorem \ref{thm:minimax}}
\label{section:proof-minimax-reg}

In this section, we derive the minimax lower bound for the regression risk (Definition \ref{def:minimax_risk}). Our analysis follows the information-theoretic framework of \citet{yang1999information}, which was further streamlined by \citet{raskutti2012minimax} and \citet{suzuki2021deep}.
The main tool developed by \citet{suzuki2021deep} is stated below.

\begin{lemma}[Lemma 4 of \citealp{suzuki2021deep}]
\label{lemma:minimax-suzuki}
    Let $\mathcal{F}$ be a function space and consider the minimax risk $\mathcal{M}_{\mathrm{reg}, n}(\mathcal{F})$ defined in Definition \ref{def:minimax_risk}. Let $\mathcal{Q}(\varepsilon) = \mathcal{Q}(\varepsilon; \mathcal{F}, \|\cdot\|_2)$ and $\mathcal{N}(\varepsilon) = \mathcal{N}(\varepsilon; \mathcal{F}, \|\cdot\|_2)$ denote the packing and covering numbers, respectively. 
    
    Suppose that for some $\zeta_n, \varrho_n > 0$ with $\log \mathcal{Q}(\zeta_n) \geq 4 \log 2$, the following entropy condition holds:
    \begin{equation*}
        \frac{n \varrho_n^2}{2K^2} 
        \;\leq\; \log \mathcal{N}(\varrho_n) 
        \;\leq\; \frac{1}{8} \log \mathcal{Q}(\zeta_n).
    \end{equation*}
    Then, the minimax risk is lower bounded by
    \begin{equation*}
        \mathcal{M}_{\mathrm{reg}, n}(\mathcal{F}) \geq \frac{\zeta_n^2}{4}.
    \end{equation*}
\end{lemma}

The proof hence reduces to establishing the lower and upper bounds on the metric entropy of the PSHAB space. To this end, we invoke the results for standard anisotropic Besov spaces provided in Proposition 10 of \citet{anisot-besov-3}.
To adapt these results to our setting, we introduce some new notation:
For any index set $S \subset [d]$ and $A\subseteq [0,1]^d$, we let $A_S = \{\bx_S \in [0,1]^{|S|} : \bx \in A\}$, that is the  projection of $A$ onto the coordinates in $S$.
Recall also that if $f$ is an $s$-sparse function with relevant index set $S$, we define $f_S$ by $f_S(\bx_S) = f(\bx)$. To simplify notation, we let $\Omega = [0,1]^d$ in the rest of Appendix \ref{section:proof-minimax-reg}.
We denote by $f\circ T_A$ the function obtained by precomposing $f:A\to \mathbb{R}$ with the affine map $T_A$ in Lemma~\ref{lemma:affine-map}.

\begin{lemma}[Covering number bound for anisotropic Besov spaces]
\label{lemma:covering-aniso-besov}
    Fix a subset of indices $S = \{i_1, \dots, i_s\} \subseteq [d]$ and let $\boldsymbol{\alpha} \in (0,1]^d$. Define the effective harmonic smoothness $\bar{\alpha}$ via the relation $\bar{\alpha}=\paren*{(1/s)\sum_{k=1}^s 1/\alpha_{i_k}}^{-1}$.
    Let $A \subset [0,1]^d$ be an axis-aligned rectangle satisfying the condition $\min_{j \in [d]} \ell_j(A)^{s/\bar{\alpha}} \geq C_1$ for some $C_1 > 0$. Suppose that $(1/2 + \bar{\alpha}/s)^{-1}<p\leq \infty$ and $0<q \leq \infty$. Then for any $\varepsilon>0$
    \begin{equation}
        \log \mathcal{N}\bigl(\varepsilon; B_{p,q}^{\boldsymbol{\alpha}}(A, \besov)_S, \|\cdot\|_{L^2(A)}\bigr) 
        \asymp \left( |A|^{\frac{1}{p} - \frac{1}{2}} \, \varepsilon/\besov \right)^{-s/\bar{\alpha}}.
    \end{equation}
\end{lemma}

\begin{proof}[Proof of Theorem \ref{thm:minimax}]

    By first restricting our attention to fixed sequences $\{S_1,\ldots,S_B\}$ and $\{\boldsymbol{\alpha}_1,\ldots,\boldsymbol{\alpha}_B\}$ such that $|S_b|=s$ and $H(S_b,\balpha_b)=\bar\alpha$ for $b=1,\ldots,B$, we can define the following specific set:
    \begin{equation*}
        \mathcal{B}\coloneq \braces*{f: f|_{G_b} \in \paren*{B_{p,q}^{\boldsymbol{\alpha}_b}(G_b,\besov)}_{S_b}, \forall b\in[B]}.
    \end{equation*}
    Given that $\mathcal{B} \subseteq \PSHABn$, deriving the minimax lower bound over $\mathcal{B}$ is sufficient. Let $\mathbf{v}_1 = (v_1, \ldots, v_B)$ be as defined in Definition \ref{def:quantities}.

    \textit{Step 1: Metric entropy bounds for local covering and packing nets.}
    By Lemma~\ref{lemma:covering-aniso-besov}, for each block $b \in [B]$, the covering number satisfies
    \begin{equation*}
        \log \mathcal{N}\bigl(\varepsilon; (B_{p,q}^{\boldsymbol{\alpha}_b}(G_b, \besov_b))_{S_b}; \|\cdot\|_{L^2(G_b)}\bigr) 
        \asymp 
        \left( |G_b|^{\frac{1}{p} - \frac{1}{2}} \besov_b^{-1}\varepsilon\right)^{-s/\bar{\alpha}}
        =\left( v_b^{-1}\varepsilon\right)^{-s/\bar{\alpha}}.
    \end{equation*}
    In light of the asymptotic equivalence $v_1 \asymp \dots \asymp v_B$, we normalize these quantities by setting $v \coloneqq \min_{b \in [B]} v_b$ and $w_b \coloneqq v_b/v$. This construction ensures that $v_b = w_b v$ with normalized weights satisfying $\min_{b \in [B]} w_b = 1$. Evaluating the above bound at the scaled radius $w_b B^{-1/2}t$ yields
    \begin{equation} \label{eq:local_entropy_bound}
        \log \mathcal{N}\bigl(w_b B^{-1/2}\varepsilon; (B_{p,q}^{\boldsymbol{\alpha}_b}(G_b, \besov_b))_{S_b}; \|\cdot\|_{L^2(G_b)}\bigr) 
        \asymp 
        \bigl( v^{-1}B^{-1/2}\varepsilon \bigr)^{-s/\bar{\alpha}}, \quad \forall b \in [B].
    \end{equation}
    Invoking the standard metric entropy relation $\mathcal{Q}(2\varepsilon; \funcclass; d) \leq \mathcal{N}(\varepsilon; \funcclass; d) \leq \mathcal{Q}(\varepsilon; \funcclass; d)$, which is valid for any function class $\funcclass$, radius $\varepsilon>0$, and metric $d$, it immediately follows that the packing numbers satisfy an analogous bound:
    \begin{equation} \label{eq:local_entropy_bound-2}
        \log \mathcal{Q}\bigl(w_b B^{-1/2}\varepsilon; (B_{p,q}^{\boldsymbol{\alpha}_b}(G_b, \besov_b))_{S_b}; \|\cdot\|_{L^2(G_b)}\bigr) 
        \asymp_{s,\bar\alpha} 
        \bigl( v^{-1}B^{-1/2}\varepsilon \bigr)^{-s/\bar{\alpha}}, \quad \forall b \in [B].
    \end{equation}

    \textit{Step 2: Proof for the lower bound.}
    To lift these local bounds to the global function space $\mathcal{B}$, we construct a global packing set by aggregating the local ones. 
    For each block $b \in [B]$, let $\mathcal{G}_b$ be a $(w_b B^{-1/2}\varepsilon)$-packing set in $L^2(G_b)$ with uniform cardinality $|\mathcal{G}_b|  \eqcolon W \geq  \exp\braces*{C_1(v^{-1}B^{-1/2}\varepsilon)^{-s/\bar{\alpha}} }$, indexed by the set $\mathcal{W} = \{1, \ldots, W\}$.
    Here $C_1$ depends only on $s$ and $\bar\alpha$.
    We define the global packing set $\mathcal{G}$ as the collection of functions whose restrictions to each block reside in the corresponding local packing sets; that is, $\mathcal{G} = \{f : f|_{G_b} \in \mathcal{G}_b, \forall b \in [B]\}$. This construction induces a natural bijection between any $f \in \mathcal{G}$ and an index vector $\mathcal{I}(f) = (\mathcal{I}_1(f), \ldots, \mathcal{I}_B(f)) \in \mathcal{W}^B$, where $\mathcal{I}_b(f) \in \mathcal{W}$ denotes the index of $f|_{G_b}$ within $\mathcal{G}_b$.
    
    Since the blocks $\{G_b\}_{b=1}^B$ are mutually disjoint, the squared $L^2$-distance between any two functions $f, g \in \mathcal{G}$ decomposes additively. Recalling the definition of the local packing sets and the constraint $\min_{b \in [B]} w_b = 1$, we obtain
    \begin{equation*}
        \norm{f-g}_{L^2(\Omega)}^2 = \sum_{b=1}^B \norm{f|_{G_b}-g|_{G_b}}_{L^2(G_b)}^2
        \geq \sum_{b=1}^B w_b^2B^{-1}\varepsilon^2 \1\{\mathcal{I}_b(f) \neq \mathcal{I}_b(g)\}
        \geq B^{-1}\varepsilon^2 d_{\mathrm{H}}\bigl(\mathcal{I}(f), \mathcal{I}(g)\bigr),
    \end{equation*}
    where $d_{\mathrm{H}}(\cdot, \cdot)$ denotes the Hamming distance on $\mathcal{W}^B$.
    
    Invoking Lemma~\ref{lemma:gv}, there exists a subset $\mathcal{T} \subseteq \mathcal{W}^B$ such that 
    \begin{equation*}
        \min_{x, y \in \mathcal{T}, x \neq y} d_{\mathrm{H}}(x, y) \geq \frac{B}{2}, \quad \text{and} \quad |\mathcal{T}| \geq W^{B(1 - H_W(1/2-1/B))},
    \end{equation*}
    where $H_W(\delta)$ denotes the $W$-ary entropy function defined in Lemma~\ref{lemma:gv}. This implies the existence of a $(2^{-1/2}\varepsilon)$-packing net for $\mathcal{G}$ with cardinality at least $W^{B(1 - H_W(1/2-1/B))}$. Consequently, the global metric entropy satisfies
    \begin{equation}\label{eq:global_entropy_sum}
        \begin{split}
            \log \mathcal{Q}\bigl(2^{-1/2}\varepsilon; \mathcal{B}; \|\cdot\|_{L^2(\Omega)}\bigr)
            & \geq  B\bigl(1 - H_W(1/2-1/B)\bigr)\log W \\
            & \geq  \bigl(1 - H_{2}(1/2)\bigr)v^{\frac{s}{\bar\alpha}}B^{\frac{s+2\bar\alpha}{2\bar\alpha}} \varepsilon^{-s/\bar{\alpha}},
        \end{split}
    \end{equation} 
    where the last inequality follows from Lemmas~\ref{lemma:H_Q-1} and \ref{lemma:H_Q-2}, along with the condition $W \geq 2$. Observing that $1 - H_{2}(1/2)$ is a strictly positive absolute constant, the right-hand side of \eqref{eq:global_entropy_sum} is bounded from below by 
    \begin{equation}\label{eq:minimax-lower-bound}
        \log \mathcal{Q}\bigl(\varepsilon; \mathcal{B}; \|\cdot\|_{L^2(\Omega)}\bigr)
        \gtrsim_{s,\bar\alpha} v^{\frac{s}{\bar\alpha}}B^{\frac{s+2\bar\alpha}{2\bar\alpha}} \varepsilon^{-s/\bar{\alpha}}
        \asymp \|\boldsymbol{v}_1\|_{\frac{s}{s+2\bar{\alpha}}}^{\frac{s}{2\bar{\alpha}}}\varepsilon^{-s/\bar{\alpha}},
    \end{equation}
    where the final asymptotic equivalence stems from the fact that $v_1 \asymp \dots \asymp v_B$.

    \textit{Step 3: Proof of the upper bound.}
    Proceeding directly from the local bounds in \eqref{eq:local_entropy_bound}, for each block $b \in [B]$, we can construct a $(w_b B^{-1/2}\varepsilon)$-covering net in $L^2(G_b)$, denoted by $\mathcal{G}_b'$. 
    We define the global covering net $\mathcal{G}'$ as the collection of functions whose restrictions to each block reside in the corresponding local covering nets; that is, $\mathcal{G}' = \{g : g|_{G_b} \in \mathcal{G}_b', \forall b \in [B]\}$. Consequently, for any $f \in \mathcal{B}$, there exists a function $g \in \mathcal{G}'$ such that
    \begin{equation*}
        \norm{f-g}_{L^2(\Omega)}^2 = \sum_{b=1}^B \norm{f|_{G_b}-g|_{G_b}}_{L^2(G_b)}^2
        \leq \sum_{b=1}^B w_b^2 B^{-1}\varepsilon^2 
        \leq C_2\varepsilon^2,
    \end{equation*}
    where the last inequality relies on the fact that $B^{-1}\sum_{b=1}^B w_b^2$ is bounded by a universal constant $C_2$. 
    
    This construction inherently implies that the global covering number is bounded above by the product of the local covering numbers. Setting $C_3= C_2^{1/2}$, we deduce that
    \begin{equation}\label{eq:minimax-upper-bound}
        \begin{split}
            \log \mathcal{N}\bigl(C_3\varepsilon; \mathcal{B}; \|\cdot\|_{L^2(\Omega)}\bigr)
            &\leq \log \paren*{\prod_{b=1}^B\mathcal{N}\bigl(w_b B^{-1/2}\varepsilon; (B_{p,q}^{\boldsymbol{\alpha}_b}(G_b, \besov_b))_{S_b}; \|\cdot\|_{L^2(G_b)}\bigr)}\\
            &\asymp v^{\frac{s}{\bar\alpha}}B^{\frac{s+2\bar\alpha}{2\bar\alpha}} \varepsilon^{-s/\bar{\alpha}}\\
            &\asymp \|\boldsymbol{v}_1\|_{\frac{s}{s+2\bar{\alpha}}}^{\frac{s}{2\bar{\alpha}}}\varepsilon^{-s/\bar{\alpha}},
        \end{split}
    \end{equation}
    where the first asymptotic equivalence follows directly from \eqref{eq:local_entropy_bound}.

    \textit{Step 4: Application of Lemma~\ref{lemma:minimax-suzuki}.}
    Under Assumption~\ref{assum:bound-density}(ii), the $L^2(\mu)$ norm is equivalent to the standard $L^2$ norm (up to constant factors). Consequently, the metric entropy and packing numbers of the class $\mathcal{B}$ satisfy the following bounds for any  $\varepsilon > 0$:
    \begin{align}
        \log \mathcal{N}\bigl(\varepsilon; \mathcal{B}; \|\cdot\|_{2}\bigr)
        &\lesssim_{s,\bar\alpha,c_{\min},c_{\max}} \|\boldsymbol{v}_1\|_{\frac{s}{s+2\bar{\alpha}}}^{\frac{s}{2\bar{\alpha}}} \varepsilon^{-s/\bar{\alpha}}, \label{eq:minimax-upper-entropy} \\
        \log \mathcal{Q}\bigl(\varepsilon; \mathcal{B}; \|\cdot\|_{2}\bigr)
        &\gtrsim_{s,\bar\alpha,c_{\min},c_{\max}} \|\boldsymbol{v}_1\|_{\frac{s}{s+2\bar{\alpha}}}^{\frac{s}{2\bar{\alpha}}} \varepsilon^{-s/\bar{\alpha}}. \label{eq:minimax-lower-packing}
    \end{align}
    
    We now instantiate the critical rates $\varrho_n$ and $\zeta_n$ as
    \begin{equation*}
        \varrho_n = C_4 \|\boldsymbol{v}_1\|_{\frac{s}{s+2\bar\alpha}}^{\frac{s}{2(s+2\bar\alpha)}} n^{-\frac{\bar\alpha}{s + 2\bar\alpha}}
        \quad \text{and} \quad 
        \zeta_n = C_5 \|\boldsymbol{v}_1\|_{\frac{s}{s+2\bar\alpha}}^{\frac{s}{2(s+2\bar\alpha)}} n^{-\frac{\bar\alpha}{s + 2\bar\alpha}}.
    \end{equation*}
    By appropriately selecting constants $C_4$  and $C_5$ that depend only on $s, \bar\alpha, c_{\min}, c_{\max}, K$, and invoking the general bounds \eqref{eq:minimax-upper-entropy} and \eqref{eq:minimax-lower-packing}, we can simultaneously satisfy the following chain of inequalities:
    \begin{equation*}
        \frac{n \varrho_n^2}{2K^2} 
        \leq \log \mathcal{N}\bigl(\varrho_n; \mathcal{B}; \|\cdot\|_{2}\bigr)
        \leq \frac{1}{8} \log \mathcal{Q}\bigl(\zeta_n; \mathcal{B}; \|\cdot\|_{2}\bigr).
    \end{equation*}
    With these conditions verified, the final claim follows directly from an application of Lemma~\ref{lemma:minimax-suzuki}.
\end{proof}

\subsection{Proof of Theorem \ref{thm:minimax-cls}}
\label{section:proof-minimax-cls}

In this section, we derive the minimax lower bound for the classification risk (Definition \ref{def:minimax_risk-cls}). 
We follow the general strategy of \citet{audibert2007fast} and \citet{sun2016stabilized}, which utilizes Assouad's Lemma \citep{audibert2004classification}.
Our proof of the minimax lower bound hence hinges on the construction of a $(t,w,b,b')$-hypercube of probability distributions as introduced in the following definition and lemma.

\begin{definition}[Definition 5.1 in \cite{audibert2004classification}]
\label{def:hypercube}
    Let $t$ be a positive integer, $w\in[0,1]$, $b\in(0,1)$ and $b'\in(0,1)$.
    We say that the collection
    \begin{equation*}
        \mathcal{H}=\Bigl\{\mu_{\boldsymbol{\sigma}}:\ \vec{\sigma}\stackrel{\Delta}{=}(\sigma_1,\ldots,\sigma_t)\in\{-1,+1\}^t\Bigr\}
    \end{equation*}
    of probability distributions $\mu_{\boldsymbol{\sigma}}$ of $(\bX,Y)$ on $\mathcal{Z} \coloneqq [0,1]^d \times \{0,1\}$ is a $(t,w,b,b')$-hypercube if there exists a partition $\{\Omega_j\}_{j=0}^t$ of the domain $\Omega = [0, 1]^d$ such that each $\mu_{\boldsymbol{\sigma}} \in \mathcal{H}$:
    \begin{enumerate}
        \item[(i)] for any $j\in\{0,\ldots,t\}$ and any $\bx\in\Omega_j$, we have
        \begin{equation*}
            \mu_{\boldsymbol{\sigma}}(Y=1\mid \bX=\bx)=\frac{1+\sigma_j\psi(\bx)}{2},
        \end{equation*}
        with $\sigma_0=1$ and $\psi:\Omega\to(0,1]$ satisfies, for any $j\in\{1,\ldots,t\}$,
        \begin{equation*}
            \paren*{1-\Bigl(\mathbb{E}_{\boldsymbol{\sigma}}\braces*{\sqrt{1-\psi^2(\bX)}\mid \bX\in\Omega_j}\Bigr)^2}^{1/2} = b,
        \qquad
        \E_{\boldsymbol{\sigma}}\braces*{\psi(\bX)\mid \bX\in\Omega_j} = b',
        \end{equation*}
        where $\mathbb{E}_{\boldsymbol{\sigma}}$ denotes the expectation with respect to $\boldsymbol{\sigma}$;
        \item[(ii)] its marginal on $\Omega$ is a fixed distribution $\nu$ with $\nu(\Omega_j)=w$ for $j\in\{1,\ldots,t\}$.
    \end{enumerate}
\end{definition}

\begin{lemma}[Lemma 5.1 in \citet{audibert2004classification}]\label{lemma:audibert-lemma}
    If a collection of probability distributions $\mathcal{Q}$ contains a $(t, w, b, b')$-hypercube,
    then for any measurable estimator $\hat f$ measurable with respect to $\mathcal{D}$ there exists a distribution $\mu \in \mathcal{Q}$ with
    \begin{equation*}
        \E\braces*{\excess_{\operatorname{cls}}(\hat f)}\geq 
        twb'(1-b\sqrt{nw})/2,
    \end{equation*}
     where the expectation is taken over $\data=\{(\bX_i,Y_i)\}_{i=1}^n$ with $(\bX_i,Y_i)\sim \mu$ sampled independently.
\end{lemma}

We structure the proof of the minimax lower bound into the following three primary steps:

\begin{itemize}
    \item[\textbf{(i)}] \textbf{Construction of the partition:} We first construct an $r^s$-grid on each component $G_b$, thereby inducing a partition $\{\Omega_0, \Omega_{11}, \dots, \Omega_{Bm}\}$ of the domain $[0,1]^d$, where the number of elements $m \leq r^s$ is a fixed constant to be determined later.
    Building upon this grid and a specific test function $\psi$, we define a $(t,w,b,b')$-hypercube $\mathcal{H}$.
    
    \item[\textbf{(ii)}] \textbf{Verification of assumptions:} For any distribution $\mu_{\boldsymbol{\sigma}} \in \mathcal{H}$, let $\eta_{\boldsymbol{\sigma}}(\mathbf{x}) = \mathbb{P}(Y=1 \mid \mathbf{X} = \mathbf{x})$.
    from the $(t,w,b,b')$-hypercube. 
    We verify that $\eta_{\boldsymbol{\sigma}}$ belongs to the PSHAB space. Furthermore, we demonstrate that $(\mu_{\boldsymbol{\sigma}}$ satisfies the Tsybakov margin condition (Assumption \ref{assum:tsybakov}) as well as the bounded density assumption (Assumption \ref{assum:bound-density}(i)).

    \item[\textbf{(iii)}] \textbf{Application of the reduction lemma:} By leveraging Lemma \ref{lemma:audibert-lemma} and carefully selecting the parameters $w, t$, and $r$, we derive the desired minimax lower bound.
\end{itemize}
    
\begin{proof}[Proof of Theorem \ref{thm:minimax-cls}]

    Let $\balpha=(\bar\alpha,\ldots,\bar\alpha)$ be a $d$-dimensional vector. We define the following class of isotropic functions:
    \begin{equation*} 
            \mathcal{B} \coloneq
            \braces*{f \in L^p([0,1]^d) \colon \forall~b \in [B],\exists S_b = (i_{bk})_{k=1}^s \subset [d] \text{ such that}~f|_{G_b} \in  \paren*{B_{\infty,\infty}^{\boldsymbol{\alpha}}(G_b,\Lambda_b)}_{S_b} }.
    \end{equation*} 
    Then it is evident that 
    $\mathcal{B}\subset\mathcal{B}_{\infty,\infty}^{\mathscr{S},\mathscr{A}}(\partition_*,\boldsymbol{\Lambda})$ 
    and thus 
    \begin{equation}\label{eq:cal-B-vs-Pshab}
        \mathcal{M}_{\operatorname{cls},n}(\mathcal{B}_{\infty,\infty}^{\mathscr{S},\mathscr{A}}(\partition_*,\boldsymbol{\Lambda}))\geq \mathcal{M}_{\operatorname{cls},n}(\mathcal{B}).
    \end{equation}
    In the remainder of the proof, we establish the minimax lower bound for $\eta$ over $\mathcal{B}$. The same lower bound then holds for $\mathcal{B}_{\infty,\infty}^{\mathscr{S},\mathscr{A}}(\partition_*,\boldsymbol{\Lambda})$ by \eqref{eq:cal-B-vs-Pshab}.

\textit{Step 1: Construction of hypercube $\mathcal{H}$ of distributions.} For an integer $r\geq 1$ and each block index $b \in [B]$, we construct a regular grid $V_b$ on the domain $G_b$, defined as
\begin{equation*}
    V_b \coloneq \braces*{ \paren*{ \frac{2t_1+1}{2r}\ell_{1}(G_b), \ldots, \frac{2t_s+1}{2r}\ell_{d}(G_b) }_{S_b} : t_i \in \{0, \ldots, r-1\}, \forall i \in \{1, \ldots, s\} }.
\end{equation*}
For any $\bx \in G_b$, let $n_b(\bx)$ denote the nearest neighbor of $\bx_{S_b}$ within the grid $V_b$. We assume $n_b(\bx)$ is unique; if there are multiple closest points in $V_b$, we define $n_b(\bx)$ to be the one closest to $0$. Fix $m \le r^s$. The grid $V_b$ canonically induces a partition of $G_b$ (that is, $\bx_1$ and $\bx_2$ belong to the same subset if and only if $n_b(\bx_1)=n_b(\bx_2)$); we select $m$ such regions, denoted by $\{\Omega_{b,1}, \ldots, \Omega_{b,m}\}$.
To complete the partition of $[0,1]^d$, define the residual set $\Omega_0 \coloneq [0,1]^d \setminus \bigcup_{b=1}^B \bigcup_{j=1}^{m} \Omega_{b,j}$. Consequently, the collection $\{\Omega_0\} \cup \{ \Omega_{b,j} : b \in [B], j \in [m] \}$ forms a disjoint partition of the domain.

We now define the family of distributions $\mathcal{H} = \{ \mu_{\boldsymbol{\sigma}} : \boldsymbol{\sigma} \in \{0,1\}^{Bm} \}$. 
For any $\mu_{\boldsymbol{\sigma}} \in \mathcal{H}$, the marginal distribution of $\bX$, i.e. $\nu$, is independent of $\boldsymbol{\sigma}$ and admits a density $p_{\bX}$ with respect to the Lebesgue measure, constructed as follows.
Fix a weight parameter $0 < w \leq (Bm)^{-1}$ and let $A_0 \subseteq \Omega_0$ be a measurable set with positive Lebesgue measure.
To explicitly capture the sparsity structure within each subdomain $G_b$ for $b \in [B]$, we introduce anisotropic scaling factors $\boldsymbol{\zeta}^{(b)} \in \mathbb{R}^d$ where $\zeta^{(b)}_j = r$ if $j \in S_b$ and $\zeta^{(b)}_j = 1$ otherwise. For any $\bx \in G_b$, define the rescaled coordinates $\bx^{(b)}$ element-wise by $x^{(b)}_j \coloneq \zeta^{(b)}_j x_j / \ell_j(G_b)$.
Associated with each grid point $z \in V_b$, we define a mapped ball $B_b(z, 1/4)$ in the original domain $G_b$ via the condition on rescaled coordinates:
\begin{equation*}
    B_b(z, 1/4) \coloneq \braces*{ \bx \in G_b : \norm*{ (\bx^{(b)} - z^{(b)})_{S_b} }_2 \le \frac{1}{4} }.
\end{equation*}
Finally, the marginal density $p_{\bX}(\bx)$ is defined as:
\begin{equation}
    p_{\bX}(\bx) = 
    \begin{cases}
        \frac{w}{|B_b(z, 1/4)|} & \text{if } \bx \in B_b(z, 1/4) \text{ for some } z \in V_b, b \in [B], \\
        \frac{1 - Bmw}{|A_0|} & \text{if } \bx \in A_0, \\
        0 & \text{otherwise}.
    \end{cases}
\end{equation}

\textit{Step 2: Construction of the regression function $\eta_{\boldsymbol{\sigma}}$.}
First, let $u: \mathbb{R}_+ \to [0,1]$ be a non-increasing, infinitely differentiable function satisfying $u(x) = 1$ for $x \in [0, 1/4]$ and $u(x) = 0$ for $x \ge 1/2$. An explicit construction of such a function can be found in Section~6.2 of \citet{audibert2007fast}. Based on $u$, we define the anisotropic bump function $\phi_b$ for each $b \in [B]$ as 
\begin{equation*}
    \phi_b(\bx) \coloneq C_\phi \norm{\bLambda}_{\infty} u(\norm{\bx_{S_b}}_{2}),
\end{equation*}
where the constant $C_\phi > 0$ is chosen sufficiently small to ensure that $|\phi_b(\bx)| \le \Lambda_b$. Crucially, as shown in \citet{audibert2007fast}, this choice also guarantees the smoothness condition:
\begin{equation}\label{eq:phi:holder}
    |\phi_b(\bx_1) - \phi_b(\bx_2)| 
    \le \Lambda_b \norm{(\bx_1 - \bx_2)_{S_b}}_2^{\bar\alpha} 
    \le \Lambda_b \norm{(\bx_1 - \bx_2)_{S_b}}_{\bar\alpha}^{\bar\alpha},
\end{equation}
for any $\bx_1, \bx_2 \in G_b$. We recall that $C_\phi$ can be chosen uniformly across $b$ due to the equivalence $\Lambda_1 \asymp \ldots \asymp \Lambda_B$.

Next, we specify the conditional distribution of $Y$ given $\bX$ for any index $\boldsymbol{\sigma} \in \mathcal{H}$. The regression function $\eta_{\boldsymbol{\sigma}}(\bx) = \mathbb{P}(Y=1 \mid \bX=\bx)$ is defined as:
\begin{equation*}
    \eta_{\boldsymbol{\sigma}}(\bx) = \frac{1 + \delta_{\boldsymbol{\sigma}}(\bx)}{2},
\end{equation*}
where the perturbation term $\delta_{\boldsymbol{\sigma}}(\bx)$ is given by
\begin{equation*}
    \delta_{\boldsymbol{\sigma}}(\bx) = 
    \begin{cases}
        \sigma_{b,j} \psi_b(\bx) & \text{if } \bx \in \Omega_{b,j} \text{ for some } b \in [B], j \in [m], \\
        0 & \text{if } \bx \in \Omega_0.
    \end{cases}
\end{equation*}
Here, $\boldsymbol{\sigma}$ is indexed as $(\sigma_{b,j})_{b,j}$ with $\sigma_{b,j} \in \{-1, 1\}$. The localized perturbation function $\psi_b$ is defined by rescaling and shifting the base bump $\phi_b$:
\begin{equation}
    \psi_b(\bx) \coloneq \bigl(r/\ell\bigr)^{-\bar\alpha} \phi_b \paren{\bx^{(b)} - n_b(\bx)^{(b)}},
\end{equation}
where $\ell \coloneq \min_{b \in [B]} \min_{j \in [d]} \ell_j(G_b)$ and $\bx^{(b)}$ denotes the rescaled coordinates defined in Step 1.
Recalling the geometric property $\ell_j(G_b) \asymp \ell \asymp B^{-1/d}$, we must ensure that the regression function $\eta_{\boldsymbol{\sigma}}$ remains within $[0,1]$. This requirement is satisfied provided $|\delta_{\boldsymbol{\sigma}}(\bx)| \le 1$, which imposes the following constraint on the scaling constants:
\begin{equation}\label{eq:condition-minimax-cls}
    C_\phi \norm{\bLambda}_{\infty} \le B^{\bar\alpha/d} r^{\bar\alpha},
\end{equation}
a condition that is verified in Step 6.

    \textit{Step 3: Verification of PSHAB membership.}
We now verify that the constructed regression function satisfies the smoothness constraints, i.e., $\eta_{\boldsymbol{\sigma}} \in \mathcal{B}$. Consider any two points $\bx_1, \bx_2 \in G_b$.

\noindent \textbf{Case 1:} $n_b(\bx_1) = n_b(\bx_2)$.
In this case, both points belong to the same local neighborhood associated with a single grid point. We have:
\begin{equation}\label{eq:minimax-cls-check-minimax-1}
    \begin{split}
        |\eta_{\boldsymbol{\sigma}}(\bx_1) - \eta_{\boldsymbol{\sigma}}(\bx_2)|
        &= \frac{1}{2} |\psi_b(\bx_1) - \psi_b(\bx_2)| \\
        &= \frac{1}{2} (r/\ell)^{-\bar\alpha} \abs*{ \phi_b \paren{\bx_1^{(b)} - n_b(\bx_1)^{(b)}} - \phi_b \paren{\bx_2^{(b)} - n_b(\bx_2)^{(b)}} } \\
        &\le \frac{1}{2} (r/\ell)^{-\bar\alpha} \Lambda_b \norm*{ (\bx_1^{(b)} - \bx_2^{(b)})_{S_b} }_{\bar\alpha}^{\bar\alpha} \\
        &\le C_{\bar\alpha} \norm{\bLambda}_{\infty} \norm*{ (\bx_1 - \bx_2)_{S_b} }_{\bar\alpha}^{\bar\alpha},
    \end{split}
\end{equation}
where the last second line uses \eqref{eq:phi:holder}, and the final inequality follows from the scaling definition $\bx^{(b)}_j \asymp (r/\ell) \bx_j$ and the property $\Lambda_b \asymp \norm{\bLambda}_{\infty}$.

\noindent \textbf{Case 2:} if $n_b(\bx_1) \neq n_b(\bx_2)$. Without loss of generality, we assume that $\bx_1 \in \Omega_{b,1}$ and $\bx_2 \in \Omega_{b,2}$ (the case where at least one of $\bx_1$ and $\bx_2$ lies in $\Omega_0$ follows a similar argument). Let $\bx_3$ and $\bx_4$ denote the intersection points of the line segment connecting $\bx_1$ and $\bx_2$ with the boundaries of $\Omega_{b,1}$ and $\Omega_{b,2}$, respectively. By the definition of $u(x)$, it is evident that $\psi(\bx_3) = \psi(\bx_4) = 0$, and thus
\begin{equation}\label{eq:minimax-cls-check-minimax-2}
    \begin{split}
        |\eta_{\boldsymbol{\sigma}}(\bx_1) - \eta_{\boldsymbol{\sigma}}(\bx_2)|
        &\le |\eta_{\boldsymbol{\sigma}}(\bx_1) - \eta_{\boldsymbol{\sigma}}(\bx_3)| + |\eta_{\boldsymbol{\sigma}}(\bx_4) - \eta_{\boldsymbol{\sigma}}(\bx_2)| \\
        &= \frac{1}{2}|\psi(\bx_1) - \psi(\bx_3)| + \frac{1}{2}|\psi(\bx_4) - \psi(\bx_2)| \\
        &\le C_{\bar\alpha} C_\phi \norm{\bLambda}_{\infty} \norm{(\bx_1 - \bx_3)_{S_b}}_{\bar\alpha}^{\bar\alpha} + C_{\bar\alpha} C_\phi \norm{\bLambda}_{\infty} \norm{(\bx_4 - \bx_2)_{S_b}}_{\bar\alpha}^{\bar\alpha} \\
        &\le 2 C_{\bar\alpha} C_\phi \norm{\bLambda}_{\infty} \norm{(\bx_1 - \bx_2)_{S_b}}_{\bar\alpha}^{\bar\alpha},
    \end{split}
\end{equation}
where the penultimate inequality follows from \eqref{eq:minimax-cls-check-minimax-1}. Combining \eqref{eq:minimax-cls-check-minimax-1} and \eqref{eq:minimax-cls-check-minimax-2} confirms that $\eta_{\boldsymbol{\sigma}} \in \mathcal{B}$ if $C_{\phi}$ is small enough.

\textit{Step 4: Verification of Assumption \ref{assum:tsybakov}.}
We now verify that the constructed distribution satisfies the margin assumption. 
Let $\bx_0 = \bigl(\ell_{1}(G_1)/(2r), \ldots, \ell_{d}(G_1)/(2r)\bigr)$ be the center of the first grid cell. For any $\boldsymbol{\sigma} \in \{-1,1\}^{Bm}$, denote the corresponding probability measure by $\mathbb{P}_{\boldsymbol{\sigma}}$. We evaluate the margin probability on the first block $G_1$:
\begin{equation*}
    \begin{split}
        &\mathbb{P}_{\boldsymbol{\sigma}}\paren*{0 < |\eta_{\boldsymbol{\sigma}}(\bX) - 1/2| \le t \mid \bX \in G_1} \\
        &= m \mathbb{P}_{\boldsymbol{\sigma}}\paren*{0 < \psi_1(\bX) \le 2t \mid \bX \in \Omega_{1,1}} \\
        &= m \mathbb{P}_{\boldsymbol{\sigma}}\paren*{0 < (r/\ell)^{-\bar\alpha} \phi_1(\bX^{(1)} - \bx_0^{(1)}) \le 2t \mid \bX \in \Omega_{1,1}} \\
        &= m \int_{B_1(\bx_0, 1/4)} \1\braces*{ 0 < \phi_1(\bx^{(1)} - \bx_0^{(1)}) \le 2t (r/\ell)^{\bar\alpha} } \frac{w}{|B_1(\bx_0, 1/4)|} \, d\bx \\
        &= \frac{mw}{|B(\mathbf{0}, 1/4)|} \int_{B(\mathbf{0}, 1/4)} \1\braces*{ 0 < \phi_1(\bz) \le 2t (r/\ell)^{\bar\alpha} } \, d\bz \quad (\text{via change of variables } \bz = \bx^{(1)} - \bx_0^{(1)}) \\
        &= mw \, \1\braces*{ t \ge \frac{C_\phi \norm{\bLambda}_{\infty}}{2 (r/\ell)^{\bar\alpha}} }.
    \end{split}
\end{equation*}
Aggregating over all blocks $b \in [B]$, we obtain:
\begin{equation}
    \mathbb{P}_{\boldsymbol{\sigma}}\paren*{0 < |\eta_{\boldsymbol{\sigma}}(\bX) - 1/2| \le t} 
    = Bmw \, \1\braces*{ t \ge \frac{C_\phi \norm{\bLambda}_{\infty}}{2 (r/\ell)^{\bar\alpha}} }.
\end{equation}
Recalling that $\ell \asymp B^{-1/d}$, Assumption \ref{assum:tsybakov} is satisfied provided that:
\begin{equation}\label{eq:condition-minimax-cls-2}
    Bmw \le C_1 \norm{\bLambda}_{\infty}^\rho (r B^{1/d})^{-\rho \bar\alpha},
\end{equation}
where $C_1$ is a constant depending only on $M$, $\rho$, $\bar\alpha$, and $C_\phi$. Condition \eqref{eq:condition-minimax-cls-2} will be verified in Step 5.

    \textit{Step 5: Parameter selection and application of Lemma \ref{lemma:audibert-lemma}.}
Invoking Lemma \ref{lemma:audibert-lemma}, for any classifier $\hat{f}$, the minimax risk is lower-bounded by
\begin{equation}\label{eq:minimax-cls-goal}
    \sup_{\mu \in \mathcal{H}} \mathbb{E} \braces*{ \excess_{\operatorname{cls}}(\hat{f}) } 
    \ge 
    \frac{1}{2} B m w b' (1 - b\sqrt{nw}),
\end{equation}
where  $b$ and $b'$ are defined and calculated as:
\begin{align*}
    b 
    &\coloneq \paren*{ 1 - \Bigl( \mathbb{E}_{\boldsymbol{\sigma}} \braces*{ \sqrt{1 - \psi^2(\bX)} \mid \bX \in \Omega_{b,j} } \Bigr)^2 }^{1/2} 
    = C_\phi \norm{\bLambda}_{\infty} (r/\ell)^{-\bar\alpha}, \\
    b' 
    &\coloneq \mathbb{E}_{\boldsymbol{\sigma}} \braces*{ \psi(\bX) \mid \bX \in \Omega_{b,j} } 
    = C_\phi \norm{\bLambda}_{\infty} (r/\ell)^{-\bar\alpha}.
\end{align*}

To satisfy the conditions of the lemma and optimize the bound, we select the set $A_0$ to be a Euclidean ball contained within $\Omega_0$. We set the number of bins $m = r^s/2$ and specify the scaling parameters $w$ and $r$ as follows:
\begin{align*}
    w &= C_2 \norm{\bLambda}_{\infty}^{-\frac{2s}{s+(2+\rho)\bar\alpha}} B^{-\frac{2(d-s)\bar\alpha}{d(s+(2+\rho)\bar\alpha)}} n^{-\frac{s+\rho\bar\alpha}{s+(2+\rho)\bar\alpha}}, \\
    r &= \biggl\lfloor C_3 \norm{\bLambda}_{\infty}^{\frac{2+\rho}{s+(2+\rho)\bar\alpha}} B^{-\frac{d+(2+\rho)\bar\alpha}{d(s+(2+\rho)\bar\alpha)}} n^{\frac{1}{s+(2+\rho)\bar\alpha}} \biggr\rfloor,
\end{align*}
where $C_2$ and $C_3$ are positive constants depending only on $s, \bar\alpha$, and $\rho$. 
By choosing $C_3$ sufficiently large and $C_2$ sufficiently small, we ensure that the constraints $0 < w \le 1$ and $r \ge 1$ are satisfied, and that the condition \eqref{eq:condition-minimax-cls-2} holds.

   Substituting the selected parameters back into \eqref{eq:minimax-cls-goal}, we obtain the lower bound:
    \begin{equation*}
        \sup_{\mu_{\boldsymbol{\sigma}}\in\mathcal{H}} \mathbb{E} \braces*{\excess_{\operatorname{cls}}(\hat f)} 
        \geq 
        C_4 \norm{\bLambda}_{\infty}^{\frac{(1+\rho)s}{s+(2+\rho)\bar\alpha}} \paren*{ \frac{B^{\frac{d-s}{d}}}{n} }^{\frac{(1+\rho)\bar\alpha}{s+(2+\rho)\bar\alpha}},
    \end{equation*}
    where $C_4$ is a positive constant depending only on $s, \bar\alpha$, and $\rho$. 
    Finally, the asserted bound \eqref{eq:minimax-conclude-cls} follows from the observation that the constraint $\log B \lesssim d/s$ implies $B^{-s/d} \geq c$ for some universal constant $c > 0$.
    
    \textit{Step 6: Verification of condition \eqref{eq:condition-minimax-cls} and Assumption \ref{assum:bound-density}(i).}
It remains to verify the compatibility conditions derived earlier. First, substituting the selected expressions for $w$ and $r$ into \eqref{eq:condition-minimax-cls}, we find that this condition implies a lower bound on the sample size:
\begin{equation*}
    n \ge C_5 B^{1-s/d} \norm{\bLambda}_{\infty}^{s/\bar\alpha},
\end{equation*}
where $C_5$ is a positive constant depending only on $s$, $\bar\alpha$, and $\rho$. Since $B^{-s/d} \le 1$, a sufficient condition for this to hold is $n \ge C_5 B \norm{\bLambda}_{\infty}^{s/\bar\alpha}$.

Next, we verify the bounded density assumption (Assumption \ref{assum:bound-density}(i)). Consider the density on the support of the perturbations. For any $\bx \in B_b(z, 1/4)$ with $z \in V_b$ and $b \in [B]$, the density is given by $\mu(\bx) = w / |B_b(z, 1/4)|$. By construction, the volume of the mapped ball scales as $|B_b(z, 1/4)| \asymp |G_b| r^{-s} \asymp B^{-1} r^{-s}$. Recalling that $m \asymp r^s$, we have
\begin{equation*}
    \mu(\bx) \asymp \frac{w}{B^{-1} r^{-s}} = B w r^s \asymp B m w.
\end{equation*}
Substituting the definitions of $w$ and $r$ into the expression for $Bmw$ yields:
\begin{equation}\label{eq:minimax-cls-check-minimax-3}
    Bmw = C_6 \paren*{ \norm{\bLambda}_{\infty}^{\frac{s}{\bar\alpha}} B^{\frac{d-s}{d}} n^{-1} }^{\frac{\rho\bar\alpha}{2+(2+\rho)\bar\alpha}},
\end{equation}
where the constant $C_6$ depends only on $s$, $\bar\alpha$, $\rho$, and the pre-factor $C_2$. Under the sample size condition $n \ge C_5 B^{1-s/d} \norm{\bLambda}_{\infty}^{s/\bar\alpha}$, the base term in parentheses is bounded. Consequently, by choosing the constant $C_2$ (in the definition of $w$) sufficiently small, we ensure that $C_6$ is small enough such that the right-hand side of \eqref{eq:minimax-cls-check-minimax-3} is strictly less than $1$ (and can be made arbitrarily small). This establishes a uniform upper bound $\mu(\bx) \le C_0$ on the union of the balls.

Finally, on the residual set $A_0$, we have $\mu(\bx) = (1 - Bmw) / |A_0| \le 1/|A_0|$. Since $A_0$ is a fixed set with positive Lebesgue measure, $\mu(\bx)$ is uniformly bounded on $A_0$. Thus, $\mu(\bx)$ is bounded uniformly over the entire domain $[0,1]^d$.
\end{proof}

\section{Auxiliary proofs}
\subsection{Proof of Remarks \ref{rmk:Interpretability-accuracy tradeoff} and \ref{rmk:Interpretability-accuracy tradeoff}}
\label{appendix:proof-remark}

\begin{proof}[Proof of Remark \ref{rmk:Interpretability-accuracy tradeoff}]
If $\Excess_{\mathrm{reg},L}\leq 2(\fbound+K)^2(L\log(nd)+u)/n$, then by \eqref{eq:thm-statement:regression-1}, taking $\delta=1/2$ yields
\begin{equation}
\begin{split}
\excess_{\mathrm{reg}}(\hat f_{L})
\;\;&\leq\;\;
3\,\Excess_{\mathrm{reg},L}
+6\,\frac{C(\fbound+K)^2\big(L\log(nd)+u\big)}{n}\\
&\leq
\Excess_{\mathrm{reg},L}
+(4+6C)\,\frac{C(\fbound+K)^2\big(L\log(nd)+u\big)}{n}.
\end{split}
\end{equation}
Taking square roots on both sides then yields \eqref{eq:reg_erm_opt_a}. Otherwise, \eqref{eq:thm-statement:regression-1} yields
\begin{equation*}
    \excess_{\mathrm{reg}}(\hat f_{L})\;\;\leq\;\;
    \Excess_{\mathrm{reg},L}+\frac{2}{1-\delta}\,\paren*{\delta\Excess_{\mathrm{reg},L}+\,\frac{C(\fbound+K)^2\big(L\log(nd)+u\big)}{\delta n}}+\frac{C(\fbound+K)^2\big(L\log(nd)+u\big)}{n}.
\end{equation*}
Letting $\delta = (\fbound+K)^2(L\log(nd)+u)/(\Excess_{\mathrm{reg},L}\, n)$, we have $\delta\leq 1/2$. It then follows from the above displays that
\begin{equation*}
\begin{split}
\excess_{\mathrm{reg}}(\hat f_{L})
\;\;&\leq\;\;
\Excess_{\mathrm{reg},L}
+4(C+1)\,\paren*{\Excess_{\mathrm{reg},L}\,
\frac{C(\fbound+K)^2\big(L\log(nd)+u\big)}{n}}^{1/2}
+\frac{C(\fbound+K)^2\big(L\log(nd)+u\big)}{n}\\
&\leq
\paren*{\Excess_{\mathrm{reg},L}^{1/2}
+C_1\paren*{\frac{C(\fbound+K)^2\big(L\log(nd)+u\big)}{n}}^{1/2}}^2.
\end{split}
\end{equation*}
Taking square roots on both sides yields \eqref{eq:reg_erm_opt_a}.
\end{proof}

\begin{proof}[Proof of Remark \ref{rmk:Interpretability-accuracy tradeoff-cls}]
This is analogous to the phenomenon described in Remark~\ref{rmk:Interpretability-accuracy tradeoff}.
\end{proof}

\subsection{Proofs for Section \ref{section:heavy-tailed}}
\label{section:proof-regression-heavy}

\begin{lemma}[High probability bound for finite maxima]
\label{lem:finite_max_bound}
Let $X_1, \dots, X_n$ be random variables in an Orlicz space $L^\Phi$ defined by a Young function $\Phi$. Let $U = \max_{1 \le i \le n} \|X_i\|_\Phi$. For any $\delta \in (0, 1)$, with probability at least $1 - \delta$, we have
$$ \max_{1 \le i \le n} |X_i| \le U \Phi^{-1}\left(\frac{n}{\delta}\right). $$
\end{lemma}

\begin{proof}
Without loss of generality, assume $U > 0$. By the definition of the Luxemburg norm, we have $\mathbb{E}[\Phi(|X_i|/U)] \le 1$ for all $i$. For any $t > 0$, applying the union bound and Markov's inequality yields
\begin{align*}
    \mathbb{P}\left( \max_{1 \le i \le n} |X_i| > t \right) 
    &\le \sum_{i=1}^n \mathbb{P}\left( |X_i| > t \right) \\
    &= \sum_{i=1}^n \mathbb{P}\left( \Phi\left(\frac{|X_i|}{U}\right) > \Phi\left(\frac{t}{U}\right) \right) \\
    &\le \sum_{i=1}^n \frac{\mathbb{E}[\Phi(|X_i|/U)]}{\Phi(t/U)} 
    \le \frac{n}{\Phi(t/U)}.
\end{align*}
Setting the right-hand side equal to $\delta$, we obtain $\Phi(t/U) = n/\delta$. Solving for $t$, we choose $t = U \Phi^{-1}(n/\delta)$, which completes the proof.
\end{proof}

\begin{proof}[Proof of Theorem \ref{thm:regression-heavy}]
By Lemma \ref{lem:finite_max_bound}, we have
\begin{equation}
    \begin{split}
        \P\braces*{\max\{|\xi_1|,\ldots,|\xi_n|\}
        \leq K~|~\bX_1,\bX_2,\ldots,\bX_n}
        \geq 1-p_0
        \label{eq:thm-regress-psi-1}
    \end{split}
\end{equation}
for any choice of $\bX_1,\bX_2,\ldots,\bX_n$.
On this event, the conditional distributions of the noise variables are
sub-Gaussian with sub-Gaussian norm bounded by $K$. Consequently, \eqref{eq:thm-statement:regression-heavy} follows from Theorem~\ref{thm:regression}.
\end{proof}

\begin{proof}[Proof of Theorem \ref{thm:regression-exact-rate-heavy}]
The proof follows the same strategy as that of Theorem~\ref{thm:regression-heavy}. We condition on the event in \eqref{eq:thm-regress-psi-1}, under which the conditional distributions of the noise variables are sub-Gaussian, and then apply Theorem~\ref{thm:regression-exact-rate-PSHAB}.
\end{proof}

\subsection{Empirical equivalence between \texorpdfstring{$\mathcal{P}_L$}{P\_L} and \texorpdfstring{$\mathcal{P}_L^{\mathcal{X}}$}{P\_L\textasciicircum X}}
\label{section:understanding-tree-partition}

In this section, we show that the infinite set of all tree-based partitions $\partitions{\leaves}$ can be faithfully represented by the finite set of valid empirical partitions $\emppartitions{\leaves}$. This equivalence is crucial for establishing uniform concentration, as it allows us to bound the empirical complexity of the tree space using the finite cardinality of $\emppartitions{\leaves}$.

\begin{definition}
Fix a sample $\sample = \{\bX_1, \ldots, \bX_n\}$. Two cells $A$ and $A'$ are said to be \textit{$\sample$-equivalent} (denoted $A \sequiv A'$) if they contain the exact same subset of data points, i.e., $\1_A(\bX_i) = \1_{A'}(\bX_i)$ for all $i \in [n]$. 
Similarly, two partitions $\partition$ and $\partition'$ are \textit{$\sample$-equivalent} (denoted $\partition \sequiv \partition'$) if for every cell $A \in \partition$, there exists a cell $A' \in \partition'$ such that $A \sequiv A'$.
\end{definition}

If $A \sequiv A'$ or $\partition \sequiv \partition'$, they can be regarded as the same cell or tree partition empirically. This is because any potential splits for the two cells or partitions are identical in terms of their effect on the sample.

\begin{lemma}
\label{lemma:partition-equiv-sample-population}
For any tree-based partition $\partition \in \partitions{\leaves}$, there exists an integer $\leaves' \leq \leaves$ and a valid tree-based partition $\partition'\in \emppartitions{\leaves'}$ such that $\partition\sequiv\partition'$.
\end{lemma}

\begin{proof}
We construct $\partition'$ constructively from $\partition$ through a simple top-down modification of the decision tree that generates $\partition$.

First, for any internal node of the tree that splits a cell along coordinate $j$ at threshold $\tau$, we adjust the threshold to $\tau' = \max \{ X_{ij} : X_{ij} \leq \tau, i \in [n] \}$ (setting $\tau'=0$ if no such data point exists). Because the interval $(\tau', \tau]$ contains no observed data points in the $j$-th coordinate, the condition $x_j \leq \tau$ is empirically identical to $x_j \leq \tau'$ for all $\bx \in \sample$. Applying this adjustment to every split in the tree yields a new partition where all split thresholds belong to the observed data values, without altering the empirical assignment of any data point.

Second, we prune any empirically empty splits. If a split routes all of a cell's empirical data points to one child (leaving the other child empty), the split is redundant. We delete the split, assign the parent cell entirely to the non-empty child, and remove the empty branch. 

Because each threshold adjustment preserves $\sample$-equivalence, and each pruning step preserves $\sample$-equivalence while strictly decreasing the number of leaves, the resulting tree defines a valid data-driven partition $\partition' \in \emppartitions{\leaves'}$ with $\leaves' \leq \leaves$ and $\partition' \sequiv \partition$.
\end{proof}

\subsection{Proof of Lemma \ref{cor:approx-aniso-besov-piecewise}}



We define a collection of dyadic rectangles according to the given anisotropic smoothness $\boldsymbol{\alpha}$. 
For any fixed level $j \in \mathbb{N}$, let $\mathcal{G}_j^{\boldsymbol{\alpha}}$ denote the set of all dyadic rectangles $\times_{i=1}^d I_i \subset [0,1]^d$ such that, for all $1 \leq i \leq d$:
\begin{equation*}
   I_i=\Big[0,2^{-\lfloor j\alpha_{\min}/\alpha_i\rfloor}\Big] \quad\text{ or }
   \quad
   I_i=\Big(k_i2^{-\lfloor j\alpha_{\min}/\alpha_i\rfloor},
   (k_i+1)2^{-\lfloor j\alpha_{\min}/\alpha_i\rfloor}\Big],
\end{equation*}
where $k_i\in \Big\{1,\ldots,2^{\lfloor j\alpha_{\min}/\alpha_i\rfloor}-1\Big\}$. The set of all dyadic rectangles across all levels is defined as $\mathcal{G}^{\boldsymbol{\alpha}} \coloneq \cup_{j \in \mathbb{N}} \mathcal{G}^{\boldsymbol{\alpha}}_j$. 
In Section~2.2 of \citet{akakpo2012adaptation}, the author designs an algorithm that constructs a tree-based partition whose elements belong to $\mathcal{G}^{\boldsymbol{\alpha}}$, 
and establishes the optimal approximation theorem for anisotropic Besov spaces using piecewise dyadic constant functions. 
Specifically, given a partition $\partition = \{A_j\}_{j=1}^\leaves$, the class of piecewise dyadic constant functions based on the partition $\partition$ is defined as
\begin{equation*}
    \mathcal{S}_{\partition}(\leaves)\coloneq \bigg\{\sum_{j=1}^\leaves a_j\1_{A_j}:a_j\in\mathbb{R}\bigg\}.
\end{equation*}

\begin{lemma}[Corollary 1 in \citet{akakpo2012adaptation}]
\label{lemma:approx-aniso-besov-piecewise}
    Let $\boldsymbol{\alpha}\in (0,1)^d$, $0 < p \leq \infty$, and $1 \leq m \leq \infty$ such that
    \begin{equation*}
        \bar\alpha / d > (1/p - 1/m)_+.
    \end{equation*}
    Assume $f \in B_{p,q}^{\boldsymbol{\alpha}}([0,1]^d, \besov)$, where $(p,q)$ satisfy the one of the two following conditions:
    \begin{enumerate}
        \item $0\leq q\leq \infty$, $0<p\leq 1$ or  $m\leq p\leq\infty$.
        \item $0\leq q\leq p$, $1<p<m$.
    \end{enumerate}
    Then there exists a constant $C_1$ that depends only on $d$, $\alpha_{\min},\bar\alpha$, and $p$ and a tree-based partition $\partition$ whose elements all belong to $\mathcal{D}^{\boldsymbol{\alpha}}$ such that, 
    for any $\leaves \geq C_1$,
    \begin{equation*}
        \inf_{\tilde{f} \in \mathcal{S}_{\partition}(\leaves)}
        \norm{\tilde{f} - f}_{L^m([0,1]^d)} \leq C_{d,\alpha_{\min},\bar\alpha,p} \besov \leaves^{-\bar\alpha / d}.
    \end{equation*}
\end{lemma}

\begin{remark}
    There are several different key differences between the statements of Lemma \ref{lemma:approx-aniso-besov-piecewise} and Corollary 1 in \citet{akakpo2012adaptation}.
    \begin{itemize}
        \item Corollary~1 in \citet{akakpo2012adaptation} does not explicitly state that the partition $\partition$ is tree-based. 
        However, in their proof, $\partition$ is indeed constructed by the algorithm described in Section~2.2 of \citet{akakpo2012adaptation}, 
        which in fact yields a tree-based partition.
        \item Although Corollary~1 in \citet{akakpo2012adaptation} does not include the case $p = \infty$, 
        Lemma~\ref{lemma:approx-aniso-besov-piecewise} also covers the space $B_{\infty,\infty}^{\boldsymbol{\alpha}}([0,1]^d, \besov)$, 
        which corresponds precisely to the anisotropic H{\"o}lder space. 
        This follows from the relationship $\|f\|_{B_{m,\infty}^{\boldsymbol{\alpha}}([0,1]^d)}\lesssim \|f\|_{B_{\infty,\infty}^{\boldsymbol{\alpha}}[0,1]^d}$ for any $1\leq m\leq \infty$, 
        which yields the embedding 
        $B_{\infty,\infty}^{\boldsymbol{\alpha}}([0,1]^d,\besov)
        \subseteq B_{m,\infty}^{\boldsymbol{\alpha}}([0,1]^d,C\besov)$ for a universal constant, 
        and we note that the conclusion holds for $B_{m,\infty}^{\boldsymbol{\alpha}}([0,1]^d,\Lambda)$.
        \item Although the constants $C_1$ and $C_2$ in \citet{akakpo2012adaptation} are stated as depending on $\boldsymbol{\alpha}$, an inspection of the proof (see page~25 therein) reveals that they depend only on $\alpha_{\min}$ and $\bar\alpha$.
    \end{itemize}
\end{remark}

We now recall several embedding results for anisotropic Besov spaces. The next lemma shows that the anisotropic Besov space with smoothness parameter $\boldsymbol{\alpha}$ is embedded into the space with smoothness $\gamma\boldsymbol{\alpha}$, where $0<\gamma\leq 1$. Related embeddings can also be found in \cite{anisot-besov-3} and \cite{embedding2008embedding}.

\begin{lemma}[Proposition~1 in \citet{suzuki2021deep}]
\label{lemma:embedding-anisotropic-besov-1}
There exist the following relations between the spaces:
\begin{enumerate}
    \item Let $0 < p_1, p_2, q \leq \infty$, $p_1 \leq p_2$, and $\boldsymbol{\alpha} \in \mathbb{R}_{+}^{d}$\footnote{We let $\mathbb{R}_{+} \coloneq \{x \in \R : x > 0\}$.} with $\bar\alpha/d > (1/p_1 - 1/p_2)_+$. 
    Set $\gamma = 1 - (1/p_1 - 1/p_2)_+\cdot d/\bar\alpha$ and $\boldsymbol{\alpha'} = \gamma \boldsymbol{\alpha}$, then $B_{p_1,q}^{\boldsymbol{\alpha}}([0,1]^d,\besov)
    \hookrightarrow B_{p_2,q}^{\boldsymbol{\alpha'}}([0,1]^d,\besov)$\footnote{ The symbol $\hookrightarrow$ denotes a continuous embedding: for two normed spaces $X$ and $Y$, $X\hookrightarrow Y$ if $X\subseteq Y$ and $\exists C>0$, s.t. $\|x\|_Y\leq C\|x\|_X$ for all $x\in X$.}.
    \item Let $0 < p, q_1,q_2 \leq \infty$, $q_1 <q_2$, and $\boldsymbol{\alpha} \in \mathbb{R}_{++}^{d}$, then it holds
    $B_{p,q_1}^{\boldsymbol{\alpha}}([0,1]^d,\besov)
    \hookrightarrow B_{p,q_2}^{\boldsymbol{\alpha}}([0,1]^d,\besov)$.
\end{enumerate}
\end{lemma}

\begin{corollary}
\label{cor:embedding}
    Let $0 < p, q \leq \infty$, $\boldsymbol{\alpha} \in (0,1]^d$, and assume $\max_{1 \leq i \leq d} \alpha_i = 1$. 
For any $\epsilon_1>0$, define 
$\boldsymbol{\alpha'} = (1-\epsilon_1)\boldsymbol{\alpha}$ 
and $\epsilon_2 = p^2\bar\alpha\epsilon_1/(d+p\bar\alpha\epsilon_1)$. 
Then the embedding holds
\begin{equation*}
    B_{p-\epsilon_2,q}^{\boldsymbol{\alpha}}([0,1]^d,\besov)
    \hookrightarrow 
    B_{p,\, q}^{\boldsymbol{\alpha'}}([0,1]^d,\besov).
\end{equation*}

\end{corollary}
\begin{proof}
    Apply the first claim of Lemma~\ref{lemma:embedding-anisotropic-besov-1} with 
    $1/p - 1/p_2 = \epsilon_1 \bar\alpha / d$ directly.
\end{proof}

\begin{proof}[Proof of Lemma \ref{cor:approx-aniso-besov-piecewise}] 
    Let $C_1$ be the constant in Lemma~\ref{lemma:approx-aniso-besov-piecewise} depending only on $d$, $\alpha_{\min}$, and $\bar\alpha$. If $\leaves<C_1$, then \eqref{eq:conclud-approx-anis-besov} holds trivially since $\leaves^{-\bar\alpha/d}\geq C_{d,\alpha_{\min},\bar\alpha}$ and, by the triangle inequality, $\norm{\Lambda-f}_{L^m([0,1]^d)}\leq 2\Lambda$. We therefore restrict attention to the case $\leaves\geq C_1$.
    
    When $\boldsymbol{\alpha} \in (0,1)^d$, 
    Lemma~\ref{lemma:approx-aniso-besov-piecewise} already ensures the existence of a tree-based partition 
    composed of dyadic decision trees, 
    and then \eqref{eq:conclud-approx-anis-besov} holds. 
    Consequently, the result extends to functions in $\partfunc_\leaves$. 
    Since in the statement $p$ and $m$ are fixed, the value of $q$ is accordingly determined. 
    Thus, in the remainder of the proof, we fix $m$ and regard $q=q(p)$ as a function of $p$: $q(p)=\infty$ when $0<p\leq 1$ or  $m\leq p\leq\infty$; $q(p)=p$ when $1<p<m$.

    Now suppose $\max_{1 \leq i \leq d} \alpha_i = 1$. By Corollary~\ref{cor:embedding}, it follows that
    \begin{equation}\label{eq:embedding:q+epsilon}
        B_{p-\epsilon_2, q(p)}^{\boldsymbol{\alpha}}([0,1]^d,\besov)
        \hookrightarrow 
        B_{p, q(p)}^{\boldsymbol{\alpha'}}([0,1]^d,\besov),
    \end{equation}
    where $\epsilon_1$ is an arbitrary constant, $\boldsymbol{\alpha'}=(1-\epsilon_1)\boldsymbol{\alpha} \in (0,1)^d$ and $\epsilon_2=p^2\bar\alpha\epsilon_1/(d+p\bar\alpha\epsilon_1)$. By Lemma~\ref{lemma:approx-aniso-besov-piecewise}, for any 
    $f \in B_{p, q(p)}^{\boldsymbol{\alpha'}}([0,1]^d, \Lambda)$,
    \begin{equation*}
        \inf_{\tilde{f} \in \partfunc_\leaves}
        \norm{\tilde{f} - f}_{L^m([0,1]^d)} \leq C_2 \besov \leaves^{-(1-\epsilon_1)\bar\alpha / d}.
    \end{equation*}
    Here $C_2$ depends only on $d$, $\alpha_{\min},\bar\alpha$, $p$, and $m$.
    If we let $\epsilon_1< d/(\bar\alpha\log n)$, since $\leaves\leq n$, then we have
    \begin{equation}\label{eq:p<1-approx-error-anisot-besov}
        \inf_{\tilde{f} \in \partfunc_\leaves}
        \norm{\tilde{f} - f}_{L^m([0,1]^d)} \leq C_2 \besov   \leaves^{-\bar\alpha / d}  L^{\epsilon_1\bar\alpha/d}\leq C_2 e\besov   \leaves^{-\bar\alpha / d}.
    \end{equation}
    Therefore, by \eqref{eq:embedding:q+epsilon}, \eqref{eq:p<1-approx-error-anisot-besov} holds for any 
    $f \in B_{p-\epsilon_2, q(p)}^{\boldsymbol{\alpha}}([0,1]^d, \Lambda)$.
    
    We distinguish two cases.

    \textit{Case 1: $0<p<m$.}
    In this case we have $1/m < 1/p < 1/m + \bar\alpha/d$. 
    For any $0<\delta<p$ (Here $\delta$ can always be taken as $\delta=(1/m + \bar\alpha/d)^{-1}$), if 
    \[
    \epsilon_1 < \frac{d\delta}{(p-\delta)p\bar\alpha} \wedge \frac{d}{\bar\alpha\log n},
    \]
    then this condition simultaneously ensures that $\epsilon_2 < \delta$ and that
    \eqref{eq:p<1-approx-error-anisot-besov} holds for every 
    $f \in B_{p-\epsilon_2,\, q(p)}^{\boldsymbol{\alpha}}([0,1]^d, \Lambda)$ with $p$ satisfying 
    $1/m < 1/p < 1/m + \bar\alpha/d$. 
    Since $p$ varies over an open interval and $\epsilon_2$ can be made arbitrarily small by taking $\epsilon_1$ sufficiently small, we conclude that \eqref{eq:p<1-approx-error-anisot-besov} holds for every 
    $f \in B_{p,\, q(p)}^{\boldsymbol{\alpha}}([0,1]^d, \Lambda)$ whenever 
    $1/m < 1/p < 1/m + \bar\alpha/d$.
    
    \textit{Case 2: $p\geq m$.}
    Since \eqref{eq:p<1-approx-error-anisot-besov} holds for every 
    $f \in B_{p-\epsilon_2,\, q(p)}^{\boldsymbol{\alpha}}([0,1]^d, \Lambda)$ with $p \geq m$, it also holds for any 
    $f \in B_{p_0,\, q(p_0)}^{\boldsymbol{\alpha}}([0,1]^d, \Lambda)$ by choosing $p_0 = p + \epsilon_2$ and noting that $q(p_0) = q(p) = \infty$ when $p \geq 2$. we conclude that \eqref{eq:p<1-approx-error-anisot-besov} holds for every 
    $f \in B_{p,\, q(p)}^{\boldsymbol{\alpha}}([0,1]^d, \Lambda)$ whenever 
    $p\geq m$ as $p_0-\epsilon\geq m$.
    
    The claim then follows from Lemma~\ref{lemma:embedding-anisotropic-besov-1}, claim~2.
\end{proof}

\subsection{Proof of Lemma \ref{lemma:approx-aniso-besov-piecewise-1}}

Recall that if a function $f$ belongs to the PSHAB space $\PSHABn$ associated with a tree-based partition $\partition_*=\{G_b\}_{b=1}^B$, then there exist vectors $(\boldsymbol{\alpha}_1,\ldots,\boldsymbol{\alpha}_B)$ and $(S_1,\ldots,S_B)$ such that, on each cell $G_b$, the restriction $f|_{G_b}$ is $s$-sparse and satisfies $f|_{G_b}\in \paren*{B_{p,q}^{\boldsymbol{\alpha}_b}(G_b,\Lambda_b)}_{S_b}$.
 
To invoke Corollary~\ref{cor:approx-aniso-besov-piecewise} and derive the approximation error over the PSHAB space, we first study the best approximation partition on each piece $G_b$. The main tool is an affine mapping that reduces the approximation problem on $G_b$ to the canonical domain.

For any fixed index set $S \subset [d]$, let $A_S = \{\bx_S \in [0,1]^{|S|} : \bx \in A\}$. Furthermore, if $f$ is a sparse function with relevant index set $S$, we define $f_S$ by $f_S(\bx_S) = f(\bx)$. To simplify notation, we let $\Omega = [0,1]^d$ in the rest of Appendix \ref{appendix:proof-apprrox-PSHAB}.

\begin{lemma}\label{lemma:affine-map}
    Let $A = \prod_{j=1}^d [v_j, \, v_j+\ell_j(A)] \subseteq \Omega = [0,1]^d$ be an axis-aligned rectangle. Define the affine bijection $T_A: \Omega \to A$ by
    \begin{equation}
        T_A(\bx) = \big(v_j + \ell_j(A) x_j \big)_{j=1}^d.
    \end{equation}
    For any $f: A \to \mathbb{R}$, consider the pullback $f \circ T_A: \Omega \to \mathbb{R}$. Let $p, q \in (0, \infty]$ and $\boldsymbol{\alpha} \in (0, \infty)^d$. Then:
    \begin{enumerate}
        \item[(1)] $\norm{f\circ T_A}_{L^p(\Omega)} = |A|^{-1/p} \norm{f}_{L^p(A)}$.
        \item[(2)] The Besov norm satisfies the scaling inequalities:
        \begin{equation*}
             |A|^{-1/p} \min_{j\in[d]}\ell_j(A)^{\alpha_j} \norm{f}_{B_{p,q}^{\boldsymbol{\alpha}}(A)} 
             \le \norm{f\circ T_A}_{B_{p,q}^{\boldsymbol{\alpha}}(\Omega)} 
             \le |A|^{-1/p} \norm{f}_{B_{p,q}^{\boldsymbol{\alpha}}(A)}.
        \end{equation*}
    \end{enumerate}
\end{lemma}

\begin{proof}
    \textit{Step 1: $L^p$ norm scaling.}
    The case $p=\infty$ is trivial. For $p < \infty$, observe that the Jacobian determinant of $T_A$ is $|\det J_{T_A}| = \prod_{j=1}^d \ell_j(A) = |A|$. By the change-of-variables formula with $\by = T_A(\bx)$, we have $d\by = |A| \, d\bx$, and thus
    \begin{equation}\label{eq:affine-pnorm}
        \int_{\Omega} |f(T_A(\bx))|^p \, d\bx
        = \frac{1}{|A|} \int_A |f(\by)|^p \, d\by.
    \end{equation}
    Taking the $1/p$-th power yields $\norm{f\circ T_A}_{L^p(\Omega)} = |A|^{-1/p} \norm{f}_{L^p(A)}$.

    \textit{Step 2: Besov norm scaling.}
    We focus on the case $q < \infty$ (the case $q=\infty$ follows similarly). Recall that the Besov norm on $\Omega$ is composed of the $L^p$ norm and directional semi-norms. For the $j$-th direction, let $r = \lfloor \alpha_j \rfloor + 1$. The finite difference operator satisfies the scaling relation:
    \begin{equation*}
        \Delta_{h \be_j}^r (f \circ T_A)(\bx) = \Delta_{h \ell_j(A) \be_j}^r f (T_A(\bx)).
    \end{equation*}
    Using the change of variables as in Step 1, the $L^p$-modulus of smoothness on $\Omega$ relates to that on $A$ via:
    \begin{equation}\label{eq:modulus-scaling}
        \| \Delta_{h \be_j}^r (f \circ T_A) \|_{L^p(\Omega')}
        = |A|^{-1/p} \| \Delta_{h \ell_j(A) \be_j}^r f \|_{L^p(A')},
    \end{equation}
    where $\Omega'=\Omega(r,(h/\ell_j(A))\e_j)$ and $A'=A(r,h\e_j)$ denote the appropriate domains where the differences are defined.
    Substituting this into the definition of the directional semi-norm $|f \circ T_A|_{B_{j,p,q}^{\alpha_j}(\Omega)}$ and applying the variable change $u = h \ell_j(A)$ (noting $dh/h = du/u$), we obtain:
    \begin{equation*}
        \begin{split}
            |f \circ T_A|_{B_{j,p,q}^{\alpha_j}(\Omega)} 
            &= \left( \int_0^\infty \left( h^{-\alpha_j} \| \Delta_{h \be_j}^r (f \circ T_A) \|_{L^p(\Omega')} \right)^q \frac{dh}{h} \right)^{1/q} \\
            &= |A|^{-1/p} \left( \int_0^\infty \left( \left(\frac{u}{\ell_j(A)}\right)^{-\alpha_j} \| \Delta_{u \be_j}^r f \|_{L^p(A')} \right)^q \frac{du}{u} \right)^{1/q} \\
            &= |A|^{-1/p} \ell_j(A)^{\alpha_j} |f|_{B_{j,p,q}^{\alpha_j}(A)}.
        \end{split}
    \end{equation*}
    Now, combine the $L^p$ part and the semi-norm parts. Since $\norm{g}_{B_{p,q}^{\boldsymbol{\alpha}}} = \norm{g}_{L^p} + \sum_{j=1}^d |g|_{B_{j,p,q}^{\alpha_j}}$, we have:
    \begin{equation}\label{eq:besov-combined}
        \norm{f\circ T_A}_{B_{p,q}^{\boldsymbol{\alpha}}(\Omega)} 
        = |A|^{-1/p} \left( \norm{f}_{L^p(A)} + \sum_{j=1}^d \ell_j(A)^{\alpha_j} |f|_{B_{j,p,q}^{\alpha_j}(A)} \right).
    \end{equation}
    Since $A \subseteq \Omega$, we have $\ell_j(A) \le 1$, and thus $\ell_{\min}^{\boldsymbol{\alpha}} \le \ell_j(A)^{\alpha_j} \le 1$, where $\ell_{\min}^{\boldsymbol{\alpha}} = \min_{j} \ell_j(A)^{\alpha_j}$. Applying these bounds to \eqref{eq:besov-combined} immediately yields
    \begin{equation*}
        |A|^{-1/p} \ell_{\min}^{\boldsymbol{\alpha}} \norm{f}_{B_{p,q}^{\boldsymbol{\alpha}}(A)} 
        \le \norm{f\circ T_A}_{B_{p,q}^{\boldsymbol{\alpha}}(\Omega)} 
        \le |A|^{-1/p} \norm{f}_{B_{p,q}^{\boldsymbol{\alpha}}(A)}.
    \end{equation*}
    This completes the proof.
\end{proof}

\begin{proof}[Proof of Lemma \ref{lemma:approx-aniso-besov-piecewise-1}]
    We apply the affine map in Lemma~\ref{lemma:affine-map} so that $f(T(\bx)) = f\circ T_A(\bx)$. Since $f$ is $s$-sparse, $f\circ T_A$ is also $s-$sparse and thus $f_S(\bx_S) = f(\bx)$ 
    and $f_S(T(\bx_S)) = f_S\circ T_A(\bx_S)$.

    By Lemma \ref{lemma:affine-map}, we have 
    \begin{equation}\label{eq:affine-norm}
        \norm{f_S\circ T_A}_{B_{p,q}^{\boldsymbol{\alpha}}(\Omega_S)}=
         \norm{f\circ T_A}_{B_{p,q}^{\boldsymbol{\alpha}}(\Omega)}\leq |A|^{-1/p}  \besov.
    \end{equation}
    Then we apply Corollary \ref{cor:approx-aniso-besov-piecewise} for $f_S\circ T_A$ on $\Omega_S$, there is a function $f_1\in \partfunc_\leaves(\Omega_S):\Omega\to \mathbb{R}$, such that
    \begin{equation}\label{eq:affine-approx-1}
        \norm{f_1 - f_S\circ T_A}_{L^m(\Omega_S)} 
        \leq 
        C_{s,\alpha_{\min},\bar\alpha,p,m} |A|^{-1/p} \besov   \leaves^{-\bar\alpha / s}.
    \end{equation}
    Consider $f_2:[0,1]^d \to \mathbb{R}$ as an $s-$sparse function such that $f_2(\bx)=f_1(\bx_S)$. Let $f_3:A\to \mathbb{R}$ such that after an affine map the equivalence holds: $f_3(T(\bx))=f_2(\bx)$. 
    Then by first claim of Lemma \ref{lemma:affine-map}:
    \begin{equation}\label{eq:affine-approx-2}
        \norm{f_3 - f}_{L^m(A)}= |A|^{1/m}\norm{f_2 - f\circ T_A}_{L^m(\Omega)} 
        =|A|^{1/m}\norm{f_1 - f_S\circ T_A}_{L^m(\Omega_S)}.
    \end{equation}
    Combining \eqref{eq:affine-approx-1}, \eqref{eq:affine-approx-2}:
    \begin{equation}
        \norm{f_3- f}_{L^m(A)}\leq C_{s,\alpha_{\min},\bar\alpha,p,m} |A|^{1/m-1/p} \besov   \leaves^{-\bar\alpha / s}.
    \end{equation}
    Since $f_1\in \partfunc_\leaves(\Omega_S),$ it is evident that $f_3\in \partfunc_\leaves(A)$. The claim then follows.
\end{proof}

\subsection{Proof for Appendix  \ref{appendix:proof-apprrox-PSHAB}}
\begin{lemma}\label{lemma:opimization}
Let $v_b > 0$ and $L_b > 0$ for $b = 1, \dots, B$. For a fixed constant $\theta > 0$ and $L > 0$, consider the constrained optimization problem:
\begin{equation*}
    \min_{\{L_b\}_{b=1}^B} \sum_{b=1}^B v_b L_b^{-\theta} \quad \text{subject to} \quad \sum_{b=1}^B L_b = L.
\end{equation*}
The unique global minimum is attained at
\begin{equation} \label{eq:optimal_L}
    L_b^* = \frac{v_b^{1/(\theta+1)}}{\sum_{j=1}^B v_j^{1/(\theta+1)}} L, \quad b=1, \dots, B.
\end{equation}
Furthermore, the minimum value of the objective function is given by
\begin{equation*}
    \left( \sum_{b=1}^B v_b^{1/(\theta+1)} \right)^{\theta+1} L^{-\theta}.
\end{equation*}
\end{lemma}

\begin{proof}
Let $f(L_1,\ldots,L_B) = \sum_{b=1}^B v_b L_b^{-\theta}$. The objective function $f$ is strictly convex on the positive orthant $\mathbb{R}_{++}^B$ since its Hessian is diagonal with strictly positive entries $\frac{\partial^2 f}{\partial L_b^2} = \theta(\theta+1) v_b L_b^{-(\theta+2)} > 0.$

Given that the constraint set is a convex simplex, the first-order conditions are both necessary and sufficient for a global minimum. We define the Lagrangian $\mathcal{L}(L_1, \dots, L_B, \lambda) = \sum_{b=1}^B v_b L_b^{-\theta} + \lambda \left( \sum_{b=1}^B L_b - L \right)$.
Setting the partial derivatives with respect to $L_b$ to zero:
\begin{equation*}
    \frac{\partial \mathcal{L}}{\partial L_b} = -\theta v_b L_b^{-(\theta+1)} + \lambda = 0,
\end{equation*}
which yields
\begin{equation*}
    L_b = \left( \frac{\theta v_b}{\lambda} \right)^{1/(\theta+1)}
\end{equation*}
Summing over $b$ to satisfy the constraint $\sum_{b=1}^B L_b = L$, we obtain $\left( \theta/\lambda \right)^{1/(\theta+1)} \sum_{b=1}^B v_b^{1/(\theta+1)} = L $, or equivalently:
\begin{equation*}
    \left( \frac{\theta}{\lambda} \right)^{1/(\theta+1)} = \frac{L}{\sum_{j=1}^B v_j^{1/(\theta+1)}}.
\end{equation*}
Substituting this back into the expression for $L_b$ gives the result in \eqref{eq:optimal_L}. Finally, substituting $L_b^*$ into the objective function completes the proof.
\end{proof}

\begin{lemma} \label{lemma:optimization-2}
Let $v_b > 0$ for $b = 1, \dots, B$. For a fixed constant $\theta > 0$ and $L > 0$, consider the minimax optimization problem:
\begin{equation*}
    \min_{\{L_b\}_{b=1}^B} \max_{b \in \{1, \dots, B\}} v_b L_b^{-\theta} \quad \text{subject to} \quad \sum_{b=1}^B L_b = L, \quad L_b > 0.
\end{equation*}
The optimal allocation is given by:
\begin{equation} \label{eq:minimax_L}
    L_b^* = \frac{v_b^{1/\theta}}{\sum_{j=1}^B v_j^{1/\theta}} L, \quad b=1, \dots, B.
\end{equation}
The resulting minimum value of the maximum objective is:
\begin{equation*}
    \left( \frac{\sum_{b=1}^B v_b^{1/\theta}}{L} \right)^\theta.
\end{equation*}
\end{lemma}

\begin{proof}
Let $f(\mathbf{L}) = \max_{b} v_b L_b^{-\theta}, \boldsymbol{L}=(L_1,\ldots,L_B)$. 
First, we observe that at the optimal solution $\mathbf{L}^*=(L_1^*,\ldots,L_B^*)$, we must have $v_1 (L_1^*)^{-\theta} = \dots = v_B (L_B^*)^{-\theta}$. Suppose, for contradiction, that they are not all equal. Let $\mathcal{I}$ be the set of indices such that $v_i (L_i^*)^{-\theta} = f(\mathbf{L}^*)$ for $i \in \mathcal{I}$, and $\mathcal{J}$ be the complement where $v_j (L_j^*)^{-\theta} < f(\mathbf{L}^*)$. 
If $\mathcal{J}$ is non-empty, we can decrease the objective value by slightly increasing $L_i$ for all $i \in \mathcal{I}$ (which decreases the maximum) and decreasing $L_j$ for some $j \in \mathcal{J}$. Since $v_j (L_j^*)^{-\theta}$ is strictly less than the maximum, a sufficiently small perturbation will not make any $j \in \mathcal{J}$ the new maximum. This contradicts the optimality of $\mathbf{L}^*$.

Thus, for some constant $K$, we have $v_b L_b^{-\theta} = K$ and thus $ L_b = \left( \frac{v_b}{K} \right)^{1/\theta} $ for all of $b=1,\ldots,B$.
Summing over $b$ to satisfy the constraint $\sum_{b=1}^B L_b = L$, we get $\sum_{b=1}^B \left( \frac{v_b}{K} \right)^{1/\theta} = L $ and thus:
\begin{equation*}
    K = \left( \frac{\sum_{j=1}^B v_j^{1/\theta}}{L} \right)^\theta.
\end{equation*}
Substituting $K^{-1/\theta}$ back into the expression for $L_b$ yields \eqref{eq:minimax_L}.
Since $v_b L_b^{-\theta}$ is strictly decreasing in $L_b$ and the constraint set is a simplex, this equalizing solution is the unique global minimum.
\end{proof}

\begin{lemma}[Allocation]\label{lemma:allocation}
Let $L, B$ be positive integers with $L \geq B$, and let $w = (w_1, \dots, w_B)$ be a sequence of non-negative weights such that $\sum_{b=1}^B w_b = 1$. There exists a sequence of positive integers $L_1, \dots, L_B$ such that $\sum_{b=1}^B L_b = L$, $L_b \geq 1$ for all $b=1, \dots, B$, and
\begin{equation*}
    (L-B)w_b < L_b \leq (L-B)w_b + 2.
\end{equation*}
\end{lemma}

\begin{proof}
Let $L_b^* = \lfloor (L-B)w_b \rfloor + 1$. Since $L \geq B$ and $w_b \geq 0$, we have $L_b^* \ge 1$. Summing over $b$ and using $\sum w_b = 1$ yields
\begin{equation*}
    L - B \le \sum_{b=1}^B L_b^* \le L.
\end{equation*}
Define the residual $R = L - \sum_{b=1}^B L_b^*$, where $0 \leq R < B$. We construct the final allocation as $L_b = L_b^* + \mathbb{I}(b \leq R)$, which ensures $\sum_{b=1}^B L_b = L$ and $L_b \ge 1$.
Finally, the inequality $x - 1 < \lfloor x \rfloor \leq x$ implies
\begin{equation*}
    (L-B)w_b < L_b^* \le (L-B)w_b + 1.
\end{equation*}
Adding $0 \leq \mathbb{I}(b \leq R) \leq 1$ to the inequalities gives $(L-B)w_b < L_b \leq (L-B)w_b + 2$, concluding the proof.
\end{proof}

\subsection{Proof for Appendix \ref{section:proof-minimax-reg}}

\subsubsection{Helpful lemmas}
\begin{lemma}[Gilbert-Varshamov bound]
\label{lemma:gv}
Let $\mathcal{W} = \{0, 1, \dots, W-1\}^B$ with $W \geq 2$, and let $d_{\mathrm{H}}(\cdot, \cdot)$ denote the Hamming distance. For any $\delta \in (0, 1 - 1/W)$, there exists a subset $\mathcal{T} \subset \mathcal{W}$ such that 
\begin{equation*}
    \min_{x, y \in \mathcal{T}, x \neq y} d_{\mathrm{H}}(x, y) \geq \delta B, \quad \text{and} \quad |\mathcal{T}| \geq W^{B(1 - H_W(\delta-1/B))},
\end{equation*}
where $H_W(\delta) \coloneqq \delta \log_W(W-1) - \delta \log_W \delta - (1-\delta) \log_W(1-\delta)$ denotes the $W$-ary entropy function.
\end{lemma}

\begin{lemma}\label{lemma:H_Q-1}
For any fixed $\delta \in (0, 1)$, $H_W(\delta)$ is strictly decreasing in $W$ for all integers $W \geq \max\{2, (1-\delta)^{-1}\}$.
\end{lemma}

\begin{proof}
We relax the integer base $W$ to a continuous variable $x \geq 2$. Using the natural logarithm, we have
\begin{equation*}
    H_x(\delta) = \frac{\delta \ln(x-1) + h(\delta)}{\ln x},
\end{equation*}
where $h(\delta) \coloneq -\delta \ln \delta - (1-\delta) \ln (1-\delta)$. Differentiating $H_x(\delta)$ with respect to $x$ yields
\begin{equation*}
    \frac{\partial H_x(\delta)}{\partial x} = \frac{1}{x(\ln x)^2} \braces*{ \delta \left( \frac{x}{x-1} \ln x - \ln(x-1) \right) - h(\delta) }.
\end{equation*}
To determine the sign of the derivative, let $g(\delta)$ denote the term in the braces. The second derivative of $g(\delta)$ with respect to $\delta$ is
\begin{equation*}
    g''(\delta) = -h''(\delta) = \frac{1}{\delta} + \frac{1}{1-\delta},
\end{equation*}
which is strictly positive for all $\delta \in (0, 1)$. Thus, $g(\delta)$ is strictly convex in $\delta$. 

We evaluate $g(\delta)$ at the boundaries of the interval $[0, 1 - 1/x]$. As $\delta \to 0$, $h(\delta) \to 0$, yielding $g(0) = 0$. At the upper boundary $\delta = 1 - 1/x $, a direct calculation gives $g(\delta)$ yields $g(1 - 1/x) = 0$. 
Because $g(\delta)$ is strictly convex and vanishes at both endpoints, it must be strictly negative on the interior: $g(\delta) < 0$ for all $\delta \in (0, 1 - 1/x)$. Consequently, $\frac{\partial H_x(\delta)}{\partial x} < 0$, proving that $H_x(\delta)$ is strictly decreasing in $x$. Restricting $x$ to integer values completes the proof.
\end{proof}

\begin{lemma}\label{lemma:H_Q-2}
For any integer $W \geq 2$, the $W$-ary entropy function is strictly increasing in $\delta$ on the interval $(0, 1 - 1/W)$.
\end{lemma}

\begin{proof}
Taking the first derivative of $H_W(\delta)$ with respect to $\delta$ yields
\begin{equation*}
    H_W'(\delta)
    =\log_W \left( \frac{(W-1)(1-\delta)}{\delta} \right).
\end{equation*}
Therefore, for all $\delta \in (0, 1 - 1/W)$, we have $H_W'(\delta) > 0$. This establishes that $H_W(\delta)$ is strictly increasing on the given interval.
\end{proof}
\subsubsection{Proof of Lemma \ref{lemma:covering-aniso-besov}}
\begin{proof}[Proof of Lemma \ref{lemma:covering-aniso-besov}]
First, by Theorem 4 of \citet{suzuki2021deep} (see also Proposition 10 of \citet{anisot-besov-3}), the metric entropy of the unit ball in the $s$-dimensional anisotropic Besov space satisfies
\begin{equation} \label{eq:entropy_s_dim}
    \log \mathcal{N}\bigl(\varepsilon; B_{p,q}^{\boldsymbol{\alpha}_S}([0,1]^s, 1), \|\cdot\|_{L^2([0,1]^s)}\bigr) \asymp \varepsilon^{-s/\bar{\alpha}},
\end{equation}
provided that $\bar{\alpha}/s > (1/p - 1/2)_+$. 
The class of $S$-sparse functions on $\Omega=[0,1]^d$, equipped with the $L^2(\Omega)$ and Besov norms, is isometric to the corresponding space on $[0,1]^s$ under the canonical restriction map. Consequently, the covering number of the $S$-sparse class on $\Omega$ satisfies
\begin{equation}\label{eq:covering number-anis-besov}
    \log \mathcal{N}\bigl(\varepsilon; (B_{p,q}^{\boldsymbol{\alpha}}(\Omega, 1))_S, \|\cdot\|_{L^2(\Omega)}\bigr) \asymp \varepsilon^{-s/\bar{\alpha}}.
\end{equation}

Let $T_A:\Omega\to A$ be the affine map and consider the pullback operator $f\mapsto f\circ T_A$.

\medskip
\noindent
\textit{Step 1: Lower bound.}

By Lemma \ref{lemma:affine-map}, 
\begin{equation} \label{eq:norm_scaling}
    \|f \circ T_A\|_{B_{p,q}^{\boldsymbol{\alpha}}(\Omega)} 
    \geq |A|^{-1/p} \min_{j \in [d]} \ell_j(A)^{\alpha_j} \|f\|_{B_{p,q}^{\boldsymbol{\alpha}}(A)}
    \geq
     \kappa \|f\|_{B_{p,q}^{\boldsymbol{\alpha}}(A)},
\end{equation}
where $\kappa \coloneq |A|^{-1/p} \min_{j \in [d]} \ell_j(A)$. Define 
\[
\mathcal{F}\coloneq\left\{ f \circ T_A : f \in (B_{p,q}^{\boldsymbol{\alpha}}(A, \besov))_S \right\}.
\]
Then
\[
    \mathcal{F}
    \supseteq 
    (B_{p,q}^{\boldsymbol{\alpha}}\bigl(\Omega, \kappa\besov \bigr))_S.
\]

Moreover, the $L^2$-norm satisfies $\|f \circ T_A\|_{L^2(\Omega)} = |A|^{-1/2} \|f\|_{L^2(A)}$. Hence,
\begin{align} \label{eq:metric-entropy-derivation}
    \log \mathcal{N}\bigl(\varepsilon; (B_{p,q}^{\boldsymbol{\alpha}}(A, \besov))_S, \|\cdot\|_{L^2(A)}\bigr)
    & = \log \mathcal{N}\bigl(|A|^{-1/2}\varepsilon; \mathcal{F}, \|\cdot\|_{L^2(\Omega)}\bigr) \nonumber \\
    & \geq \log \mathcal{N}\bigl(|A|^{-1/2}\varepsilon; (B_{p,q}^{\boldsymbol{\alpha}}(\Omega, \kappa \besov))_S, \|\cdot\|_{L^2(\Omega)}\bigr) \nonumber \\
    & = \log \mathcal{N}\bigl( (\kappa \besov)^{-1} |A|^{-1/2} \varepsilon ; (B_{p,q}^{\boldsymbol{\alpha}}(\Omega, 1))_S, \|\cdot\|_{L^2(\Omega)}\bigr).
\end{align}

Substituting the expression of $\kappa$ and combining with \eqref{eq:covering number-anis-besov} yields that the right-hand side of \eqref{eq:metric-entropy-derivation} is of order
\[
     \left( \frac{|A|^{\frac{1}{p} - \frac{1}{2}} \varepsilon}{\besov \min_{j \in [d]} \ell_j(A)} \right)^{-s/\bar{\alpha}}.
\]
Invoking the assumption $\min_{j \in [d]} \ell_j(A)^{s/\bar{\alpha}} \geq C_1$ completes the proof of the lower bound.

\medskip
\noindent
\textit{Step 2: Upper bound.}

Again by Lemma \ref{lemma:affine-map},
\begin{equation} \label{eq:norm_scaling-2}
    \|f \circ T_A\|_{B_{p,q}^{\boldsymbol{\alpha}}(\Omega)} 
    \leq |A|^{-1/p}  \|f\|_{B_{p,q}^{\boldsymbol{\alpha}}(A)}.
\end{equation}
Hence,
\[
    \left\{ f \circ T_A : f \in (B_{p,q}^{\boldsymbol{\alpha}}(A, \besov))_S \right\} 
    \subseteq 
    (B_{p,q}^{\boldsymbol{\alpha}}\bigl(\Omega, |A|^{-1/p}\besov \bigr))_S.
\]

Proceeding as above,
\begin{align} \label{eq:metric-entropy-derivation-2}
    \log \mathcal{N}\bigl(\varepsilon; (B_{p,q}^{\boldsymbol{\alpha}}(A, \besov))_S, \|\cdot\|_{L^2(A)}\bigr)
    & = \log \mathcal{N}\bigl(|A|^{-1/2}\varepsilon; \mathcal{F}, \|\cdot\|_{L^2(\Omega)}\bigr) \nonumber \\
    & \leq \log \mathcal{N}\bigl(|A|^{-1/2}\varepsilon; (B_{p,q}^{\boldsymbol{\alpha}}(\Omega, |A|^{-1/p} \besov))_S, \|\cdot\|_{L^2(\Omega)}\bigr) \nonumber \\
    & = \log \mathcal{N}\bigl( (|A|^{-1/p} \besov)^{-1} |A|^{-1/2} \varepsilon ; (B_{p,q}^{\boldsymbol{\alpha}}(\Omega, 1))_S, \|\cdot\|_{L^2(\Omega)}\bigr).
\end{align}
The desired upper bound follows from \eqref{eq:covering number-anis-besov}.
\end{proof}

\section{Related work on ERM trees}
\label{section:relatedwork}

\paragraph{Theoretical guarantees.} 
Despite the widespread use of decision trees, rigorous theoretical analysis of the empirical risk minimization (ERM) paradigm remains limited. Most existing literature focuses on \emph{dyadic} ERM trees, where splits are restricted to the midpoints of cells. In regression, \citet{donoho1997cart} showed dyadic trees attain optimal rates for certain bivariate anisotropic functions, with subsequent works extending these ideas \citep{binev2005universal-class-1, binev2007universal, chatterjee2021adaptive}. In classification, dyadic ERM trees have been studied by \citet{scott2006minimax, blanchard2007optimal}, and \citet{AOS2014erm}. However, dyadic partitions are rarely used in practice because they are less adaptive than non-dyadic ERM trees, which allow splits at arbitrary data points. Yet, theoretical guarantees for non-dyadic ERM trees are sparse; \citet{nobel1996histogram} established basic consistency, and \citet{chatterjee2021adaptive} provided oracle inequalities and optimal rates for bounded variation functions, but only under a restrictive, fixed-design lattice setting.  

\paragraph{Greedy and non-adaptive variants.}
Because exact ERM tree optimization can be NP-hard, historical theoretical focus has often shifted to approximations. Purely non-adaptive trees—such as Mondrian trees \citep{mourtada2017universal}—offer consistency but fail to fully capture complex spatial heterogeneity due to their lack of data-driven splitting. 
Conversely, analyses of greedy algorithms like CART \citep{scornet2016random,klusowski2020sparse,klusowski2024large,chi2022asymptotic, Mazumder2024} typically require strong assumptions (e.g., Sufficient Impurity Decrease) to prove consistency and rarely achieve minimax optimality due to the path-dependence of the greedy heuristics.

\paragraph{Algorithmic advances in optimization.}
The recent empirical interest in ERM trees has been driven by breakthroughs in exact optimization. Formulations utilizing mixed-integer programming (MIP) \citep{bertsimas2017optimal, verwer2019learning, LIU2024106629} and SAT solvers \citep{narodytska2018learning, Schidler_Szeider_2021} have demonstrated that globally optimal trees strictly improve the interpretability-accuracy trade-off. More recently, customized dynamic programming and branch-and-bound strategies have significantly improved the scalability of exact tree optimization for both classification \citep{hu2019optimal,lin2020generalized,demirovic2022murtree} and regression \citep{zhang2023optimal,bos2024piecewise,he2025foundational}. These computational advances underscore the critical need for the general statistical theory developed in this paper.

\end{document}